\documentclass[11pt]{article}
\usepackage{amsmath,amssymb,amsfonts}

\usepackage{natbib}
 \bibpunct[, ]{(}{)}{,}{a}{}{,}%

\usepackage{rotating}
\usepackage{fancyvrb}
\usepackage{url}
\usepackage{hyperref}

\usepackage{amsthm}
\usepackage{mathtools}
\usepackage{algorithm}
\usepackage{algpseudocode}
\usepackage{tikz}
\usetikzlibrary{positioning, arrows.meta, calc}
\usepackage{caption}
\usepackage{array}
\usepackage{pgfplots}
\usepackage{pgfplotstable}
\usepackage{subcaption}
\usepackage{tabularx, longtable}
\pgfplotsset{compat=newest}

\newtheorem{lemma}{Lemma}
\newtheorem{theorem}{Theorem}

\newtheorem{definition}{Definition}

\newtheorem{proposition}{Proposition}

\newtheorem{assumption}{Assumption}

\newtheorem{remark}{Remark}

\DeclareMathOperator*{\argmax}{arg\,max}
\DeclareMathOperator*{\argmin}{arg\,min}

\newcommand{\E}{\mathbb{E}}

\newcommand{\Prob}{\mathbb{P}}

\newcommand{\bpi}{\boldsymbol{\pi}}
\newcommand{\bsigma}{\boldsymbol{\sigma}}

\begin{document}
\begin{center}
        \Large \bf Reinforcement Learning for Intensity Control: An Application to Choice-Based Network Revenue Management
	\end{center}
	\begin{center}
		{Huiling Meng}\,\footnote{Department of Systems Engineering and Engineering Management, The Chinese University of Hong Kong, Hong Kong, China. Email: \url{hmeng@se.cuhk.edu.hk.} },
		{Ningyuan Chen}\,\footnote{Rotman School of Management, University of Toronto, Toronto, Canada. Email: \url{ningyuan.chen@utoronto.ca.}},
		{Xuefeng Gao}\,\footnote{Department of Systems Engineering and Engineering Management, The Chinese University of Hong Kong, Hong Kong, China. Email: \url{xfgao@se.cuhk.edu.hk.} }
	\end{center}
	\begin{center}
		\today
	\end{center}

\begin{abstract}
Intensity control is a class of continuous-time dynamic optimization problems with many important applications in Operations Research including queueing and revenue management. In this study, we propose a practical continuous-time reinforcement learning framework for intensity control using choice-based network revenue management as a case study, which is a classical problem in revenue management that features a large state space, a large action space, and a continuous time horizon.  We show that by leveraging the event-driven structure of the problem and the inherent discretization of sample paths created by the state-jump times, a defining feature of intensity control, one does not need to discretize the time horizon in advance. We adapt discrete-time Monte Carlo and temporal difference learning algorithms for policy evaluation to continuous time and develop policy-gradient-based actor-critic algorithms for event-driven intensity control. Through a comprehensive numerical study, we evaluate the proposed approach against various state-of-the-art benchmarks, demonstrating its overall superior performance and effective scalability to large-scale problems. Notably, compared to discretization-based reinforcement learning methods, our continuous-time approach delivers significantly superior performance while maintaining comparable computational efficiency. This advantage is particularly pronounced in highly non-stationary environments. 
\end{abstract}

\section{Introduction}
Many dynamic optimization problems in Operations Research are intensity controls problems, which is a class of problems with continuous time and a discrete state space.
Two notable areas are control problems in queueing \citep{bremaud1981point, chen1990optimal} and dynamic pricing/assortment problems \citep{gallego1997multiproduct,strauss2018review,gallego2019revenue} in revenue management.
Although both areas have been studied extensively in the literature, it is fair to say that most problems are still challenging to solve in practice, due to a large number of states.
In dynamic pricing and assortment, for example, the possible combinations of the remaining inventory of the products/resources make the state space impractically large and render the exact optimal solutions extremely difficult.

Meanwhile, reinforcement learning (RL) provides a computational framework to solve general dynamic optimization problems that can be formulated as Markov decision processes (MDPs).
For a comprehensive introduction to RL, see \cite{sutton2018reinforcement}.
A prototypical problem that can be solved by RL is tabular MDPs: 
There is a finite state space, a finite action space, and a discrete time horizon.

Faced with intensity control problems, one may be tempted to convert such problems to tabular MDPs and then apply RL algorithms.
One conspicuous discrepancy between intensity control and tabular MDPs is whether the time horizon is continuous or discrete.
In fact, the continuous time horizon is the defining feature of intensity control problems.
For the conversion, one may approximate continuous-time stochastic processes with discrete-time ones.
For example, arrivals of customers following a Poisson process are widely used and driving the dynamics of the two areas mentioned above. 
They can be approximated by a single arrival in a period having a Bernoulli distribution when the time horizon is discretized with a sufficiently refined grid.
This type of discretization scheme is usually carried out before the RL algorithm is executed with a uniform and pre-specified grid size.

As for the choice of the discretization grid size, one can clearly see the computational trade-off.
On one hand, an accurate approximation of intensity control problem requires a fine grid. 
This point can be perfectly illustrated by the approximation of Poisson processes: the grid size $\Delta t$ needs to be sufficiently small so that it is unlikely to have more than one arrival during a period of $\Delta t$.
As such, the dynamics can be discretized and the probability of an arrival in a time period has a Bernoulli distribution.
On the other hand, if a time step in the discrete-time system corresponds to a minuscule duration in the continuous-time system, then the computational cost is high because of the inflated length of horizon in the discrete-time system and it may lead to numerical instabilities.
To make things worse, there is no guideline on how to choose a proper discretization scheme and the trade-off cannot be evaluated beforehand.
In practice, one may experiment RL algorithms on a set of diminishing grid sizes and inspect if the obtained solutions have converged as the grid becomes finer. The computational cost is prohibitively high because a sequence of increasingly challenging problems have to be solved, not to mention that the convergence may not even be warranted at the first place. 
Indeed, it is known in the RL community that the performance of RL algorithms can be very sensitive with respect to the discreteization grid size; see, e.g., \cite{tallec2019making} in which it is empirically shown that standard $Q$-learning methods are not robust to changes in time discretization of continuous-time control problems.

In this study, we focus on a class of intensity control problems where the underlying state dynamics are driven by a Poisson arrival process and only actions taken upon these arrivals affect the system, referred to as \emph{event-driven} intensity control problems.
For this class of problems, we provide a practical framework to implement RL algorithms \emph{without the need to discretize the time horizon upfront}.
We use the classical application of choice-based network revenue management (see \citealt{strauss2018review} for a recent review) as the primary expository vehicle, while further demonstrating the generality of the proposed framework through an additional application in queueing systems in  Appendix~\ref{app:queueing-control}.
The key insight is that for event-driven 
intensity control problems, a continuous-time RL policy can be implemented exactly by querying it only at arrival times, and the state trajectory is piecewise constant and inherently discretized by its own jump times.
Although these arrivals and state jumps occur at different time points across the sample paths of the system and thus cannot be pre-determined, they are finite during the time horizon and typically substantially sparser (for a given sample path) than the grid points required in the na\"ive discretization scheme.
These structural properties enable an efficient continuous-time RL approach that circumvents the approximation error in the na\"ive discretization.
We summarize the contributions of the study below.
\begin{itemize}
    \item
    Based on the continuous-time RL formulation for the choice-based network revenue management problem, we adapt policy evaluation (including Monte Carlo and Temporal Difference methods) and policy gradient from standard discrete-time RL to the continuous-time setting. We then combine them and develop actor-critic algorithms.
    Crucially, we show that by leveraging the inherent discretization of the state-jump times to evaluate integrals, the adapted RL algorithms can be implemented largely free of discretization errors.
    Compared to the na\"ive procedure that first discretizes the time horizon and then applies the standard RL algorithms for MDPs, our approach has two major strengths.
    First, it offers several structural advantages.
    By adopting a continuous-time formulation, our approach allows the agent to learn the true continuous-time optimal policy, rather than a sub-optimal approximation restricted to the discretization grid.
    Moreover, it naturally circumvents the numerical instability
    and convergence checking of the na\"ive discretization.
    Second, it is more computationally efficient, as 
    the arrival and state-jump times for a given sample path are typically much sparser than the grid points in a refined discretization scheme.
    In particular, we do not need to consider the union of the arrival or state-jump times of all sample paths.
    \item  We provide a careful martingale formalization to legitimize the use of continuous-time policy evaluation and policy gradient methods.
    In this process, we extend the martingale approach, originally proposed in \cite{jia2022policyevaluation, jia2022policygradient} for entropy-regularized RL in controlled diffusion processes, to event-driven intensity control problems with discrete states.  While this extension certainly constitutes a technical contribution to the literature, it is somewhat independent of the application-driven focus of this paper; therefore, we relegate its detailed discussion to Section~\ref{ssec:literature}.
    \item We conduct comprehensive numerical experiments to compare the performance of the proposed continuous-time actor-critic algorithm with three approximation schemes against two categories of benchmarks: (i) classical heuristics and state-of-the-art non-RL algorithms in the literature, including the greedy policy, the CDLP policy \citep{liu2008choice}, the ADP policy \citep{zhang2009approximate}, and the optimal dynamic programming policy with refined time discretization;
    and (ii)
    discretization-based RL methods, specifically the A2C algorithm~\citep{mnih2016asynchronous} applied under different discretization levels.
    Our empirical findings are as follows: First, compared to classical heuristics and state-of-the-art non-RL benchmarks, the performance of the proposed RL algorithm is among the best, despite the fact that it is the only policy that does not need to know the environment and has to learn it through simulated samples. 
    The ADP policy from \cite{zhang2009approximate}, when the time horizon is discretized properly, has a similar performance on small or medium sized problems. But its performance may be unstable and non-monotone with respect to the decreasing size of the grid.
    Moreover, the proposed RL algorithm inherits the superior scalability of discrete-time RL methods equipped with function approximation.
    With an appropriate neural network-based function approximation, 
    it can effectively handle very large-scale problems, as demonstrated in an experiment featuring a state space of size $11^{100}$ and an action space of size $2^{200}$.
    The numerical results highlight the strong potential of deploying the proposed RL algorithm in practice.
    Second, in a bursty arrival environment, our continuous-time RL algorithm exhibits superior performance compared to the A2C algorithm applied under various discretization levels, while only incurring a computational cost comparable to that of the coarser discretization.
    This advantage is particularly pronounced in the considered bursty environment, 
    where discretization-based methods face an inherent performance-efficiency trade-off.
    \footnote{All code for the experiments in this paper is available at \url{https://github.com/huilingmeng/ct-rl-intensity-control}.}
\end{itemize}
Below we discuss the connection of this work to the literature.


\subsection{Literature Review}\label{ssec:literature}
The network revenue management problem \citep{gallego1997multiproduct} is one of the classical problems in revenue management that has been studied by numerous papers.
Its choice-based variants have been proposed and studied by  \citet{gallego2004managing,talluri2004revenue,zhang2005revenue,liu2008choice,zhang2009approximate,zhang2011improved} and many subsequent papers.
See \cite{strauss2018review} for a review.
As a dynamic optimization problem, the discrete-time version can be formulated as an MDP.
However, to solve the optimal policy, even numerically, is essentially infeasible due to the exponentially large state and action spaces.
The focus of the literature has been to provide efficient algorithms, usually with provable performance guarantees, that solve the problem approximately.
The objective of this study is to use the choice-based network revenue management as a case study and show how to adapt the RL framework to the continuous time, because RL algorithms have been shown to have impressive empirical performance for a large class of practical problems.
We note that in this literature, a number of studies including \cite{zhang2009approximate,ma2020approximation} design algorithms based approximate dynamic programming (ADP), which is an important concept and approach in RL. 
However, they focus on the discrete-time formulation and value function approximations, while we study the continuous-time formulation and general RL algorithms including exploration and the policy gradient method.
Moreover, our theoretical results are not focused on the performance guarantee but the foundation and well-posedness of RL algorithms in the continuous time.
In the numerical experiments, we compare our algorithm to two important benchmarks in the literature \citep{liu2008choice,zhang2009approximate}.

In terms of methodology, our paper builds on a series of recent studies \citep{wang2020reinforcement, jia2022policygradient, jia2022policyevaluation} on continuous-time reinforcement learning with \emph{continuous} state and action spaces. 
In particular, the stochastic processes driving the system are controlled diffusion processes, and the reward is continuously accrued over time in their models.
In contrast, we consider continuous-time reinforcement learning 
for event-driven intensity control with piecewise constant sample paths, where
both the state and action spaces are discrete, and the reward is collected only at jump times.
This leads to several subtle yet significant differences in our theoretical analysis and algorithm design, which we elaborate below. 
First and foremost, the major discovery in this study that inherent discretization can be exploited is specific to event-driven intensity control problems.
For controlled diffusion processes, one typically discretizes the time uniformly in the implementation of RL algorithms \citep{jia2022policygradient}.
Second, to facilitate a theoretical analysis of value functions under stochastic policies, \citep{wang2020reinforcement, jia2022policygradient} derived the so-called exploratory state process by applying a law-of-large-numbers argument to the drift and diffusion coefficients of the controlled diffusion process. 
This approach does not apply to event-driven intensity control, and we instead derive the exploratory dynamics based on analyzing the infinitesimal generator of the observable/sample state process. 
Third, given that jump rewards in our setting may arise only at times when customers arrive, we aim to 
generate actions only at these specific times, rather than continuously throughout the time horizon as in \cite{jia2022policygradient}. 
Fourth, while our general framework extends the martingale approach proposed in \cite{jia2022policyevaluation, jia2022policygradient} to event-driven intensity control, there are important differences in the specific formulas and implementations for policy evaluation (PE) and policy gradient (PG).  
A notable distinction arises from the treatment of integrals with respect to the time, state and randomized policies. 
Because the sample paths of the state in our problem are piecewise constant, we propose an adaptive discretization approach that takes into account the jump times of each sample trajectory to compute the aforementioned integrals. This is in sharp contrast to \cite{jia2022policyevaluation, jia2022policygradient} where integrals are computed by discretizing the horizon uniformly.
Our strategy is expected to significantly reduce the approximation errors that often arise with regular uniform discretization scheme. 
We also mention a concurrent work \citep{gao2024reinforcement} which study RL for general jump-diffusion processes. Their focus is to develop q-learning algorithms, the continuous-time counterpart of Q-learning, for jump-diffusion processes, whereas our focus is to develop PE and PG based actor-critic methods tailored for intensity control.  
Note that a few recent RL studies have been focusing on the discretization of the (continuous) state space \citep{sinclair2020adaptive,sinclair2023adaptive}, whereas this study considers intensity control with a discrete space and the inherent discretization of the (continuous) time horizon by jump points.

The (finite-horizon) intensity control problem that we study is related to continuous-time MDPs or the more general semi-Markov decision processes (SMDPs, \cite{puterman2014markov}) with discrete state spaces. It is important to note that an infinite-horizon continuous-time MDP can be transformed to an equivalent discrete-time MDP using the method of uniformization (without time discretization), but this is not the case for finite-horizon continuous-time MDPs or intensity control problems \cite[Chapter 11]{puterman2014markov}. 
More precisely, 
for infinite-horizon continuous-time MDPs (CTDMP) with the discounted or average reward, the (relative) value function is a function of the state but not the time; the optimal policy is also stationary.
One can convert the optimality equation for the infinite-horizon CTMDP
to that for an infinite-horizon discrete-time MDP with the uniformization technique under certain technical conditions; See Chapter 11.5 of \cite{puterman2014markov}. 
Specifically, the uniformization procedure transforms the original infinite-horizon CTMDP to an equivalent infinite-horizon CTMDP where the expected transition times are equal for all states and actions.
Then, the new CTMDP is essentially equivalent to a discrete-time MDP. 
For finite-horizon intensity control, however, the optimal policy is non-stationary and the optimal value function is a function of both the state and the (continuous) time. 
As a result, one can not directly convert the problem to an equivalent discrete-time finite-horizon MDP whose value function does not depend on the continuous time. This is a fundamental challenge.
Several RL algorithms were developed for infinite-horizon continuous-time MDPs as well as for SMDPs very early on  \citep{bradtke1995reinforcement, das1999solving}. 
In terms of theoretical results, \cite{gao2022logarithmic} recently establish logarithmic regret bounds for learning tabular continuous-time MDPs in the infinite-horizon average-reward setting. \cite{gao2025square} establish regret bounds for continuous-time MDPs in the finite-horizon episodic setting.
By contrast, we develop model-free RL algorithms for the \textit{finite-horizon} network revenue management problem without considering regret bounds. 
Besides RL for continuous-time MDPs with discrete spaces, there is also a surge of interest in studying continuous-time RL for controlled diffusion processes and its applications (mostly in finance), see, e.g., \cite{wang2020continuous, guo2022entropy, wang2023reinforcement, jia2023q, zhao2024policy, wu2024reinforcement, dai2023learning}. 
In contrast, we study continuous-time RL for intensity control problems with discrete state spaces. 
As we will see, the need for discretization is significantly reduced in our setting compared to the problems above.

Finally, we mention a growing body of literature on RL algorithms applied to Operations Management. 
\cite{dai2022queueing} develop proximal policy optimization methods for queueing network control problems with a long-run average cost objective. 
\cite{gijsbrechts2022can} demonstrate that the RL algorithm can match the performance of the state-of-the-art policies in inventory management, although tuning the hyperparameters for instances is needed.
\cite{zhalechian2023data} apply a learning approach to hospital admission control.
See \cite{oroojlooyjadid2022deep,li2023deep,azagirre2024better} for other applications.

\section{Problem Formulation}\label{sec:formulation}
We consider the network revenue management problem with $m$ resources and $n$ products.
The consumption matrix is given by $A \coloneqq [a_{ij}]_{m \times n}$. 
The entry $a_{ij}$ represents the amount of resource $i$ used by selling one unit of product $j$, making the $j$th column $A^{j}$ of $A$ the incidence vector for product $j$.
Let $\mathcal{J} \coloneqq \{1, \ldots, n\}$ be the set of products, with fixed prices denoted by $p = (p_1, \ldots, p_n)^\top$.

We consider a continuous-time finite selling horizon $[0, T]$. 
The initial inventory of the resources is denoted by $c = (c_1, \ldots, c_m)^\top$.
Consumers arrive according to a Poisson process with rate $\lambda$.\footnote{The reinforcement learning framework can be easily extended to non-stationary arrival rates, because (1) it is data-driven and doesn't assume the knowledge of $\lambda$ (or $\lambda_t$), and (2) the solved policy is non-stationary itself.
We can also extend the main theoretical results of the paper to the non-stationary setting.
In contrast, some of the benchmarks we consider in Section~\ref{sec:Numerical_Experiments} have to assume a stationary arrival process.
In this study, for simplicity, we focus on stationary arrivals.}
Upon arrival, based on the assortment offered by the firm $S \subseteq \mathcal{J}$ at the moment, the customer makes a choice $j \in S\cup \{0\}$.
For convenience, 
we use $0$ to represent the no-purchase option.
The choice behavior is typically captured by the choice probability $P_j(S)\in [0,1]$.
In other words, the customer purchases product $j$ with probability $P_j(S)$ when the offered assortment is $S$.
The choice probabilities are fixed over time and satisfy the standard regularity conditions such as $P_j(S) = 0$ for $j \notin S$ and $\sum_{j\in S\cup\{0\}} P_j(S) =1$, although they may be unknown for the RL algorithms.
The firm's decision problem is to find a dynamic policy that offers assortment $S_t$ at time $t$ that maximizes the expected total revenue over the selling horizon $[0, T]$.

We briefly discuss why we focus on the continuous-time setting of the problem at the first place,
because many other studies in revenue management (for example, \citealt{zhang2009approximate}) start with the discrete-time setting in which at most one customer may arrive in a time period.
(Note that we are not referring to the na\"ive discretization of the continuous-time formulation, but the discrete-time formulation of the problem itself.)
We choose the continuous-time setting mainly to illustrate the design and implementation of the RL algorithm.
Moreover, although it is well expected that the discrete-time formulation is a good approximation of the continuous-time formulation in practice, to our knowledge, there are no theoretical results characterizing the gap between their value functions and optimal policies.
Therefore, we believe there are theoretical and practical values in demonstrating how to adapt the RL algorithms to the continuous-time formulation.


\subsection{Classical Formulation of Optimal Intensity Control}\label{ssec:optimal-control}
In this section, we formulate the problem in the language of optimal intensity control.
The system state is the remaining inventory levels, and the action is the assortments offered by the firm. Accordingly, the state space $\mathcal{X}$ is $\{0, \ldots, c_1\} \times \cdots \times \{0, \ldots, c_m\}$, and the action space $\mathcal{A}$ is the collection of all subsets of $\mathcal{J}$. 
For each state $x \in \mathcal{X}$, we define $\mathcal{A}(x) \coloneqq \{S \in \mathcal{A}: x \geq A^j \text{ for all } j \in S\}$ as the collection of all available actions at state $x$.
We further denote $\mathbb{K} \coloneqq \{(t, x, S) : t \in [0, T],\ x \in \mathcal{X},\ S \in \mathcal{A}(x)\}$ to represent all valid time-state-action triples.

{Let $\{N_t^{\lambda}: t \in [0, T] \}$ be a standard Poisson process with rate $\lambda$, modeling the arrival process of all potential consumers.}
A control process can be represented as $\boldsymbol{u} = \{S_t \in \mathcal{A}: t \in [0, T] \}$, where $S_t$ specifies the firm's offered set at time $t$. 
Given a control process $\boldsymbol{u}$, let $N_t^{\boldsymbol{u}} = (N_{1,t}^{\boldsymbol{u}}, \ldots, N_{n, t}^{\boldsymbol{u}})^{\top}$ be a vector of controlled Poisson processes with intensities $(\lambda P_{1}(S_t), \ldots, \lambda P_{n}(S_t))$, interpreted as the cumulative number of the $n$ products sold by time $t$ under the control $\boldsymbol{u}$.
The process $\{N_t^{\boldsymbol{u}}: t\in [0, T]\}$ is defined on a filtered probability space $(\Omega, \mathcal{F}, \Prob; \{\mathcal{G}_t\}_{t \geq 0})$, where $\mathcal{G}_t$ includes all the information regarding customer arrival times associated with $N^{\lambda}$ and customer choices up to time $t$.
Note that the customer arrival process $N_t^{\lambda}$ is not influenced by our control process $\boldsymbol{u} = \{S_t : t \in [0, T]\}$, which adjusts the offered assortments. 
In contrast, $N_t^{\boldsymbol{u}}$ is a vector-valued process that is generated by thinning $N_t^{\lambda}$ using the control policy and captures the number of each product sold over time.

The remaining inventory of the resources at time $t$ is represented by $X_t^{\boldsymbol{u}} = c - AN_t^{\boldsymbol{u}}$. 
Given a control process $\boldsymbol{u}$ generated by a deterministic function $u$ as $S_t = u(t, X_{t-}^{\boldsymbol{u}})$, the process $X_t^{\boldsymbol{u}} \in \mathcal{X}$ is a continuous-time Markov chain.
In particular, for $(t, x, S) \in \mathbb{K}$, the controlled transition rates of $X_t^{\boldsymbol{u}}$ are given by
\begin{equation}\label{eq:q-rate}
        q ( y \mid t, x, S) = \sum_{\{j \in \mathcal{J}: A^j = x - y\}} \lambda P_j(S), \quad \forall\,  y \neq x; \quad
        q ( x \mid t, x, S) = - \lambda [1 - P_0(S)].
\end{equation}
The state $x$ can only transition to state $y$ if a product $j$ consumes an array of resources $A^j=x-y$.
Consider $\mathcal{U}$ to be the set of all non-anticipating control processes, which satisfies $\int_{0}^{T} A dN_t^{\boldsymbol{u}} \leq c$, $\Prob\text{-a.s.}$
Then, for a policy $\boldsymbol{u} \in \mathcal{U}$, the expected total revenue is given by 
\begin{align}\label{eq:value-func-V}
    V(0, c; \boldsymbol{u})\coloneqq \E\left[\int_{(0, T]} p^{\top} d {N}^{\boldsymbol{u}}_{t}\right] = \E\left[\int_0^T r(S_t) dt\right],
\end{align}
where $r(S)\coloneqq \lambda \sum_{j=1}^{n} p_j P_j(S) \text{ for all } S \in \mathcal{A}$.
The value function, defined by $V(t, x; \boldsymbol{u}) \coloneqq $ $\E\big[\int_{(t, T]} p^{\top} d {N}^{\boldsymbol{u}}_{s} \mid X_t^{\boldsymbol{u}} = x\big]$,
calculates the expected revenue during the time interval $(t, T]$ given that the vector of remaining inventory at time $t$ is $x$.
The goal of this intensity control problem is to find a control $\boldsymbol{u}^* \in \mathcal{U}$ which achieves 
$V^{*}(t, x) = \sup_{\boldsymbol{u} \in \mathcal{U}} V(t, x; \boldsymbol{u})$ for all $(t, x) \in [0, T] \times \mathcal{X}$.
From the optimal control theory, the optimal value function $V^*(t, x)$ satisfies the following Hamilton–Jacobi–Bellman (HJB) equation

\begin{equation}\label{eq:hjb}
\left\{
    \begin{aligned}
        &\frac{\partial V^*}{\partial t} (t, x) + \max_{S \in \mathcal{A}(x)} H(t, x, S, V^*(\cdot, \cdot)) = 0, \quad (t, x) \in [0, T] \times \mathcal{X} \\
        &V^*(T, x) = 0, \quad x \in \mathcal{X},
    \end{aligned}
\right.
\end{equation}
where the \textit{Hamiltonian} $H: \mathbb{K} \times C^{1, 0}([0, T] \times \mathcal{X}) \mapsto \mathbb{R}$ is defined as:
\begin{equation}\label{eq:Hamiltonian}
H(t, x, S, v(\cdot, \cdot)) = r(S) + \sum_{y \in \mathcal{X}} v(t, y) q(y \mid t, x, S).
\end{equation}
The space $C^{1, 0}([0, T] \times \mathcal{X})$ consists of all real-valued functions defined on $[0, T] \times \mathcal{X}$ that are {continuously differentiable} in $t$ over $[0, T]$ for all $x \in \mathcal{X}$.

For readers' convenience, we  provide below a heuristic derivation of the HJB equation \eqref{eq:hjb}. A rigorous treatment of general intensity control problems can be found  in Chapter VII of \cite{bremaud1981point}.
Consider the system for an infinitesimal time interval $[t, t+\Delta t]$.
The probability that no customers arrive during this interval is $1 - \lambda \Delta t + o(\Delta t)$. 
Meanwhile, the probability that exactly one customer arrives is $\lambda \Delta t + o(\Delta t)$. In the latter case, assuming the offered set is $S$, then a product $j \in S$ is sold with probability $P_j(S)$, generating a revenue of $p_j$, or a no-sale occurs with probability $P_0(S)$, resulting in no revenue.
Additionally, the probability that more than one customer arrives within $[t, t+\Delta t]$ is $o(\Delta t)$. 
Then, we have
\begin{align*}
    V^*(t, x) = {}& [1 - \lambda \Delta t + o(\Delta t)] \times V^*(t + \Delta t, x) + [\lambda \Delta t + o(\Delta t)] \times
    \\& \max_{S \in \mathcal{A}(x)}
    \bigg\{ \sum_{j \in S} P_j(S)\cdot [p_j + V^*(t + \Delta t, x - A^j)] + P_0(S) \cdot V^*(t + \Delta t, x) \bigg\} + o(\Delta t).
\end{align*}
By rearranging the terms, dividing both sides by $\Delta t$ and taking the limit as $\Delta t \rightarrow 0$, we obtain the HJB equation \eqref{eq:hjb}.

We note that the optimal control problem \eqref{eq:hjb} is challenging to solve, both analytically and computationally.
First, the continuous time horizon generally has to be discretized in order to obtain a numerical solution.
This procedure inevitably introduces discretization errors.
Furthermore, the discretization scheme needs to be carefully designed to avoid instability and guarantee convergence.  
Unfortunately, there are no general guidelines, and the practice is rather ad hoc depending on the application.
Second, the state and action spaces of the problem are of the sizes $\prod_{i=1}^m (1+c_i)$ and $2^n$, respectively.
It is virtually impossible to solve the problem exactly for a medium $m$ or $n$.
Third, in practice, the choice probabilities $P_j(S)$ that determine the transition rates $q(\cdot)$ are typically unknown to the firm and have to be learned through the collected data.
In other words, the algorithm needs to be fully data-driven \citep{chen2023frontiers}.
These challenges motivate the use of RL, which provides a computational framework for high-dimensional control problems. 
However, most classical RL algorithms are developed for discrete-time settings.
Applying such methods to our continuous-time problem typically requires an upfront time discretization to convert it into a discrete-time MDP, which reintroduces the aforementioned discretization challenges in traditional numerical methods. To circumvent this, we directly formulate the problem within a continuous-time RL framework and develop continuous-time counterparts of discrete-time RL algorithms to solve the problem without upfront time discretization.


\subsection{Formulation of Continuous-Time Reinforcement Learning}\label{ssec:RL-formulation}
In this study, we focus on policy-based reinforcement learning.
To start, we consider the following policy class, following Definition 2.1 in Chapter 2 of \cite{guo2009continuous}.
\begin{definition}
A randomized Markov policy is 
{a function $\bpi: [0, T] \times \mathcal{X} \times \mathcal{A} \mapsto [0, 1]$} that satisfies 
\begin{itemize}
    \item [(i)]
    For all $(x, S) \in \mathcal{X} \times \mathcal{A}$, the mapping $t \mapsto \bpi(S \mid t, x)$ is measurable on $[0, T]$.
    \item [(ii)]
    For all $(t, x) \in [0, T] \times \mathcal{X}$, $\bpi(\cdot \mid t,x)$ is a probability distribution on the action space $\mathcal{A}$.
\end{itemize}
Moreover, a randomized Markov policy $\bpi(\cdot \mid \cdot, \cdot)$ is called admissible if it further satisfies
	\begin{itemize}
        \item [(iii)]
        For all $(t, x) \in [0, T] \times \mathcal{X}$, it holds that $\bpi ( S \mid t, x ) = 0$ if $S \notin \mathcal{A}(x)$;
        \item [(iv)]
        For all $(x, S) \in \mathcal{X} \times \mathcal{A}$, the mapping $t \mapsto \bpi(S \mid t, x)$ is continuous on $[0, T]$.
	\end{itemize}
\end{definition}
We denote by $\Pi$ the set of admissible randomized Markov policies. When the context is clear, we simply use $\bpi(S \mid t,x)$ to denote the probability of choosing $S$ as the offered assortment in state $(t,x)$. Note that the action randomization in the stochastic policy $\bpi(\cdot \mid \cdot, \cdot)$ is independent of the customer arrival process $N^{\lambda}$ discussed in Section \ref{ssec:optimal-control}. 

In the RL formulation, instead of solving \eqref{eq:hjb} in the hope of obtaining the optimal (deterministic) policy, we consider a class of randomized policies that may choose the offered assortment at $t$ according to some probability distribution over feasible assortments.
Such randomized policies encourage exploration of actions and states that are ``suboptimal'' in the current iteration, which is a key principle of algorithmic design in reinforcement learning.
The exploratory policies can help collect data and gradually refine the approximation of the environment.
In the meantime, the policy can improve itself and converge to the optimal policy if the RL algorithm is properly designed.

We will specify the choice of the parametric family of policies
within the class of admissible randomized Markov policies in Section \ref{sec:Numerical_Experiments}. 
Generally, the choice needs to satisfy the following two conditions. First, the family is flexible enough so that it should include the optimal policy as a member or at least be able to approximate it.
This condition allows the reinforcement learning to converge to a near-optimal policy over time.
The crafting of such policy family usually depends on the problem context.
Second, the randomness can be tuned by a parameter so that one can control the degree of exploration, depending on the phase of the algorithm. 
For example, in the end of reinforcement learning, the algorithm can easily turn off or reduce exploration to generate a near-optimal policy from $\Pi$.


\begin{remark}
Because we consider randomized Markov policies, we need to enlarge the original filtered probability space $(\Omega, \mathcal{F}, \Prob; \{\mathcal{G}_t\}_{t \geq 0})$ to include the additional randomness
from sampling actions/assortments. Let $\mathcal{F}_t$ be the new sigma-algebra generated by $\mathcal{G}_t$ and another sequence of i.i.d. uniform random variables used to generate randomized actions (at customer arrival times) up to time $t$. The new filtered probability space is denoted by $(\Omega, \mathcal{F}, \Prob; \{\mathcal{F}_t\}_{t \geq 0})$.
\end{remark}


To encourage exploration, we follow \cite{jia2022policygradient} and introduce the entropy to measure the randomness of a stochastic policy $\bpi$. For all $\bpi \in \Pi$ and $(t, x) \in [0, T] \times \mathcal{X}$, denote the entropy of $\bpi(\cdot \mid t, x)$ by {$$\mathcal{H}(\bpi(\cdot \mid t, x)) \coloneqq - \sum_{S \in \mathcal{A}(x)} \bpi ( S \mid t, x )\log \bpi ( S \mid t, x).$$}
Then, we add the entropy as a bonus to the original value function, leading to  
\begin{align}
    J ( t, x; \bpi ) &= \E \left[ \int_{(t,T]} p^{\top} d {N}^{\bpi}_{s} + \gamma \int_t^T  \mathcal{H}(\bpi(\cdot \mid s, {X}^{\bpi}_{s-})) ds \mid X_t^{\bpi} = x \right],
     \label{eq:value_func_entropy_data}
\end{align}
where $\gamma \geq 0$ is referred to as the temperature parameter and controls the degree of exploration, ${N}_t^{\bpi}$ captures the number of each product sold over time under the randomized policy $\bpi$, and ${X}_t^{\bpi} = c - A {N}_t^{\bpi}$. 
Such entropy regularization is a commonly used technique to improve exploration in RL, see also \cite{haarnoja2018soft}.
We note that the main recipe of the paper that utilizes the inherent discretization is independent of whether entropy regularization is used, although without it the RL algorithm tends to perform poorly due to a lack of exploration.

The task of RL is to find a policy $\bpi^* \in \Pi$ which attains
\begin{align}\label{eq:overall-obj-RL}
   J^*(t, x) = \sup_{\bpi \in \Pi} J(t, x; \bpi) 
\end{align}
for all $(t, x) \in [0, T] \times \mathcal{X}$. 
Compared with the original problem, the optimal policy $\bpi^*$ of RL can be characterized by a Boltzmann (or softmax) distribution:
\begin{align}\label{eq:fixed_point}
 \bpi^*(S \mid t, x) = \frac{\exp\{\frac{1}{\gamma} H(t, x, S, J^*(\cdot, \cdot))\}}{\sum_{\bar S \in \mathcal{A}(x)}\exp\{\frac{1}{\gamma} H(t, x, \bar S, J^*(\cdot, \cdot))\}},
\end{align}
where the Hamiltonian $H$ is introduced in \eqref{eq:Hamiltonian}. A proof of this characterization is provided in Appendix \ref{app:pf_statements}.
It is clear that even at optimality, the exploration parameter $\gamma$ encourages offering assortments randomly, although it is more likely to sample the assortments with a higher Hamiltonian in the original problem. With a fine-tuned small $\gamma > 0 $, we can expect that the optimal policy $\bpi^*$, obtained by maximizing $J$, will achieve performance that closely approximates the original optimal value function $V^*$ in \eqref{eq:hjb}.
 

{For theoretical analysis, we introduce a Markov process $\{\tilde X^{\bpi}_t: t \in [0, T] \}$, defined on the original probability space $(\Omega, \mathcal{F}, \Prob; \{\mathcal{G}_t\}_{t \geq 0})$
that averages out the randomness in the action/policy.} 
The process $\{\tilde X^{\bpi}_t: t \in [0, T] \}$ will be referred to as the exploratory state process.
It is equivalent to the sample state process $\{X^{\bpi}_t: t \in [0, T] \}$ (defined on $(\Omega, \mathcal{F}, \Prob; \{\mathcal{F}_t\}_{t \geq 0})$) in the sense that
the distribution (law) of $\{\tilde X^{\bpi}_t: t \in [0, T] \}$ is the same as that of $\{X^{\bpi}_t: t \in [0, T] \}$. Moreover, note that for each $j=1, \ldots, n$, the process $\{N^{\bpi}_{j,t} - \int_0^t \sum_{S \in \mathcal{A}(X^{\bpi}_{t-})} \lambda P_j(S)\bpi (S \mid s, X^{\bpi}_{t-}) dt : t \in [0, T]\}$ is an $\{\mathcal{F}_t\}_{t\geq 0}$-martingale. 
Now we can rewrite 
\begin{align} 
    J ( t, x; \bpi )
    &= \E \left[ \int_t^T \bigg\{\sum_{S \in \mathcal{A}(\tilde{X}_{s-}^{\bpi})} r(S)\bpi (S \mid s, \tilde{X}^{\bpi}_{s-}) + \gamma \mathcal{H}(\bpi (\cdot \mid s, \tilde{X}_{s-}^{\bpi}))\bigg\} ds \mid \tilde{X}_t^{\bpi} = x \right].
    \label{eq:value_func_entropy_exploratory}
\end{align}
In \eqref{eq:value_func_entropy_exploratory}, 
we use {the exploratory state process to express the value function}, which is easier to work with because {the randomness in the policy and customer choice} introduces complexity to the analysis of value functions.
Interested readers may find details about the {reformulation \eqref{eq:value_func_entropy_exploratory}} in Appendix \ref{app:rl-reformulation}. 
We also emphasize that unlike the sample state process $\{X^{\bpi}_t: t \in [0, T] \}$, 
the exploratory dynamics $\{\tilde X^{\bpi}_t: t \in [0, T] \}$ is not observable, and hence its trajectories will not be used in our RL algorithm design. 

For a given randomized Markov policy $\bpi$, its value function $J(t, x; \bpi)$ can be characterized by a differential equation. The proof of Lemma \ref{lem:ode_value_func} is deferred to Appendix \ref{app:pf_statements}.
\begin{lemma}\label{lem:ode_value_func}
    A function $v \in C^{1, 0}([0, T] \times \mathcal{X})$ coincides with the value function under policy $\bpi$, i.e., $v(t, x) = J(t, x; \bpi)$ for all $(t, x) \in [0, T] \times \mathcal{X}$, if and only if it satisfies the following differential equation:
    \begin{align}\label{eq:ode_value_func}
        \frac{\partial v }{\partial t }(t, x)  + \sum_{S \in \mathcal{A}(x)} H(t, x, S, v(\cdot, \cdot)) \bpi(S \mid t, x) + \gamma \mathcal{H}(\bpi(\cdot \mid t, x))= 0, \quad (t, x) \in [0, T) \times \mathcal{X},
    \end{align}
    with the terminal condition $v(T, x) = 0,\ x \in \mathcal{X}$.
\end{lemma}

\subsection{Structural Advantages of Continuous-Time RL Formulation}\label{sec:structural-advantage}
In this section, we elaborate on the structural advantage of our continuous-time RL framework over the discretization-based RL methods that apply discrete-time RL algorithms after time discretization.

In a discretization-based RL method, the time horizon $[0, T]$ is typically discretized upfront into a uniform grid $0 = t_0 < t_1 < \cdots < t_K = T$ with step size $\Delta t = \frac{T}{K}$.
Given a policy $\bpi$, the firm interacts with the environment as follows.
At each grid time $t_k$, the firm observes the current state $X_{t_k}^{\bpi}$ and samples an assortment $S_{t_k}^{\bpi}$ from the distribution $\bpi(\cdot \mid t_k, X_{t_k}^{\bpi})$, which is held fixed over the interval $(t_k, t_{k+1}]$. 
During this interval, the state evolves according to the continuous-time dynamics and may undergo zero, one, or multiple jumps, reaching state $X_{t_{k+1}}^{\bpi}$ at time $t_{k+1}$. The firm then receives a cumulative reward $r_{(t_k, t_{k+1}]}^{\bpi}$ over the interval $(t_k, t_{k+1}]$ in the continuous-time system, which also corresponds to the one-step reward at $(t_k, X_{t_k}^{\bpi})$ in the discrete-time RL formulation.
Treating $r_{(t_k,t_{k+1}]}^{\bpi}$ as the one-step reward, the above interaction induces a discrete-time MDP on the grid.
The dataset $\mathcal{D}_{\Delta t}^{\bpi} = \{(t_k, X_{t_k}^{\bpi}, S_{t_k}^{\bpi}, r_{(t_k,t_{k+1}]}^{\bpi}): k=0, 1, \ldots, K-1\}$ collected at grid points is then used by the discrete-time RL algorithm to learn an optimal policy on the grid.
From an optimal control perspective, this approach essentially optimizes $V(0,c;\boldsymbol{u})$ over the restricted class $\mathcal{U}_{\Delta t}$ of grid-adapted piecewise-constant controls, rather than over the full non-anticipative control class $\mathcal{U}$.
Therefore, even if the learning algorithm perfectly converges to the optimum within $\mathcal{U}_{\Delta t}$, this approach suffers from an inherent suboptimality gap. It is expected that as $\Delta t\to 0$,  $\sup_{\boldsymbol{u} \in \mathcal{U}_{\Delta t}} V(0, c; \boldsymbol{u})$ converges to the optimal value $V^*(0, c)$. 
While this is the reason why time discretization is so widely used in practice, to our knowledge, there is no guideline on how to choose $\Delta t$ considering the trade-off between performance and computational efficiency.
Moreover, there is no framework that universally guarantees the convergence and stability of discrete-time RL algorithms as $\Delta t\to 0$.
 
In contrast, within the continuous-time RL framework, we aim to design algorithms that operate without upfront time discretization, thereby directly targeting the optimal value $V^*(0,c)$ of the original control problem. As a result, it has an inherent structural advantage over the discretization-based methods. 
This advantage is expected to be particularly pronounced in highly non-stationary environments.
In such scenarios, while discretization-based methods require sufficiently fine time discretization to maintain performance---which leads to substantial computational costs---our continuous-time approach naturally adapts to rapid system dynamics without such a trade-off. 
The key observation behind our continuous-time approach is that, for event-driven intensity control problems, only actions at customer arrival times affect the state dynamics. This enables the firm to exactly implement a continuous-time policy $\bpi$. Specifically,
whenever a customer arrival occurs at time $\tau$, the firm observes the pre-jump state $X_{\tau-}^{\bpi}$ and samples an assortment $S_{\tau}^{\bpi}$ from the distribution $\bpi(\cdot \mid \tau, X_{\tau-}^{\bpi})$. 
Then, if the customer purchases a product from this assortment, the system undergoes a state jump and the firm receives the jump reward;
otherwise, the customer leaves without purchasing, the state remains unchanged and zero reward is received.
We record observations only at jump times, yielding the dataset $\mathcal{D}^{\bpi} = \{(\tau_l, X_{\tau_l}^{\bpi}, S_{\tau_l}^{\bpi},
r_{\tau_l}^{\bpi}): l =1, \ldots, L\}$, where $L$ denotes the number of jumps that occur in this sample trajectory of the state process, and $\tau_l$ is the time of the $l$-th jump.
Since the dataset $\mathcal{D}^{\bpi}$ captures the full state dynamics,
the update formulas of continuous-time algorithms---which are typically defined as integrals over the entire trajectory---can even be computed exactly from the dataset $\mathcal{D}^{\bpi}$.
The accurate computation of the update formulas ensures that the continuous-time approach preserves its structural advantage over the discretization-based methods.
Therefore, as we develop our continuous-time algorithms in the subsequent sections, we systematically exploit the nature of the state process to evaluate the update formulas exactly whenever possible.

In this study, we address the continuous-time RL task in \eqref{eq:overall-obj-RL} by decomposing it into two primary objectives.
The first objective is policy evaluation (PE): for a given policy $\bpi \in \Pi$, PE aims to employ a (numerical) procedure to determine $J(t, x; \bpi)$ as a function of $(t, x)$ without any knowledge of the customer arrival rate or the choice probabilities.
This is presented in Section \ref{sec:PE}.
The second objective is policy improvement, in particular, the well-established policy gradient (PG) method: 
we attempt to estimate the policy gradient $\nabla_{\phi} J(0, c; \bpi^{\phi})$ within a suitably chosen parametric family $\{\bpi^{\phi}: \phi \in \Phi\}$ to optimize the value function $J(0, c; \bpi^{\phi})$. 
We also require this method to operate solely on observable data, as well as the learned value function of $\bpi^{\phi}$, in the absence of the knowledge about the environment.
This is presented in Section \ref{sec:PG}.
It is important to point out that we do not aim to develop fundamentally new PE and PG algorithms, but rather to provide a principled adaptation of standard discrete-time RL methods to the continuous-time setting.
In Section \ref{sec:AC}, we will combine PE and PG in an iterative manner which leads to actor-critic algorithms.
Actor-critic algorithms with PE and PG components are a popular class of model-free algorithms in discrete-time RL, where ``actor'' and ``critic'' refer to the learned
policy and value function, respectively. 
The actor–critic approach employs an actor to improve the policy using PG, and a critic to evaluate the current policy through PE. It combines the strengths of actor-only and critic-only methods. See e.g. \citep{konda1999actor, sutton2018reinforcement} for more details. Hence, we adopt this approach for choice-based network revenue management which features large state and action spaces, and requires function approximations.
An alternative (model-free) approach is the deep $Q$-learning type method from the discrete-time RL literature \citep{mnih2015human}, which uses a deep neural network to approximate the $Q$-function (the state-action value) instead of approximating the value and policy functions separately. 
However, in a continuous-time setting, the standard $Q$-function can be degenerate \citep{jia2023q} and may not be applicable to our problem unless an upfront discretization is used.

\section{Policy Evaluation}\label{sec:PE}
The objective of PE is to estimate the value function $J(t, x; \bpi)$ of a given policy $\bpi$ from data collected at jump times, without the knowledge of the environment.
In practice, this is often achieved through function approximations, where $J(t, x; \bpi)$ is approximated by a parametric family of functions $\{J^{\theta}(t, x): \theta \in \Theta\}$. The particular form of $\{J^{\theta}(t, x): \theta \in \Theta\}$  for the network revenue management problem will be discussed later. 

Before delving into the proposed continuous-time PE methods, we first recall how PE is usually conducted in standard discrete-time RL by applying discretization-based RL methods to our problem, and then draw the connection to our continuous-time approach.
In discrete-time RL, Monte Carlo and Temporal Difference (TD) methods are two most common techniques for PE.
While the Monte Carlo methods are suited for offline learning, TD methods work both online and offline.
With the goal of evaluating the approximate value function $J_{\Delta t}(t_k, x; \bpi)$, the \textit{gradient Monte Carlo} method (see Chapter 9 in \citep{sutton2018reinforcement}) updates $\theta$ using 
\begin{align}\label{eq:discrete_time_MC}
    \theta \leftarrow \theta + \alpha_{\theta} \sum_{k=0}^{K-1} \nabla_{\theta} J_{\Delta t}^{\theta}(t_k, X_{t
    _k}^{\bpi}) \bigg(\sum_{k'=k}^{K - 1} \left(r_{(t_{k'}, t_{k'+1}]}^{\bpi} + \gamma \mathcal{H}(\bpi(\cdot \mid t_{k'}, {X}^{\bpi}_{t_{k'}})) \Delta t\right) -  J_{\Delta t}^{\theta}(t_k, X_{t
    _k}^{\bpi})\bigg),
\end{align}
where $\alpha_{\theta}$ is the learning rate for $\theta$ and 
the term $\mathcal{H}(\bpi(\cdot \mid t_k, {X}^{\bpi}_{t_k}))$ represents the exploration/entropy bonus.
To interpret \eqref{eq:discrete_time_MC} at a high level, note that the underlying loss function that the gradient Monte Carlo algorithm seeks to minimize can be formulated as
\begin{align}\label{eq:discrete_time_MC_loss_func}
    L_{\Delta t}(\theta) = \frac{1}{2}\E \left[\sum_{k=0}^{K-1} \bigg(\sum_{k'=k}^{K - 1} \left(r_{(t_{k'}, t_{k'+1}]}^{\bpi} + \gamma \mathcal{H}(\bpi(\cdot \mid t_{k'}, {X}^{\bpi}_{t_{k'}})) \Delta t\right) - J_{\Delta t}^{\theta}(t_k, X_{t
    _k}^{\bpi})\bigg)^2\right],
\end{align}
where the difference term captures the deviation of the estimated value function from the realized reward (and the exploration bonus) aggregated for each time step along the sample path.
In contrast to the Monte Carlo methods, which use the whole trajectory to update $\theta$, the TD methods, {when used online}, update the estimate of the value function at each discrete time point.
For instance, the online TD(0) algorithm (with function approximations) updates $\theta$ at every time step $k$ using the following formula:
\begin{align}\label{eq:discrete_time_TD}
    \theta \leftarrow \theta + \alpha_{\theta} \nabla_{\theta} J_{\Delta t}^{\theta}(t_k, X_{t
    _k}^{\bpi}) \Big(r_{(t_k, t_{k+1}]}^{\bpi} + \gamma \mathcal{H}(\bpi(\cdot \mid t_k, {X}^{\bpi}_{t_k})) \Delta t + J_{\Delta t}^{\theta}(t_{k+1}, X_{t
    _{k+1}}^{\bpi}) - J_{\Delta t}^{\theta}(t_k, X_{t
    _k}^{\bpi}) \Big),
\end{align}
{where the TD error characterizes the difference between the estimated value of the current state and the estimated value of the subsequent state plus the realized reward associated with the transition from $X_{t_{k}}^{\bpi}$ to $X_{t_{k+1}}^{\bpi}$.}
For more details, we refer the readers to
Chapter 9 in \cite{sutton2018reinforcement}.

%
In the next two subsections, we adapt the discrete-time gradient Monte Carlo and TD algorithms to the continuous-time setting: the former is used for offline learning, while the latter enables online learning.
We emphasize that the continuous-time approach we develop can be combined with popular RL frameworks, including tabular MDPs and value/policy function approximations.



\subsection{Monte Carlo Methods}
\label{sec:Loss_Function_MC}
In this subsection, we formulate a principled loss function $L(\theta)$ for our continuous-time Monte Carlo method.
With the loss function, a gradient-based update rule similar to \eqref{eq:discrete_time_MC} can be derived in continuous time.
An ideal loss function is the mean-squared error between the estimated value function $J^{\theta}(\cdot, \cdot)$ and the true value function $J(\cdot, \cdot; \bpi)$, which we refer to as the \textit{mean-squared value error} (MSVE):
\begin{align}\label{eq:MSVE}
    \operatorname{MSVE}(\theta) \coloneqq \frac{1}{2}\E \left[\int_{0}^{T} | J(t, X_t^{\bpi}; \bpi) - J^{\theta}(t, X_t^{\bpi}) |^{2} dt\right].
\end{align}
However, since the true value function $J(\cdot, \cdot; \bpi)$ is not known, minimizing MSVE does not directly produce a feasible algorithm.

Following the loss function \eqref{eq:discrete_time_MC_loss_func} designed for discrete-time MDPs, which tracks the error between the estimated value function and the realized reward along sample paths, we propose its continuous-time counterpart $L(\theta)$:
\begin{align}\label{eq:CT-Loss}
    L(\theta) = \frac{1}{2} \E \left[\int_{0}^{T} \bigg(\int_{(t, T]} p^\top d {N}^{\bpi}_{s} + \gamma \int_t^T  \mathcal{H}(\bpi(\cdot \mid s, {X}^{\bpi}_{s-})) ds - J^{\theta} ( t, {X}_t^{\bpi}) \bigg)^{2} dt \right].
\end{align}
It is worthwhile to note that our proposed loss function $L(\theta)$ also emerges when replacing $J(t, X_t^{\bpi}; \bpi)$ in 
\eqref{eq:MSVE}---the expected value-to-go when starting at time $t$ and state $X_t^{\bpi}$---with the reward along the trajectory afterwards.
Beyond this observation, we proceed to establish a certain equivalence between $L(\theta)$ and $\operatorname{MSVE}(\theta)$ from a theoretical standpoint.

{The next theorem states that minimizing the loss function $L(\theta)$ is equivalent to minimizing MSVE.
{Indeed, the difference between $L(\theta)$ and $\operatorname{MSVE}(\theta)$ is merely a constant term that does not vary with $\theta$.}
This demonstrates the theoretical justification of our proposed loss function $L(\theta)$ in the continuous-time setting.
\begin{theorem}\label{thm:argminML=argminMSVE}
    It holds that {$\argmin_{\theta} L(\theta) = \argmin_{\theta} \operatorname{MSVE} (\theta)$}.
\end{theorem}

The proof of Theorem \ref{thm:argminML=argminMSVE} is somewhat delicate, and it hinges on the martingale property associated with $J(t, \tilde{X}_t^{\bpi}; \bpi)$, where $\{\tilde X^{\bpi}_t: t \in [0, T] \}$ is the exploratory dynamics introduced in \eqref{eq:value_func_entropy_exploratory} to facilitate the theoretical analysis.
To this end, we establish Proposition~\ref{prop:martingale}.
It is worth noting that Proposition~\ref{prop:martingale} delivers a stronger result than is needed for the proof: it provides a martingale characterization of the value function. This characterization lays the theoretical foundation for the continuous-time TD methods discussed in the next section.
The proofs of Proposition~\ref{prop:martingale} and Theorem \ref{thm:argminML=argminMSVE} are deferred to Appendix \ref{app:pf_statements}. 

\begin{proposition}\label{prop:martingale}
A function $v \in C^{1, 0}([0, T] \times \mathcal{X})$ coincides with the value function under policy $\bpi$, i.e., $v(t, x) = J(t, x; \bpi)$ for all $(t, x) \in [0, T] \times \mathcal{X}$, if and only if it satisfies $v(T, x) = 0$ for all $x \in \mathcal{X}$, and the associated process $\{\tilde{M}_t^{v}: t \in [0, T] \}$ is a square-integrable  $\{\mathcal{F}_t^{\tilde{X}^{\bpi}}\}_{t \geq 0}$-martingale, where
\begin{align*}
    \tilde{M}_t^{v} \coloneqq v (t, \tilde{X}_t^{\bpi}) + \int_0^t \bigg\{\sum_{S \in \mathcal{A}(\tilde{X}^{\bpi}_{s-})} r(S)\bpi (S \mid s, \tilde{X}^{\bpi}_{s-}) + \gamma \mathcal{H}(\bpi (\cdot \mid s, \tilde{X}_{s-}^{\bpi}))\bigg\} ds.
\end{align*}
\end{proposition}
\paragraph{An adaptive discretization procedure.}
We note that the computation of the continuous-time loss function in \eqref{eq:CT-Loss} involves integrals along the state trajectory. When no additional structure in the state process can be exploited, a common practice is to discretize the time interval in advance and use data collected at the resulting grid points to approximate these integrals.
However, for intensity control problems, the piecewise constant nature of the state process $X_t^{\bpi}$ enables us
to evaluate these integrals from data collected at the jump times, 
which substantially reduces and may even eliminate the approximation error inherent in a pre-specified discretization scheme.
More precisely, given the dataset $\mathcal{D}^{\bpi} = \{(\tau_l, X_{\tau_l}^{\bpi}, S_{\tau_l}^{\bpi},
r_{\tau_l}^{\bpi} ): l =1, \ldots, L\}$ collected at the jump times along a sample trajectory, 
integrals that only accumulate at jump times can be computed exactly as finite sums over the jump times and hence incur no approximation error. For instance, $\int_{(t, T]} p^\top d {N}^{\bpi}_{s} = \sum_{\tau_l \in (t, T]} r_{\tau_l}^{\bpi}$.
Moreover, for integrals of the form $\int_0^T z(t, X_t^{\bpi}) dt$, the piecewise-constant nature of $X_t^{\bpi}$ implies that
\begin{align}\label{eq:natural_discretization}
    \int_0^T z(t, X_t^{\bpi}) dt = \sum_{l = 0}^{L} \int_{\tau_l}^{\tau_{l+1}} z(t, X_{\tau_l}^{\bpi}) dt.
\end{align}
The expression in \eqref{eq:natural_discretization} essentially 
discretizes the time horizon for each trajectory without any discretization error.
For each $l$, since $X_{\tau_l}^{\bpi}$ is a constant, $z(t, X_{\tau_l}^{\bpi})$ reduces to a univariate function of $t$.
If this function admits a closed-form antiderivative with respect to $t$, such as a polynomial, it allows for exact evaluation of the integral $\int_{\tau_l}^{\tau_{l+1}} z(t, X_{\tau_l}^{\bpi}) dt$.
In such cases, we can completely avoid the numerical procedure and compute the value analytically.
If $z(t, X_{\tau_l}^{\bpi})$ is not integrable analytically,
we instead approximate each integral $\int_{\tau_l}^{\tau_{l+1}} z(t, X_{\tau_l}^{\bpi}) dt$ separately with a numerical integration method.
Even in this case, our approach eliminates the approximation error in the state trajectory that is inherent in a pre-specified discretization scheme for integrals of time- and state-dependent functions.
This advantage becomes particularly pronounced in environments with bursty customer arrivals, where a coarse pre-specified time grid would miss critical state jumps, while a fine one would incur a high computational cost.
We refer to the above integration scheme based on jump times as an adaptive discretization procedure, which yields a strictly more accurate evaluation of continuous-time integrals than pre-specified discretization and will be used repeatedly in the subsequent sections.

Denote $D(t_1, t_2, x; \theta) \coloneqq \int_{t_1}^{t_2} [J^{\theta}(s, x)]^2 ds$, $b(t_1, t_2, x; \theta) \coloneqq \int_{t_1}^{t_2} J^{\theta}(s, x) ds$, and $E(t_1, t_2, x; v, \bpi) \coloneqq \int_{t_1}^{t_2} v(s) \mathcal{H}(\bpi ( \cdot \mid s, x)) ds$. Let $\boldsymbol{1}$ denote the constant function equal to 1.
By applying the proposed adaptive discretization procedure to the continuous-time loss function in \eqref{eq:CT-Loss}, we derive the update rule for the value parameters $\theta$ as follows: 
\begin{align}
    \theta \leftarrow \theta - \alpha_{\theta} \nabla_{\theta} \Bigg\{\sum_{l = 0}^{L} \bigg( &\frac{1}{2} D(\tau_l, \tau_{l+1}, X_{\tau_l}^{\bpi}; \theta) - b(\tau_{l}, \tau_{l+1}, X_{\tau_l}^{\bpi}; \theta) \cdot \sum_{l' = l + 1}^{L} [ r_{\tau_{l'}}^{\bpi} + \gamma E(\tau_{l'}, \tau_{l'+1}, X_{\tau_{l'}}^{\bpi}; \boldsymbol{1}, \bpi^{\phi})] \notag \\
    &- \gamma E(\tau_{l}, \tau_{l+1}, X_{\tau_l}^{\bpi}; b(\tau_{l}, \cdot, X_{\tau_l}^{\bpi}; \theta), \bpi) \bigg)\Bigg\},
    \label{eq:MC-general-update-rule}
\end{align}
where $\alpha_{\theta}$ is the learning rate for $\theta$.
We note that in \eqref{eq:MC-general-update-rule}, the summations are always taken at the jump points whereas in the discrete-time Monte Carlo update rule \eqref{eq:discrete_time_MC}, they are taken at grid points.
For a given sample path, the jumps are typically much sparser than the refined time grid required by the discretization-based method to achieve comparable performance.
When the functions involved in \eqref{eq:MC-general-update-rule} admit closed-form expressions, our approach achieves both improved computational efficiency and higher accuracy.
Even when closed-form expressions are not available and numerical discretization is required, our approach remains more accurate under comparable discretization levels, as it avoids the approximation error in the state trajectory that is inherent to a pre-specified discretization scheme.

Having discussed the general idea, we turn our attention to the network revenue management problem, which exhibits unique characteristics. 
The practice of approximating the optimal value function using a linear combination of basis functions is well-documented in the literature (see, e.g., \citealt{zhang2009approximate}, \citealt{adelman2007dynamic}, \citealt{ma2020approximation}).
Inspired by this, we consider the linear parametrization $J^{\theta}(t, x) \coloneqq \theta^\top \varphi(t, x)$, where $\theta \in \mathbb{R}^{W}$ and $\varphi(t, x) \coloneqq (\varphi_1(t, x), \ldots, \varphi_W(t, x))^\top$ is the vector of basis functions.
As a result, the optimization problem $\arg\min_\theta L(\theta)$ can be considerably simplified.
In particular, we show next that we do not have to resort to the gradient method \eqref{eq:MC-general-update-rule} and can compute the optimal solution $\theta^*$ explicitly with Monte Carlo.

Denote $M_{\varphi,\varphi} \coloneqq \int_{0}^{T} \varphi(t, X_t^{\bpi}) \varphi(t, X_t^{\bpi})^{\top} dt$, $h(t, N^{\bpi}_{(t, T]}) \coloneqq \int_{(t, T]} p^\top d{N}^{\bpi}_{s}  + \gamma \int_{t}^{T} \mathcal{H}(\bpi(\cdot \mid s, {X}^{\bpi}_{s-})) ds$, $b_{\varphi,h} \coloneqq \int_{0}^{T} \varphi(t, X_t^{\bpi}) h(t, N^{\bpi}_{(t, T]}) dt $, and $c_{h,h} \coloneqq \E \int_{0}^{T} [h(t, N^{\bpi}_{(t, T]})]^2 dt $.
By expanding \eqref{eq:CT-Loss}, we obtain
\begin{align*}
    L(\theta)
    = \frac{1}{2}\theta^{\top}  \E[M_{\varphi, \varphi}] \theta - \theta^{\top} \E[b_{\varphi,h}]  + \frac{1}{2} \E[c_{h,h}].
\end{align*}
It is easy to see that the matrix
$\E[M_{\varphi, \varphi}]$ is positive semi-definite.
Hence, the optimization problem $\arg\min_\theta L(\theta)$ reduces to a simple unconstrained quadratic programming problem.
Applying the associated theory of unconstrained quadratic programming, we can assert the existence of a minimizer for $L(\theta)$, though it may not be unique. One such minimizer can be computed as \begin{align}\label{eq:PE_MC}
    \theta^* = \E[M_{\varphi,\varphi}]^{(-1)} \E[b_{\varphi, h}],
\end{align}
where $\E[M_{\varphi,\varphi}]^{(-1)}$ represents the Moore–Penrose inverse of matrix $\E[M_{\varphi,\varphi}]$.

In the implementation, the expectations in \eqref{eq:PE_MC}  are estimated using Monte Carlo across multiple sample trajectories. Given a sample trajectory $\mathcal{D}^{\bpi} = \{(\tau_l, X_{\tau_l}^{\bpi}, S_{\tau_l}^{\bpi},
r_{\tau_l}^{\bpi} ): l =1, \ldots, L\}$, we compute $M_{\varphi, \varphi}$ and $b_{\varphi, h}$ using the adaptive discretization procedure. Denote $\bar D(t_1, t_2, x) \coloneqq \int_{t_1}^{t_2} \varphi(s, x) \varphi(s, x)^{\top} ds$ and $\bar b(t_1, t_2, x) \coloneqq \int_{t_1}^{t_2} \varphi(s, x) ds$. Then we have
\begin{align}\label{eq:integral_D}
    M_{\varphi,\varphi} = 
    \sum_{l = 0}^{L} \bar D(\tau_{l}, \tau_{l+1}, X_{\tau_{l}}^{\bpi}),
\end{align}
and
\begin{equation}\label{eq:integral_MC_right_vector}
        b_{\varphi, h} = \sum_{l = 0}^{L} \left\{
        \bar b(\tau_{l}, \tau_{l+1}, X_{\tau_{l}}^{\bpi}) \cdot  \sum_{l' = l + 1}^{L} [ r_{\tau_{l'}}^{\bpi} + \gamma E(\tau_{l'}, \tau_{l'+1}, X_{\tau_{l'}}^{\bpi}; \boldsymbol{1}, \bpi)] + \gamma E(\tau_l, \tau_{l+1}, X_{\tau_{l}}^{\bpi}; \bar b(\tau_l, \cdot, X_{\tau_{l}}^{\bpi}), \bpi)\right\}.
\end{equation}
As will be detailed in Section \ref{sec:Numerical_Experiments}, 
the functions $\bar D$ and $\bar b$ admit closed-form expressions under our specific choice of polynomial basis functions $\varphi(t,x)$.
For the function $E$, which depends on the policy parametrization (see \eqref{eq:simplified_parametric_policy} and \eqref{eq:RO_parametric_policy}) and generally lacks an analytical expression, we approximate its values for each specific tuple $(t_1, t_2, x)$ via numerical integration.

To conclude this subsection, we have introduced a continuous-time loss function for the Monte Carlo method and established its theoretical justification. 
Furthermore, we have explored the special case of linear function approximations, where the optimal parameter that minimizes our proposed loss function can be explicitly expressed in a closed form. 
Concentrating on the closed form that involves integrals along sample trajectories, we have proposed an adaptive discretization procedure to compute the integrals, significantly reducing or even completely avoiding discretization errors.
With all the techniques in place, our Monte Carlo algorithm
enables an efficient implementation while maintaining high accuracy. This accuracy, in turn, preserves the structural advantage of our continuous-time RL approach over the discretization-based RL methods.

\subsection{TD Methods Based on Martingale Orthogonality Conditions}\label{ssec:TD} 
Given that the continuous-time Monte Carlo PE requires the entire sample trajectory over $[0, T]$, this approach is inherently offline and presents challenges when adapting for online use.
In this section, we propose
TD methods for the event-driven intensity control problem, which are suitable for use in both online and offline learning settings.
It is important to emphasize that while TD methods are standard in the discrete-time RL literature, they are much less studied for continuous-time RL problems. In particular, it is not immediately clear how one can develop TD methods for the finite-horizon event-driven intensity control problem under consideration.
To motivate our continuous-time TD method, we first review the discrete-time TD(0) algorithm in \eqref{eq:discrete_time_TD} and the logic behind it.
In fact, the discrete TD algorithms are rooted in the Bellman equation.
A key observation is that the Bellman equation holds if and only if for any bounded test function $\eta = (\eta_0, \ldots, \eta_{K-1})^{\top}$ adapted to the natural filtration, the following equation holds:
\begin{align}\label{eq:discrete_time_orthogonality_characterization}
    \E \left[ \sum_{k = 0}^{K-1} \eta_k \big\{ r_{(t_k, t_{k+1}]}^{\bpi} + \gamma \mathcal{H}(\bpi(\cdot \mid t_k, X_{t_k}^{\bpi})) \Delta t + J_{\Delta t}(t_{k+1}, X_{t_{k+1}}^{\bpi}; \bpi) - J_{\Delta t}(t_k, X_{t_k}^{\bpi}; \bpi)\big\}\right] = 0.
\end{align}
This characterization of the value function naturally leads to a practical learning procedure: by substituting $J_{\Delta t}$ with its parametric approximation 
$J_{\Delta t}^{\theta}$ and selecting an appropriate test function $\eta$, the system of equations in \eqref{eq:discrete_time_orthogonality_characterization} can be solved for $\theta$ via stochastic approximation  \citep{robbins1951stochastic}.
In particular, taking $\eta_k = \nabla_{\theta}  J_{\Delta t}^{\theta}(t_k, X_{t_k}^{\bpi})$ recovers the discrete-time TD(0) algorithm, while taking $\eta_k = \sum_{k' = 0}^{k} \lambda^{t_k - t_{k'}} \nabla_{\theta}  J_{\Delta t}^{\theta}(t_{k'}, X_{t_{k'}}^{\bpi})$ yields the TD($\lambda$) algorithm.

Guided by the logic underlying the discrete-time TD algorithms,
we establish in Theorem~\ref{thm:martingale_orthogonality_condition} a martingale orthogonality condition as a continuous-time analog of condition \eqref{eq:discrete_time_orthogonality_characterization}.
This condition is derived from the martingale characterization of the value function $J(\cdot, \cdot; \bpi)$ established in Proposition~\ref{prop:martingale}, and its proof is deferred to Appendix~\ref{app:pf_statements}.
\begin{theorem}\label{thm:martingale_orthogonality_condition}
    A function $v \in C^{1, 0}([0, T] \times \mathcal{X})$ coincides with the value function under policy $\bpi$, i.e., $v(t, x) = J(t, x; \bpi)$ for all $(t, x) \in [0, T] \times \mathcal{X}$, if and only if it satisfies $v(T, x) = 0$ for all $x \in \mathcal{X}$, and the following martingale orthogonality condition holds for any bounded process $\xi$ with $\xi_t \in \mathcal{F}_{t-}^{X^{\bpi}}$ for all $t \in [0, T]$:
    \begin{equation}\label{eq:orthogonality_condition_sample_alg}
        \E \left[ \int_{0}^{T} \xi_t \big\{ d v ( t, X_t^{\bpi}) + p^\top d N_t^{\bpi} + \gamma \mathcal{H}(\bpi(\cdot \mid t, X_{t-}^{\bpi})) dt \big\} \right] = 0.
    \end{equation}
\end{theorem}

The martingale orthogonality condition ~\eqref{eq:orthogonality_condition_sample_alg} serves as the theoretical foundation of our continuous-time TD method for event-driven intensity control problems. By specifying a parametric family $\{J^{\theta}(t, x): \theta \in \Theta\}$ to approximate the value function and the corresponding test function $\xi$, we can solve the equation~\eqref{eq:orthogonality_condition_sample_alg} via stochastic approximation for $\theta$, which leads to a practical learning algorithm. 

As a special case, we examine the TD(0) algorithm with linear function approximation to illustrate how the integrals in \eqref{eq:orthogonality_condition_sample_alg} can be computed with high accuracy using data collected at jump times together with the adaptive discretization procedure. 
We take $\xi_t = \nabla_{\theta} J^{\theta}(t, X_{t-}^{\bpi})$ 
and set $J^{\theta}(t, x) = \theta^{\top} \varphi(t, x)$, where $\varphi(t, x) = (\varphi_1(t, x), \ldots, \varphi_W(t, x))^{\top}$ is the vector of basis functions. Then there is a unique solution
to the system of equations in \eqref{eq:orthogonality_condition_sample_alg} under mild conditions. 
Indeed, let $\tilde{M}_{\varphi, \varphi} = \int_0^T \varphi(t, X_{t-}^{\bpi}) d \varphi(t, X_t^{\bpi})^{\top}$ and $\tilde{b}_{\varphi,r} = \int_0^T \varphi(t, X_{t-}^{\bpi}) \big\{ p^{\top} d N_t^{\bpi} + \gamma \mathcal{H}(\bpi(\cdot \mid t, X_{t-}^{\bpi})) dt \big\}$, the solution is given by
\begin{align}\label{eq:PE_TD_0}
    \theta^* = \E[\tilde{M}_{\varphi, \varphi}]^{-1} \E [\tilde{b}_{\varphi,r}],
\end{align}
assuming the existence of the inverse.

In practice, the expectations in \eqref{eq:PE_TD_0} are typically approximated by the empirical means over a batch of sample trajectories.
Given a sample trajectory $\mathcal{D}^{\bpi} = \{(\tau_l, X_{\tau_l}^{\bpi}, S_{\tau_l}^{\bpi},
r_{\tau_l}^{\bpi} ): l =1, \ldots, L\}$ under policy $\bpi$, we apply the adaptive discretization procedure detailed in Section~\ref{sec:Loss_Function_MC} to compute the two integrals $\tilde{M}_{\varphi, \varphi}$ and $\tilde{b}_{\varphi,r}$ in \eqref{eq:PE_TD_0}.
First, we denote $F(t_1, t_2, x) \coloneqq \int_{t_1}^{t_2} \varphi(s, x) \frac{\partial \varphi}{\partial s}(s, x) ds$ and evaluate $\tilde{M}_{\varphi, \varphi}$ as
\begin{equation}\label{eq:integral_F}
     \tilde{M}_{\varphi, \varphi} = \sum_{l = 1}^{L} \varphi(\tau_l, X_{\tau_{l-1}}^{\bpi}) [\varphi(\tau_l, X_{\tau_l}^{\bpi}) - \varphi(\tau_l, X_{\tau_{l-1}}^{\bpi})] +
     \sum_{l = 0}^{L} F(\tau_{l}, \tau_{l+1}, X_{\tau_l}^{\bpi}),
\end{equation}
where It\^{o}'s formula has been applied to $d\varphi(t, X_t^{\bpi})$. 
Next, with the function $E(t_1, t_2, x; v, \pi)$ defined in Section~\ref{sec:Loss_Function_MC},
we compute $\tilde{b}_{\varphi,r}$ as
\begin{equation}\label{eq:integral_TD_right_vector}
    \tilde{b}_{\varphi,r} = \sum_{l = 1}^{L} \varphi(\tau_l, X_{\tau_{l-1}}^{\bpi}) r_{\tau_l}^{\bpi} + \gamma \sum_{l = 0}^{L} E(\tau_l, \tau_{l+1}, X_{\tau_l}^{\bpi}; \varphi(\cdot, X_{\tau_l}^{\bpi}), \bpi).
\end{equation}
The accurate integral computation via the adaptive discretization procedure helps preserve the structural advantage inherent in our continuous-time approach.

\section{Policy Gradient}\label{sec:PG}
For an admissible policy, once we have obtained an estimate of its value function from the PE step, the next step is to improve the policy.
To this end, this section develops a practical continuous-time policy gradient (PG) method for event-driven intensity-control problems. 

To formalize this method, we consider a parametric family of admissible policies $\{\bpi^{\phi}(\cdot \mid \cdot, \cdot): \phi \in \Phi\}$ that satisfies Assumption~\ref{apt:policy_phi}. Our objective is to determine the optimal policy parameter by solving $\argmax_{\phi \in \Phi} J( 0, c; \bpi^{\phi} )$, 
which naturally directs our attention to the calculation of the gradient $\nabla_{\phi} J( 0, c; \bpi^{\phi})$.
\begin{assumption}\label{apt:policy_phi}
For all $(t, x, S) \in [0, T] \times \mathcal{X} \times \mathcal{A}$, the mapping $\phi \mapsto \bpi^{\phi}(S \mid t, x)$ is smooth on $\Phi$. Moreover, 
for all $(x, S, \phi) \in \mathcal{X} \times \mathcal{A} \times \Phi$, the mapping $t \mapsto \nabla_\phi \bpi^{\phi}(S \mid t, x)$ is continuous on $[0, T]$.
\end{assumption}

The core of developing an implementable PG algorithm lies in deriving a computable representation for the policy gradient $\nabla_{\phi} J( 0, c; \bpi^{\phi})$. Below, we explain the main idea behind this derivation in our continuous-time setting. 
Note that Lemma~\ref{lem:ode_value_func} provides a  differential-equation characterization of the value function $J(t, x; \bpi^{\phi})$ in terms of a system of differential equations.
By differentiating this system on both sides with respect to $\phi$, we obtain a new system of equations satisfied by $g(t, x; \phi) \coloneqq \nabla_{\phi} J( t, x; \bpi^{\phi})$:
\begin{equation} \label{eq:PG-ode}
    \left\{ 
    \begin{aligned}
        & \frac{\partial g}{\partial t}(t, x; \phi) + R(t, x; \phi) + \sum_{S \in \mathcal{A}(x)} \bigg(\sum_{y \in \mathcal{X}} g(t, y; \phi) q(y \mid t, x, S) \bigg) \bpi^{\phi} (S \mid t, x )  = 0, \quad (t, x) \in [0, T) \times \mathcal{X} \\
        & g(T, x; \phi) = 0, \quad x \in \mathcal{X},
    \end{aligned}
    \right.
\end{equation}
where an auxiliary reward function $R(t, x; \phi)$ is defined as 
\begin{align*}
    R(t, x; \phi) 
    & \coloneqq \sum_{S \in \mathcal{A}(x)} H(t, x, S, J(\cdot, \cdot; \bpi^{\phi})) \nabla_{\phi} \bpi^{\phi} (S \mid t, x) + \gamma\nabla_{\phi} \mathcal{H}(\bpi^{\phi}(\cdot \mid t, x)).
\end{align*}
A key observation is that the system \eqref{eq:PG-ode} coincides with the system \eqref{eq:ode_value_func} in Lemma~\ref{lem:ode_value_func} when the averaged reward function $\sum_{S \in \mathcal{A}(x)}r(S) \bpi(S \mid t, x)$ in \eqref{eq:ode_value_func} is replaced by the auxiliary reward function $R(t, x; \phi)$.
Using a proof similar to that of Lemma~\ref{lem:ode_value_func}, we obtain a representation of the policy gradient in terms of $R(t, x; \phi)$: 
\begin{align}\label{eq:representation_of_gradient_Hamiltonian}
    \nabla_{\phi} J( 0, c; \bpi^{\phi}) 
    = \E \left[ \int_{0}^{T} R(t, {X}^{\bpi^{\phi}}_{t-}; \phi ) dt \mid {X}_0^{\bpi^{\phi}} = c\right].
\end{align}  
However, it is important to note that the auxiliary reward function $R(t, x; \phi)$ incorporates the Hamiltonian $H$, which cannot be directly observed nor calculated in the absence of knowledge about environmental parameters. 
At this stage, the representation in \eqref{eq:representation_of_gradient_Hamiltonian} 
does not yet constitute an implementable policy gradient formula.
We further exploit the 
martingale property of  
the compensated Poisson arrival process,
$\{N_t^{\lambda} - \lambda t: t\in [0, T]\}$, to convert the Hamiltonian term in \eqref{eq:representation_of_gradient_Hamiltonian} into a computable form.
This leads to an  implementable policy gradient formula stated in Theorem~\ref{thm:PG}. The rigorous derivation is deferred to Appendix~\ref{app:pf_statements}.

\begin{theorem}\label{thm:PG}
    Given an admissible parameterized policy $\bpi^{\phi}$ satisfying Assumption~\ref{apt:policy_phi}, the policy gradient $\nabla_{\phi} J (0, c; \bpi^{\phi})$ admits the following representation: 
    \begin{align}
        \nabla_{\phi} J (0, c; \bpi^{\phi}) =
        \E \Bigg[ {}& \int_{(0,T]}  \sum_{j=1}^{n} \nabla_{\phi} \log \bpi^{\phi} ( S_t^{\bpi^{\phi}} \mid t, X_{t-}^{\bpi^{\phi}}) [J(t, X_{t-}^{\bpi^{\phi}} - A^j; \bpi^{\phi} ) - J(t, X_{t-}^{\bpi^{\phi}}; \bpi^{\phi}) + p_j] d N_{j,t}^{\bpi^{\phi}} \notag \\
        & + \gamma \int_{0}^{T} \nabla_{\phi} \mathcal{H}(\bpi^{\phi} ( \cdot \mid t, X_{t-}^{\bpi^{\phi}})) dt \mid X_0^{\bpi^{\phi}}=c\Bigg] .  \label{eq:PG-formula}
    \end{align}
\end{theorem}

Theorem~\ref{thm:PG} extends the policy gradient formula for controlled diffusion processes with continuous states in \cite{jia2022policygradient} to the event-driven intensity control problem with discrete states.
While both formulas are established in continuous time, the fundamental difference lies in whether upfront time discretization is needed for RL implementation in the continuous-time system.
For diffusion-based control problems, 
the state dynamics evolve continuously, and the control acts upon the state at every infinitesimal instant.
Therefore, one is forced to discretize the time horizon and approximately implement the continuous-time policy by sampling an action at each grid point and holding it constant over the subsequent interval.
The data collected at these grid points are then used to approximate the continuous-time integral in their policy gradient formula.
This procedure inevitably introduces discretization errors.
In contrast, 
we benefit from the inherent characteristics of event-driven intensity problems: only actions taken at the arrival times affect the state dynamics.
As a result, we can implement a continuous-time randomized policy exactly by sampling an action from the policy at each arrival time, and then collect data at the state-jump times to evaluate the policy gradient formula in Theorem~\ref{thm:PG} with high accuracy.

To be specific, given a sample trajectory $\mathcal{D}^{\bpi^{\phi}} = \{(\tau_l, X_{\tau_l}^{\bpi^{\phi}}, S_{\tau_l}^{\bpi^{\phi}},
r_{\tau_l}^{\bpi^{\phi}} ): l =1, \ldots, L\}$ generated under the current policy $\bpi^{\phi}$, suppose that the PE step has produced an estimate $J^{\theta^*}(t, x)$ of the value function $J(t, x; \bpi^{\phi})$.
We then use $\mathcal{D}^{\bpi^{\phi}}$ and $J^{\theta^*}(t, x)$ to construct a single-sample estimate of the policy gradient and derive the update rule for the parameters $\phi$.
Note that $N_{j,t}^{\bpi^{\phi}}$ is a counting process where its increment $d N_{j,t}^{\bpi^{\phi}} = 1$ if a product $j$ is sold at time $t$, and $0$ otherwise, and the term $J(t, X_{t-}^{\bpi^{\phi}} - A^j; \bpi^{\phi} ) - J(t, X_{t-}^{\bpi^{\phi}}; \bpi^{\phi}) + p_j$ captures the shadow price of product $j$.
With the value function $J(t, x; \bpi^{\phi})$ replaced by its estimate $J^{\theta^*}(t, x)$, we apply the adaptive discretization procedure to compute the integrals in \eqref{eq:PG-formula}.
This yields the following update rule for policy parameters $\phi$:
\begin{equation}\label{eq:PG-update-rule}
\begin{aligned}
    \phi \gets \phi + \alpha_{\phi} \nabla_{\phi} \Bigg\{&\sum_{l=1}^L \log \bpi^{\phi} ( S_{\tau_l}^{\bpi^{\phi}} \mid \tau_l, X_{\tau_{l-1}}^{\bpi^{\phi}}) [J^{\theta^*}(\tau_l, X_{\tau_l}^{\bpi^\phi}) - J^{\theta^*}(\tau_l, X_{\tau_{l-1}}^{\bpi^{\phi}}) + r_{\tau_l}^{\bpi^{\phi}}] \\
    &+ \gamma \sum_{l = 0}^{L} E(\tau_l, \tau_{l+1}, X_{\tau_l}^{\bpi^{\phi}}; \boldsymbol{1}, \bpi^{\phi}) \Bigg\},
\end{aligned}
\end{equation}
where $\alpha_{\phi}$ is the learning rate for $\phi$. 
In practice, while the first sum in \eqref{eq:PG-update-rule} incurs no approximation error, the term $E(\tau_l, \tau_{l+1}, X_{\tau_l}^{\bpi^{\phi}}; \boldsymbol{1}, \bpi^{\phi})$ is typically approximated over each interval $[\tau_l, \tau_{l+1}]$ via numerical integration, due to the absence of a closed-form expression for $E(t_1, t_2, x; \boldsymbol{1}, \bpi^{\phi})$ under the chosen policy parametrization $\bpi^\phi$.
Nevertheless, as clarified in Section~\ref{sec:Loss_Function_MC}, our adaptive approach eliminates the additional state discretization error associated with a pre-specified discretization scheme, which leads to a more accurate approximation of the integral.

Overall, our approach utilizes the data collected at the jump times to evaluate the continuous-time policy gradient formula with high accuracy. 
This consistent accuracy, maintained throughout both PE and PG steps ensures the structural advantage of our continuous-time RL approach relative to the discretization-based RL methods. 
Ultimately, the structural advantage translates into superior empirical performance, which will be illustrated in Section~\ref{sec:CT-vs-DT}.

\section{Actor-Critic Algorithms}\label{sec:AC}
Combining the PE and PG modules in Sections \ref{sec:PE} and \ref{sec:PG}, we next present two model-free actor-critic algorithms for the network revenue management problem introduced in Section~\ref{sec:formulation}. 
Algorithm 1, which utilizes the Monte Carlo method for PE (Section \ref{sec:Loss_Function_MC}), is detailed in the main text.
Algorithm 2, adhering to a similar framework but employing the TD(0) method for PE (Section \ref{ssec:TD}), is presented in Appendix \ref{app:AC_algs}. 

In both
algorithms, we assume a linear parametrization for $\{J^{\theta}(t, x): \theta \in \Theta\}$, that is, $J^{\theta}(t, x) = \sum_{j=1}^W \theta_j \varphi_j(t, x)$, where $\varphi_1, \ldots, \varphi_W$ are the basis functions.
We also consider a parameterized family of stochastic policies $\{\bpi^{\phi}(S \mid t, x): \phi \in \Phi\}$. Our aim is to determine the optimal values for $(\theta, \phi)$ jointly, by alternately updating each parameter. Note that both algorithms are designed for offline learning, where full trajectories are sampled and observed repeatedly
during different episodes and
$(\theta, \phi)$ are updated after every $M$ episodes, with $M$ defined as the batch size.



In addition, we employ an environment simulator to generate trajectories under given policies.
The environment simulator, denoted as $(t', x', S', r') = \text{Environment}(t, x, \bpi^{\phi}(\cdot \mid \cdot, \cdot))$, operates by first taking current time-state pair $(t, x)$ and the policy $\bpi^{\phi}(\cdot \mid \cdot, \cdot)$ as inputs.
It then samples a time interval $s$ from an exponential distribution with rate $\lambda$, indicating the duration until the next customer arrival.
Using the updated time $t' = t+s$ and the current state $x$, an offer set $S'$ is sampled from the policy $\bpi^{\phi}(\cdot \mid t', x)$.
The state transitions to $x-A^j$ with probability $P_j(S')$ and remains at $x$ with probability $P_0(S')$.
Based on this distribution, the state $x'$ at time $t'$ is sampled, and the corresponding reward $r'$ is returned.
If $x' \neq x$, the simulation outputs the jump time $t'$, new state $x'$, action $S'$ and reward $r'$.
If $x' = x$, with the current time-state pair $(t', x')$, the process repeats: a new time interval $s$ is drawn from the same exponential distribution, and the above steps are executed until a transition to a different state occurs, at which point it outputs the jump time, along with the corresponding new state, action, and reward at that time.
We note that the simulator may not coincide with the ground-truth data generating process when the environment is unknown.
Even in this case, the simulator can be constructed using the estimation from past real-world data.
We explain this subtle point in Section~\ref{sec:Numerical_Experiments} in the application to network revenue management.

\begin{algorithm}
\caption{Actor-Critic Algorithm with Linear Value Function Approximation (PE via Monte Carlo)}\label{alg:MC}
\small
    \begin{algorithmic}[1]
    \State\textbf{Inputs:} initial state $c$, time horizon $T$, number of {episodes} $N$, batch size $M$; functional forms of basis functions $\varphi_1, \ldots, \varphi_W$,
    functional form of the policy $\bpi^{\phi}(\cdot \mid \cdot, \cdot)$ and an initial value $\phi_0$; 
    {entropy factor $\gamma$, learning rate $\alpha_{\phi}$}
    \State\textbf{Required program:} an environment simulator $(t', x', S', r') = Environment(t, x, \bpi^{\phi}(\cdot \mid \cdot, \cdot))$ \label{step:simulator}
    \State Initialize $\phi \leftarrow \phi_0$
    \For{$i = 1$ to $N$} 
    \State Store $(\tau_0^{(i)}, x_0^{(i)}) \gets (0, c)$
    \State Initialize $l = 0$, $(t, x) = (0, c)$ 
    \Comment{Initialize $l$ to count jumps in each state trajectory, and $(t, x)$ to record the time and state right after each jump}
    \While{True}
    \State Apply $(t, x)$ to the environment simulator to get $(t', x', S', r') = Environment(t, x, \bpi^{\phi}(\cdot \mid \cdot, \cdot))$
    \If{$t' \geq T$}
    \State Store $L^{(i)} \gets l$, $\tau_{L^{(i)}+1}^{(i)} \gets T$
    \State \textbf{break}
    \EndIf
    \State Update $l \gets l + 1$
    \State Store observation at jump time: $(\tau_l^{(i)}, x_l^{(i)}, S_l^{(i)}, r_l^{(i)}) \gets (t', x', S', r')$
    \State Update $(t, x) \gets (t', x')$
    \EndWhile
    \If{$i= 0 \pmod{M} $} {\Comment{Perform an update using $M$ episodes generated under the policy}}
        \State Evaluate policy $\bpi^{\phi}$: [using formula \eqref{eq:PE_MC}, incorporating techniques \eqref{eq:integral_D} and \eqref{eq:integral_MC_right_vector}]
        \begin{align*}
            \theta^* = 
            &\Bigg[\frac{1}{M} \sum_{k=i-M+1}^i \sum_{l = 0}^{L^{(k)}} \bar D(\tau_l^{(k)}, \tau_{l+1}^{(k)}, x_{l}^{(k)}) \Bigg]^{(-1)} \times \\
            & \Bigg[ \frac{1}{M} \sum_{k=i-M+1}^i \sum_{l = 0}^{L^{(k)}}\bigg( \bar b(\tau_{l}^{(k)}, \tau_{l+1}^{(k)}, x_{l}^{(k)}) \cdot \sum_{l' = l + 1}^{L^{(k)}} [ r_{l'}^{(k)} + \gamma E(\tau_{l'}^{(k)}, \tau_{l'+1}^{(k)}, x_{l'}^{(k)}; \boldsymbol{1}, \bpi^{\phi})] \\
            &\hspace{2.5cm} + \gamma E(\tau_{l}^{(k)}, \tau_{l+1}^{(k)}, x_{l}^{(k)}; \bar b(\tau_{l}^{(k)}, \cdot, x_{l}^{(k)}), \bpi^{\phi})\bigg)\Bigg]
        \end{align*}
        \State Compute policy gradient at $\phi$: [using formula \eqref{eq:PG-update-rule}]
        \begin{align*}
            \Delta \phi = 
            \nabla_{\phi} \Bigg\{\frac{1}{M}\sum_{k=i-M+1}^i \bigg(&\sum_{l = 1}^{L^{(k)}} \log \bpi^{\phi} ( S_{l}^{(k)} \mid \tau_{l}^{(k)}, x_{l-1}^{(k)}) [ J^{\theta^*} (\tau_l^{(k)}, x_l^{(k)}) - J^{\theta^*} (\tau_{l}^{(k)}, x_{l-1}^{(k)} ) + r_{l}^{(k)} ] \\
            & + \gamma \sum_{l = 0}^{L^{(k)}} E(\tau_{l}^{(k)}, \tau_{l+1}^{(k)}, x_{l}^{(k)}; \boldsymbol{1}, \bpi^{\phi}) \bigg) \Bigg\}
        \end{align*}
        \State Update $\phi$ by
        \begin{align*}
            \phi \leftarrow \phi + \alpha_{\phi} \Delta \phi
        \end{align*} 
    \EndIf
    \EndFor
    \end{algorithmic}
\end{algorithm}

\begin{remark}
\label{rk:neural_network}
{We consider linear value function approximations as an example in presenting the algorithm.
It is a common approximation scheme, used by many studies including the discrete-time ADP method in \citealt{zhang2009approximate} which serves as a benchmark in our experiments in Section~\ref{sec:Numerical_Experiments}. 
Our algorithmic framework for event-driven intensity control is general, and one can apply nonlinear neural-network based approximations for value functions and policies. In such a case, instead of obtaining an explicit optimal $\theta^*$ during the PE step (see line 19 of Algorithm 1), we 
can employ the gradient method in \eqref{eq:MC-general-update-rule} to update the parameter $\theta$ in the critic network, and alternately update $\phi$ in the actor network using the policy gradient formula \eqref{eq:PG-formula}. 
The corresponding algorithm (Algorithm \ref{alg:NN}) is presented in Appendix \ref{app:AC_algs}. 
}
\end{remark}

\begin{remark}
In Algorithms \ref{alg:MC} - \ref{alg:NN}, the parameters are updated using the fixed learning rates and step sizes. However, when implementing the algorithms, a suitable optimizer can be employed to dynamically adapt the step size, making the learning process more efficient and stable. In the experiments presented in this paper, we consistently use the Adam optimizer, developed by \cite{kingma2014adam}.
\end{remark}


\section{Experimental Setup and Numerical Performance}
\label{sec:Numerical_Experiments}
In this section, we specialize the algorithms to the choice-based network revenue management.
In our experiments, we consider three different combinations of value and policy approximations, which are described in detail below.

\emph{Linear-Pair}: We first seek a suitable family of functions $\{J^{\theta}(t, x): \theta \in \Theta\}$ that can approximate the value function associated with any particular policy within our selected parametric policy family $\{\bpi^{\phi}: \phi \in \Phi\}$.
We draw inspiration from the form of the approximated optimal value function within a discrete-time framework, as suggested by \cite{zhang2009approximate}: 
\begin{align}\label{eq:discrete_approximated_optimal_value_function}
    V_{\text{ADP-}\Delta t}(t_k, x) = \theta_k + \sum_{i = 1}^{m} V_{k, i} x_i,
\end{align}
where $V_{k, i}$ estimates the marginal value of a unit resource $i$ in period $(t_k, t_{k+1}]$, and $\theta_k$ is a constant offset. For terminal condition, it is assumed that $\theta_{K} = 0$ and $V_{K, i} = 0$ for all $i = 1, \ldots, m$.
We propose to use the following continuous time counterpart for the family $\{J^{\theta}(t, x): \theta \in \mathbb{R}^{(m+1) \times (d + 1)}\}$:
\begin{align}\label{eq:continuous_approximated_value_function}
    J^{\theta}(t, x) \coloneqq \sum_{l = 0}^{d} \theta_{(0, l)} \Big( 1 - \frac{t}{T} \Big)^{l} + \sum_{i = 1}^{m} \bigg(\sum_{l = 0}^{d} \theta_{(i, l)} \Big( 1 - \frac{t}{T} \Big)^{l}\bigg) x_{i}.
\end{align}
Here, we replace $\theta_k$ and $V_{k,i}$ in the discrete-time approximation \eqref{eq:discrete_approximated_optimal_value_function}, which would be infinite-dimensional in the continuous time,
with a $d$th-order polynomial of $(1 - \frac{t}{T})$. 
The family $\{J^{\theta}(t, x): \theta \in \mathbb{R}^{(m+1) \times (d + 1) }\}$ constitutes a linear space, with the basis functions
\begin{equation}\label{eq:basis_functions}
    \begin{aligned}
        \varphi(t, x) = \bigg[1,\, 1 - \frac{t}{T}, \cdots, \Big( 1 - \frac{t}{T} \Big)^{d};\, &x_1,\, \Big( 1 - \frac{t}{T} \Big)x_1, \cdots, \Big( 1 - \frac{t}{T} \Big)^{d}x_1; \\
        \cdots; \, &x_m, \, \Big( 1 - \frac{t}{T} \Big)x_m, \cdots,  \Big( 1 - \frac{t}{T} \Big)^{d} x_m\bigg]^{\top}. 
    \end{aligned}
\end{equation} 

{We now explore the selection of the parametric family of policies $\{\bpi^{\phi}: \phi \in \Phi\}$, guided by the family of value functions above.} 
Given an approximation $J^{\theta}$ as defined in \eqref{eq:continuous_approximated_value_function} for the optimal value function $J^*$, applying the Hamiltonian $H$ to $J^{\theta}$ yields the following expression:
\begin{align*}
    H(t, x, S, J^{\theta}(\cdot, \cdot)) &= \sum_{j = 1}^{n}  [J^{\theta}(t, x - A^j) - J^{\theta}(t, x) + p_j] \lambda P_j(S) \\
    &= - \sum_{l = 0}^{d} \sum_{i = 1}^{m} \sum_{j = 1}^{n} A_{ij} \theta_{(i, l)}  \lambda  P_j(S) \Big( 1 - \frac{t}{T} \Big)^{l} + r(S),
\end{align*}
which is also a $d$th-order polynomial in $(1 - \frac{t}{T})$.
The above discussion, together with \eqref{eq:fixed_point}, motivates us to consider the parametric family of policies:
$\bpi^{\phi} (S \mid t, x) \coloneqq \frac{ \exp\{\frac{1}{\gamma} h_S (t; \phi)\} }{\sum_{\bar S \in \mathcal{A}(x)} \exp\{\frac{1}{\gamma} h_{\bar S} (t; \phi)\} }$,
where $\phi \in \mathbb{R}^{2^n \times (d + 1)}$ and $h_{S}(t; \phi) \coloneqq \sum_{l = 0}^{d} \phi_{(S, l)} ( 1 - \frac{t}{T} )^{l}$.
It should be noted that the above stochastic policy involves $2^n \times (d + 1)$ parameters $\{\phi_{(S, l)}: S \in \mathcal{A},\ l = 0, \ldots, d\}$, which grows exponentially with $n$ and becomes intractable for even moderate-sized $n$.
To address this challenge, {we limit the number of parameters by only capturing the interaction between a pair of products.
Specifically, we introduce a set of parameters $\{\phi_{(j, j', l)}: j = 1, \ldots, n;\ j' = 1, \ldots, n;\ l = 0, \ldots, d\}$, yielding a parameter space of $\mathbb{R}^{n \times n \times (d+1)}$. 
For all $S \in \mathcal{A}$ and $l = 0, \ldots, d$, we let $\phi_{(S, l)}$ be 
$\phi_{(S, l)} = \sum_{1 \leq j,\, j' \leq n} \phi_{(j, j', l)} \delta_j(S) \delta_{j'}(S)$,
where $\delta_j(S)$ equals 1 if $j \in S$, and 0 otherwise.
Intuitively, $\phi_{(j, j', l)}$ captures the interaction between product $j$ and $j'$.
This strategy effectively reduce the number of parameters from $O(2^n)$ to $O(n^2)$.
Then, we employ the reduced parametric family $\{\bpi^{\phi}: \phi \in \mathbb{R}^{n \times n \times (d + 1)}\}$, where $\bpi^{\phi}$ is specified as follows: for $(t, x, S) \in \mathbb{K}$,
\begin{align}\label{eq:simplified_parametric_policy}
    \bpi^{\phi} (S \mid t, x) = \frac{\exp\big\{ \frac{1}{\gamma} \sum_{l=0}^d \big(\sum_{1 \leq j,\, j' \leq n} \phi_{(j, j', l)} \delta_j(S) \delta_{j'}(S) \big) ( 1 - \frac{t}{T} )^l \big\}}{\sum_{\bar S \in \mathcal{A}(x)} \exp\big\{ \frac{1}{\gamma} \sum_{l=0}^d \big(\sum_{1 \leq j,\, j' \leq n} \phi_{(j, j', l)}  \delta_j(\bar S) \delta_{j'}(\bar S) \big) ( 1 - \frac{t}{T} )^l\big\} }.
\end{align}
We refer to the combination of the value approximation \eqref{eq:continuous_approximated_value_function} and the policy approximation \eqref{eq:simplified_parametric_policy} as the ``Linear-Pair'' approach.

\emph{Linear-RO}:
Drawing the intuition from the observation that the optimal assortment for the MNL model with disjoint segments is revenue-ordered (RO) in the static setting \citep{talluri2004revenue}, we consider the RO-based policy approximation.
More precisely, 
suppose the original product set $\mathcal{J} = \{1, 2, \ldots, n\}$ is composed of $\bar{L}$ disjoint consideration sets for customers, denoted as $\mathcal{J}_l$ for each segment $l = 1, \ldots, \bar{L}$. 
Sorting the products in $\mathcal{J}_l$ by descending prices yields the index sequence $\{i_l^1, i_l^2, \ldots, i_l^{n_l}\}$, which further results in $(n_l+1)$ revenue-ordered sub-assortments $S_l^{[0]} = \emptyset$, $S_l^{[1]} = \{i_l^1\}$, $S_l^{[2]} = \{i_l^1, i_l^2\}$, $\cdots$, $S_l^{[n_l]} = \{i_l^1, i_l^2, \ldots, i_l^{n_l}\}$. 
Then, a revenue-ordered assortment is defined as the union of one revenue-ordered sub-assortment from each segment.
We denote $\bar{n} = \prod_{l=1}^{\bar{L}} (n_l + 1)$ and correspondingly, the $\bar{n}$ revenue-ordered assortments are denoted as $S^{[1]}$, $S^{[2]}$, $\cdots$, $S^{[\bar{n}]}$.
We first introduce a intermediate policy $\bsigma^{\phi}$ whose action space comprises all revenue-ordered assortments without accounting for the constraints.
For $t \in [0, T]$ and $k = 1, \ldots, \bar{n}$,
\begin{align*}
    \bsigma^{\phi} (S^{[k]} \mid t) \coloneqq \frac{\exp\big\{ \frac{1}{\gamma} \sum_{l=0}^d \phi_{(k, l)} (1 - \frac{t}{T})^l \big\}}{\sum_{k' = 1}^{\bar{n}} \exp\big\{ \frac{1}{\gamma} \sum_{l=0}^d \phi_{(k', l)} (1 - \frac{t}{T})^l\big\} }.
\end{align*}
where with slight abuse of notations $\{\phi_{(k, l)}: k = 1, \ldots, \bar{n};\, l = 0, \ldots, d\}$ denotes the set of parameters.}
After sampling a revenue-ordered assortment according to the policy, we must consider the product availability at a specific state.
The set $\mathcal{J}(x) \coloneqq \{j \in \mathcal{J}: x \geq A^j\}$ represents the collection of products that are available for offering when the state is $x$.
The assortment we actually offer is the intersection of the sampled revenue-ordered assortment and $\mathcal{J}(x)$. 
In other words, the policy that is actually implemented is as follows:
for $(t, x, S) \in \mathbb{K}$,
\begin{align}\label{eq:RO_parametric_policy}
    \bpi^{\phi}(S \mid t, x ) = \sum_{k \in \{1, \ldots, \bar{n}\}:\ S^{[k]} \cap \mathcal{J}(x) = S} \bsigma^{\phi} (S^{[k]} \mid t).
\end{align}
We refer to the parameterization combining the value approximation \eqref{eq:continuous_approximated_value_function} and 
 the policy approximation \eqref{eq:RO_parametric_policy} as the ``Linear-RO" approach.

\emph{2-NNs}: as noted in Remark \ref{rk:neural_network}, we can employ neural-network-based approximation for value functions and policies, referred to as the critic and actor networks, respectively.
This parametrization facilitates broader application to various intensity control problems. 
In our experiments, we use fully connected networks with ReLU activation functions, which consist of multiple layers: the input layer, one or more hidden layers, and the output layer.
Specifically, the inputs to both the critic and actor networks are $(1+m)$-dimensional time-state vectors, where the first dimension represents time and the subsequent $m$ dimensions represent the state.
Note that the policy's outputs are $2^n$-dimensional, corresponding to the offering probabilities of $2^n$ different assortments.
Directly setting the output dimension of the actor network to $2^n$ would result in prohibitively high computational costs, even when $n$ is moderate.
To address the issue, we let the actor network output a $n$-dimensional vector. 
The vector is then transformed into a $(0, 1)$-valued $n$-dimensional vector using a sigmoid function $z \rightarrow \frac{e^z}{1+e^z}$. 
It is interpreted as the success probabilities for $n$ Bernoulli random variables. 
Sampling from these distributions yields an $n$-dimensional binary vector, where each element corresponds to a Bernoulli outcome.
Finally, we offer the assortment that consists of products corresponding to the positions with outcome $1$ in this binary vector.
If we denote the actor network as $L^{\phi}: (t, x) \mapsto \mathbb{R}^n$, then the aforementioned policy can be mathematically expressed as follows: for $(t, x, S) \in \mathbb{K}$,
\begin{align*}
    \bpi^{\phi}(S \mid t, x) = \prod_{j \in S} \frac{ e^{\frac{1}{\gamma}[L^{\phi}(t, x)]_j}}{1 + e^{\frac{1}{\gamma}[L^{\phi}(t, x)]_j}} \cdot \prod_{j \notin S} \frac{1}{1 + e^{\frac{1}{\gamma}[L^{\phi}(t, x)]_j}},
\end{align*}
where $\gamma>0$ is the temperature parameter. 
We refer to the neural-network-based approximation described above as the ``2-NNs" approach.

\begin{remark}[Constructing a simulator]
When running actor-critic algorithms, Algorithms \ref{alg:MC} - \ref{alg:NN}, we construct a simulator (Step~\ref{step:simulator}) even when the environment (the customer arrival rate and choice probabilities) is unknown.
We provide an overview of the process below.
The algorithm runs in cycles. 
In the first cycle, we initialize the environment simulator (with an arbitrary choice model) and the policy  
to generate $M$ episodes.
The data from these $M$ episodes is used in the PE and PG processes to update the policy once.
We then implement the updated policy in the simulator and perform another update.
We conduct $k_1$ updates within the simulator, generating $k_1 M$ episodes.
Then, the policy is executed in the real-world environment \emph{for the first time}.
This is only the first episode in the real world, but the $(k_1 M+1)$-th episode in our algorithms considering the simulated ones.
Moreover, the $k_1 M$ episodes are generated and PE/PG are computed offline so that the updated policy can be implemented before the start of the first episode in the real world.
In the real world, the same process of policy implementation and update is iterated $k_2$ times, generating $k_2 M$ episodes.
The collected data from the real world is then used to update the simulator by estimating a choice model, e.g., the MNL model.
In our study, we employ the maximum likelihood estimation (MLE) method for this estimation.
This concludes the first cycle consisting of $(k_1+k_2)M$ episodes, among which $k_2 M$ are collected from the real world.
Subsequently, the simulator with the estimated choice model and the updated policy will be re-implemented and the cycle repeats.
In our experiments, for the small and the medium-sized network, we set $M=10$, $k_1 = 8$, $k_2 = 2$; for the large network, we set $M=1$, $k_1 = 80$, $k_2 = 20$.
Such combination of online real-world data and offline simulated data is common in many RL applications.
This is because of two reasons: the real-world data is expensive while the simulated data is cheap; the estimation of the environment (choice probabilities) is relatively data-efficient while the learning of the optimal policy is data-hungry.
As a result, we use the real-world data to learn the unknown environment and the simulated data to facilitate the learning of the policy.
   
\end{remark}

\subsection{Benchmarks}\label{ssec:benchmarks}
We next introduce the benchmarks that have been studied in the literature to compare the performance of our proposed algorithm with. As several benchmarks are formulated in discrete time, we first describe the corresponding discrete-time model $\mathcal{M}_{\Delta t}$ as follows: The time horizon $[0, T]$ is discretized into a uniform grid $0 = t_0 < t_1 < \cdots < t_K = T$ with step size $\Delta t = \frac{T}{K}$. It is assumed that in each period $(t_k, t_{k+1}]$, one arrival occurs with probability $\lambda \Delta t$, and no arrivals occur with probability $ 1 - \lambda \Delta t$. 
\subsubsection{Optimal Policy from Discretized Dynamic Programming}\label{sec:dynamic_programming}
Given the challenges in solving the HJB equation \eqref{eq:hjb} within a continuous-time framework,
one can consider using finite difference approximation to replace derivative, thus deriving the following dynamic programming (DP) problem
\begin{align}\label{eq:discrete_time_dynamic_programming}
\left\{
    \begin{aligned}
    & V_{\Delta t}^*(t_k, x)= (\lambda \Delta t) \cdot
    \max_{S \in \mathcal{A}(x)}\bigg\{\sum_{j \in S} P_j(S)[p_j + V_{\Delta t}^*(t_{k+1}, x - A^j ) - V_{\Delta t}^*(t_{k+1}, x) ]\bigg\} + V_{\Delta t}^*(t_{k+1}, x), \\
    & \hspace{10cm} \forall\, k = 0, \ldots, K-1,\ x \in \mathcal{X} \\
    & V_{\Delta t}^*(t_{K}, x) = 0, \quad \forall\, x \in \mathcal{X}.
    \end{aligned}
\right.
\end{align}
It is worth noting that the DP problem \eqref{eq:discrete_time_dynamic_programming} corresponds exactly to the underlying dynamic program of the discrete-time model $\mathcal{M}_{\Delta t}$ described in Section \ref{sec:PE}.
Therefore, when $\Delta t$ is sufficiently small, $V_{\Delta t}^*(0, c)$ can serve as an reliable approximation to the true optimal expected revenue $V^*(0, c)$. This makes it a suitable benchmark for evaluating the performance of our proposed policy. 
When the size of the state space 
$\mathcal{X}$ is relatively small, we can solve the DP problem \eqref{eq:discrete_time_dynamic_programming} recursively to obtain $V_{\Delta t}^*(0, c)$. 
However, when the size of $\mathcal{X}$ is large, this approach becomes computationally infeasible.

\subsubsection{Other Policies}
In addition, we consider the following policies as benchmarks.
\begin{itemize}
    \item \textsc{Uniform-Random}: 
    Denoted by $\bpi^{\textsc{Uniform-Random}}$, this
    is a stationary policy that selects among all available offer sets at each given state with equal probability. That is, 
    $\bpi^{\textsc{Uniform-Random}} (S \mid x) = \frac{1}{| \mathcal{A}(x)|}$ for all $x \in \mathcal{X}$ and $S \in \mathcal{A}(x)$.

    \item \textsc{Greedy}: 
    The \textsc{Greedy} policy 
    is a deterministic policy that always selects the offer set with the highest expected revenue from all available offer sets at each state, that is, $\argmax_{S \in \mathcal{A}(x)} \sum_{j \in S} p_j P_j(S)$.

    \item \textsc{CDLP}: 
    This policy is first introduced in \cite{liu2008choice} based on the following choice-based deterministic linear programming (CDLP) model
    \begin{align*}
        z_{\mathrm{CDLP}}=\max_h\  \sum_{S \in \mathcal{A}} \lambda R(S) h(S) \quad
        \text{ s.t. } \ & \sum_{S \in \mathcal{A}} \lambda Q_i(S) h(S) \leq c_i, \quad i = 1, \ldots, m \\
        & \sum_{S \in \mathcal{A}} h(S) \leq T \\
        & h(S) \geq 0, \quad S \in \mathcal{A},
    \end{align*}
    where $h(S)$ denotes the total duration that the subset $S$ is offered, $R(S) = \sum_{j \in S}  p_j P_j(S)$ denotes the expected revenue from offering $S$ to an arriving customer, and $Q_i(S) = \sum_{j = 1}^{n} A_{ij} P_j(S)$ denotes the rate of using a unit of capacity on resource $i$ when $S$ is offered.
    The CDLP policy executes by sequentially offering each set $S$ for the duration specified by the optimal solution to the CDLP, following the lexicographic order of the variable indices.
    
    In addition, it has been demonstrated in \cite{liu2008choice} that the optimal objective function value of the CDLP provides an upper bound on the optimal expected revenue $V^*(0, c)$ in the stochastic problem.
    It is worth mentioning that although the discussions in \cite{liu2008choice} are based on a discrete-time model, the conclusion regarding the upper bound property can be readily extended to our continuous-time framework.

    \item \textsc{ADP}: 
    This policy is derived from the approximate dynamic programming (ADP) approach detailed in \cite{zhang2009approximate}.
    Since the ADP approach operates in a discrete-time setting, 
    we apply it to the discrete-time model $\mathcal{M}_{\Delta t}$ described in Section \ref{sec:PE}.
    That is, to consider an affine functional approximation ${V_{\text{ADP-}\Delta t} (t_k, x)} = \theta_{k} + \sum_{i=1}^m V_{k,i} x_i$ to the optimal value function $V_{\Delta t}^* (t_k, x)$ and then formulates the dynamic program \eqref{eq:discrete_time_dynamic_programming} as a linear program:
    \begin{equation}\label{eq:ADP_optimization_problem}
        \begin{aligned}
        \min_{\theta,\, V}\ \theta_{0} + \sum_i V_{0, i} c_i \quad
        \text{ s.t. } \ &\theta_{k} - \theta_{k+1} + \sum_i (V_{k, i} - V_{k+1, i}) x_i \geq \lambda (t_{k+1} - t_k) \sum_{j \in S} P_{j}(S) [p_j - \sum_{i=1}^m V_{k+1, i} A_{ij}], \\
        &\hspace{5.5cm} k = 0, \ldots, K-1,\ x \in \mathcal{X},\ S \in \mathcal{A}(x) \\
        & \theta_K = 0, \ V_{K, i} = 0.
        \end{aligned}
    \end{equation}
    With the optimal solution $(\theta^*, V^*)$ for the above linear program, the ADP policy is constructed by  offering $\arg\max_{S \in \mathcal{A}(x)} \sum_{j \in S} P_j(S)[p_j - \sum_{i=1}^m V_{k+1, i}^* A_{ij}] $ in time period $(t_k, t_{k+1}]$ and state $x$.
    Given that the ADP policy varies with different degrees of time discretization, we will refer to the ADP policy under the discrete-time model $\mathcal{M}_{\Delta t}$ as ADP-$\Delta t$.
\end{itemize}

Note that all the benchmarks mentioned above, except for the uniform random policy, require the knowledge of environment, specifically the customer arrival rate $\lambda$ and the choice probabilities $P_j(S)$.
Thus we provide the exact values of the parameters to these policies.
Therefore, their performance is slightly inflated compared to our RL algorithm which does not require knowing the values of such parameters.

In the subsequent experiments, the performance of a policy is evaluated using the average revenues obtained from $10,000$ sample paths. Meanwhile, the corresponding 99\% confidence interval is provided.
All experiments in this paper were conducted on a 64-bit Linux operating system equipped with
two AMD EPYC 7413 (Zen 3) CPUs @ 2.65 GHz, and four NVIDIA A100-SXM4 GPUs, each with 40 GB of memory.
Gurobi version 10.0.2 was employed for solving CDLP and ADP.
To enhance computational efficiency, we have fully vectorized the tasks in our algorithm, including sampling and basic arithmetic operations.
By employing PyTorch for these tensor-based, vectorized computations, we leverage the inherent parallel computing capabilities of GPUs.
More precisely, we parallel the simulation of a batch of $M$ trajectories for the resulting policy update, i.e., Step 5 to 16 in Algorithm~\ref{alg:MC} - \ref{alg:NN} for $ i\neq 0 \pmod{M}$.
The simulated batch data is then structured into multi-dimensional matrices (tensors in PyTorch), facilitating vectorized operations for updating $\theta$ and $\phi$ during Steps 18 and 19 in Algorithm~\ref{alg:MC} - \ref{alg:NN}.


\subsection{Experiment One: A Small Network}\label{sec:small-net}
We consider a simple example featuring 2 resources and 3 products. 
The consumption is captured by matrix $A = \begin{bmatrix}
    1 & 0 & 1 \\
    0 & 1 & 1
\end{bmatrix}$. 
The initial stocks for the 2 resources are set to $c = (5,\, 5)^\top$ and the price for the 3 products are fixed at $p = (1,\, 1,\, 1.5)^\top$. 
We set the booking horizon to $T = 15$, during which customer arrivals are modeled by a Poisson process with a rate of $\lambda = 0.9$.
The choice probabilities are determined by the weights vector $v = (v_0,\, v_1,\, v_2,\, v_3) = (27.8,\, 42,\, 42,\, 55)$ through the multinomial logit (MNL) choice model, specified as
$P_j(S) = \frac{v_j}{v_0 + \sum_{j \in S} v_j}$ for all $S\subseteq \mathcal{J}$ and $j \in S$.
For this example, we implement Algorithm \ref{alg:MC} with Linear-Pair parametrization.
The degree of the polynomial is set to $d=2$ in \eqref{eq:continuous_approximated_value_function} and \eqref{eq:simplified_parametric_policy}, and all components of $\phi$ are initialized to zero, which is equivalent to the uniform random policy.
The hyperparameters are configured as follows: batch size $M=10$, entropy factor $\gamma = 2\times 10^{-3}$ and learning rate $\alpha_{\phi} = 1 \times 10^{-5}$. 
Figure \ref{fig:3_pros_performance_evolution} illustrates how the average revenues of the updated policies varies throughout the learning process.
Upon completing the final update, the simulated average revenues of the resulting policy, achieves a value of $8.835$.

\begin{figure}[htbp]
    \centering
    \begin{tikzpicture}
    \begin{axis}[
        xlabel={Episode},
        ylabel={Average Revenue},
        xtick={0,20000,40000,60000,80000,100000},
        xticklabels={0,2,4,6,8,10},
        xtick scale label code/.code={$\times 10^4$},
        ytick={7.5, 8, 8.5, 8.934},
        yticklabel style={
            /pgf/number format/fixed,
            /pgf/number format/precision=5
        },
        scaled y ticks=false,
        xmin=0, xmax=100000,
        ymin=7.5, ymax=9,
        legend pos=south east,
        grid=minor
    ]
    \addplot[thick, blue] table [col sep=comma, x=x, y=y] {data/small_network.csv};
    \addlegendentry{Linear-Pair}
    \addplot [dashed, very thick, samples=2] coordinates {(0,8.934) (100000,8.934)};
    \end{axis}
    \end{tikzpicture}
    \caption{Average revenues of Algorithm~\ref{alg:MC} over episodes for the example in Section~\ref{sec:small-net}}
    \label{fig:3_pros_performance_evolution}
\end{figure}

Next, we solve the DP problem \eqref{eq:discrete_time_dynamic_programming}.
Setting $\Delta t = 0.001$ yields a highly reliable approximation of the optimal expected revenue, $V_{0.001}^*(0, c) = 8.934$, as indicated in Figure \ref{fig:3_pros_performance_evolution}.
Moreover, we implement the benchmarks outlined in Section \ref{ssec:benchmarks}, including the \textsc{Uniform-Random}, \textsc{Greedy}, CDLP, and ADP policies.
Table \ref{tab:3_pros_expected_revenue} reports the simulated average revenues for our Linear-Pair policy and the benchmarks.
To make it easy to compare the performance, we also provide the ratio of the simulated average revenues to $V_{0.001}^*(0, c)$.
The numerical results indicate that, our algorithm achieves 98.89\% of the optimal performance, outperforming all other benchmarks other than dynamic programming.
Moreover, in the considered small network, ADP policies maintain relatively stable performance across different degrees of time discretization and consistently outperform other benchmarks.
Despite this, our Linear-Pair policy still exhibits a slight advantage over the best-performing ADP-0.05 policy.

\begin{table}[htbp]
    \centering
    \caption{Simulation results for selected policies in Section~\ref{sec:small-net}}
    \label{tab:3_pros_expected_revenue}
    \begin{tabular}{p{3.5cm}>{\raggedleft\arraybackslash}p{2.5cm}>{\raggedleft\arraybackslash}p{2cm}>{\raggedleft\arraybackslash}p{3cm}>{\raggedleft\arraybackslash}p{3cm}}
    \hline
    Policy & Average revenue & $99\%$-CI ($\pm$) & {$\frac{\text{Average revenue}}{V_{0.001}^*(0, c)}$ ($\%$)} &  Time cost (second)\\
    \hline
    Linear-Pair &  8.835 &  0.037 & 98.89 & 1,037.82\\
    \textsc{Uniform-Random}  & 7.589 & 0.038 & 84.95 & $*$\\
    \textsc{Greedy} & 8.483 & 0.023 & 94.95 & $*$\\
    CDLP & 8.545 & 0.045 & 95.65 & 0.15\\
    ADP-1 & 8.751 & 0.042 & 97.95 & 1.07  \\
    ADP-0.5 & 8.741 & 0.042 & 97.84 & 1.23 \\
    ADP-0.2 & 8.734 & 0.042 & 97.76 & 1.50 \\
    ADP-0.1 & 8.737 & 0.042 & 97.79 & 2.23 \\
    ADP-0.05 & 8.759 & 0.042 & 98.04 & 5.50 \\
    \hline
    \end{tabular}
    
    \footnotesize
    $*$ indicates that the method involves neither solving LP problems nor learning; the time cost is virtually zero.
\end{table}

\subsection{Experiment Two: A Medium-Sized Airline Network}\label{sec:mid-net}

\begin{figure}[htbp]
    \centering
    \begin{tikzpicture}[
    >=Stealth,
    node distance=4cm,
    every node/.style={font=\scriptsize},
    city/.style={
        circle,
        draw=black,
        thick,
        minimum size=1.3cm,
        inner sep=0pt
        }
    ]
    \node[city, font=\large] (A) {A};
    \node[city, right=of A, font=\large] (H) {H};
    \node[city, right=of H, font=\large] (B) {B};
    \draw[->, thick] 
        ($(A.east)+(0,0.5)$) -- ($(H.west)+(0,0.5)$)
        node[pos=0.5, xshift=1pt,yshift=5pt] {Leg 1 (12; morning)};
    \draw[->, thick] 
        ($(A.east)+(0,0)$) -- ($(H.west)+(0,0)$)
        node[pos=0.5, xshift=2.8pt, yshift=5pt] {Leg 2 (20; afternoon)};
    \draw[->, thick] 
        ($(A.east)+(0,-0.5)$) -- ($(H.west)+(0,-0.5)$)
        node[pos=0.5, yshift=5pt] {Leg 3 (16; evening)};
    \draw[->, thick] 
        ($(H.east)+(0,0.55)$) -- ($(B.west)+(0,0.55)$)
        node[pos=0.5, xshift=1pt, yshift=5pt] {Leg 4 (20; morning)};
    \draw[->, thick] 
        ($(H.east)+(0,0)$) -- ($(B.west)+(0,0)$)
        node[pos=0.5, xshift=2.8pt, yshift=5pt] {Leg 5 (12; afternoon)};
    \draw[->, thick] 
        ($(H.east)+(0,-0.55)$) -- ($(B.west)+(0,-0.55)$)
        node[pos=0.5, yshift=5pt] {Leg 6 (16; evening)};
    \end{tikzpicture}
    \caption{Airline network for the example in Section~\ref{sec:mid-net}}
    \label{fig:9_pros_airline_network}
\end{figure}
In this example, we consider a medium-sized airline network consisting of 6 flight legs (resources) with a total of 9 itineraries (products).
Figure \ref{fig:9_pros_airline_network} presents the airline network, with each leg labeled, for example, ``Leg 1 (12; morning)'' represents a morning flight with an initial capacity of 12 seats, and so forth for the remaining legs.
For customer arrivals and the choice model, we consider three segments depending on the origin-destination pairs, shown in Table~\ref{tab:9_pros_Segments_and_Characteristics}.
For example, the arrival rates for segment $l$ is $\lambda_l$ and the total arrival rate is $\lambda = \sum_{l=1}^3 \lambda_l = 0.8$.
The choice probabilities of segment $l$ follow the MNL model:  
$P_j(S) =
\frac{\lambda_l}{\lambda} \cdot \frac{v_{lj}}{v_{l0} + \sum_{j'} v_{lj'}}$.
The selling horizon is $T=200$.

\begin{table}[htbp]
    \centering
    \caption{Segments and their characteristics}\label{tab:9_pros_Segments_and_Characteristics}
    \begin{tabular}{p{4.5cm}>{\centering\arraybackslash}p{3cm}>{\centering\arraybackslash}p{3cm}>{\centering\arraybackslash}p{4cm}}
        \hline
        Segment $l$ & 1 & 2 & 3 \\
        \hline
        Description & A $\rightarrow$ H & H $\rightarrow$ B & A $\rightarrow$ H $\rightarrow$ B \\
        Arrival rate $\lambda_l$ & 0.2 & 0.2 & 0.4 \\
        Consideration set $\mathcal{J}_l$ & $\{1,\, 2,\, 3\}$ & $\{4,\, 5,\, 6\}$ & $\{7,\, 8,\, 9\}$ \\
        Legs & $\{1,\, 2,\, 3\}$ & $\{4,\, 5,\, 6\}$ & $\{\{1,\, 4\},\, \{2,\, 5\},\, \{3,\, 6\}\}$ \\
        Price & $(8,\, 10,\, 6)$ & $(8,\, 10,\, 6)$ & $(9,\, 12,\, 7)$ \\
        Preference weights $v_{lj}$ & $(5,\, 10,\, 1)$ & $(5,\, 10,\, 1)$ & $(5,\, 1,\, 10)$ \\
        No-purchase weight $v_{l0}$ & 1 & 1 & 5 \\
        \hline
    \end{tabular}
\end{table}
In this example, we apply Algorithm \ref{alg:MC} with both Linear-Pair and Linear-RO approximations, and Algorithm~\ref{alg:NN} with 2-NNs parametrization, to demonstrate the impact of different parametric forms.
The configurations related to the network structures are as follows: For all three experiments, the batch size $M$ is set to $10$. 
In the Linear-Pair and Linear-RO experiments, the degree of the polynomials in both critic and actor approximations is set to $d=3$, and the initial values for all components of $\phi$ are set to zero.
In the 2-NNs experiment, both the actor and critic networks are configured with four hidden layers, with respective widths of $150$, $120$, $80$ and $20$. Moreover, Table~\ref{tab:9_pros_hyperparameter} presents the well-performing hyperparameters used in these three experiments.
Figure~\ref{fig:9_pros_performance_evolution} illustrates the convergence of the policies during the experiments. 
{Table~\ref{tab:9_pros_expected_revenue} reports the average revenues of the final policies and relative performance of each policy compared to the Linear-Pair policy, along with the respective total time costs.}
The Linear-Pair approximation exhibits faster convergence compared to the 2-NNs experiment, because in each iteration, Algorithm \ref{alg:MC} exactly solves the approximate value function, whereas Algorithm \ref{alg:NN} updates the value function network incrementally with just one gradient step.
The average revenues of the neural-network-based policy is slightly lower than that of the Linear-Pair policy. 
This might be due to the simplifications made during the construction of the actor neural network, as detailed in the initial part of Section \ref{sec:Numerical_Experiments}.
The suboptimal performance of the Linear-RO policy is similarly due to its relatively small policy class.
\begin{table}[tbp]
    \centering
    \caption{Well-performing hyperparameters for different parametrization in Section~\ref{sec:mid-net}}
    \label{tab:9_pros_hyperparameter}
    \begin{tabular}{p{3cm}>{\centering\arraybackslash}p{2.5cm}>{\centering\arraybackslash}p{2.5cm}>{\centering\arraybackslash}p{2.5cm}}
    \hline
    & Entropy factor $\gamma$ & Learning rate $\alpha_{\theta}$ & Learning rate $\alpha_{\phi}$ \\
    \hline
    Linear-Pair & $2 \times 10^{-3}$ & - & $3 \times 10^{-6}$ \\
    Linear-RO & $1 \times 10^{-3}$ & - & $5 \times 10^{-6}$ \\
    2-NNs & $1 \times 10^{-3}$ & $8 \times 10^{-2}$ & $2 \times 10^{-7}$ \\
    \hline
    \end{tabular}
\end{table}


\begin{figure}[tbp]
    \centering
    \begin{tikzpicture}
    \begin{axis}[
        xlabel={Episode},
        ylabel={Average Revenue},
        xmin=0, xmax=100000,
        ymin=570, ymax=680,
        xtick={0,20000,40000,60000,80000,100000},
        xticklabels={0,2,4,6,8,10},
        xtick scale label code/.code={$\times 10^4$},
        legend pos=south east,
        grid=minor
    ]
    \addplot[thick, blue] table [x=x, y=y1, col sep=comma] {data/medium_network.csv};
    \addlegendentry{Linear-Pair}
    \addplot[thick, brown] table [x=x, y=y2, col sep=comma] {data/medium_network.csv};
    \addlegendentry{Linear-RO}
    \addplot[thick, red] table [x=x, y=y3, col sep=comma] {data/medium_network.csv};
    \addlegendentry{2-NNs}
    \end{axis}
    \end{tikzpicture}
    \caption{Average revenues of Algorithm~\ref{alg:MC} and Algorithm~\ref{alg:NN} over episodes for the example in Section~\ref{sec:mid-net}} \label{fig:9_pros_performance_evolution}
\end{figure}

\begin{table}[tbp]
    \centering
    \caption{Simulation results for selected policies in Section~\ref{sec:mid-net}}
    \label{tab:9_pros_expected_revenue}
    \begin{tabular}{p{3.5cm}>{\raggedleft\arraybackslash}p{2.5cm}>{\raggedleft\arraybackslash}p{2cm}>{\raggedleft\arraybackslash}p{4cm}>{\raggedleft\arraybackslash}p{3cm}}
    \hline
    Policy & Average revenue & $99\%$-CI ($\pm$)  & {$\frac{\text{Average revenue}}{\text{Average revenue of Linear-Pair}}$ ($\%$)} & Time cost (second) \\
    \hline
    Linear-Pair & 677.356 & 0.753 & 100.00 & 2,774.59 \\
    Linear-RO & 673.595 & 0.792 & 99.44 & 2,366.73 \\
    2-NNs & 676.369 &  0.646 & 99.85 & 2,184.20  \\
    \textsc{Uniform-Random} & 595.603 & 0.845 & 87.93 & $*$ \\
    \textsc{Greedy} & 605.402 & 0.375 & 89.38 & $*$ \\
    CDLP & 651.294 & 1.164 & 96.15 & 0.39 \\
    ADP-1 & 675.566 & 0.684 & 99.74 & 35.70 \\
    ADP-0.5 & 560.899 & 1.457 & 82.81 & 88.30 \\
    ADP-0.2 & 664.745 & 0.966 & 98.14 & 374.22\\
    ADP-0.1 & 640.854 & 1.239 & 94.61 & 1,013.94 \\
    ADP-0.05 & 675.647 & 0.657 & 99.75 & 3,066.90\\
    \hline
    \end{tabular}
    
    \footnotesize
    $*$ indicates that the method involves neither solving LP problems nor learning; the time cost is virtually zero.
\end{table}

We also implement benchmark policies for comparison. Due to the problem's state space being excessively large, approximately $10^6$, dynamic programming approaches become infeasible.
We therefore implement other benchmarks outlined in Section \ref{ssec:benchmarks},
with the results reported in
Table \ref{tab:9_pros_expected_revenue}.
It shows that the Linear-Pair policy outperforms simple heuristics such as \textsc{Uniform-Random} and \textsc{Greedy} by more than {10\%}, and also shows a {3.8\%} competitive edge against the more advanced CDLP algorithm.
Although the performance of ADP-0.05 is comparable to that of Linear-Pair policy, we observe that for this medium-sized network, the performance of ADP policies is highly sensitive to the level of time discretization.
When using a suboptimal discretization level, such as $\Delta t = 0.5$, our algorithm can outperform ADP-0.5 by a margin of up to {17.2\%}.
This further demonstrates the advantage of our algorithm, as it operates within a continuous-time framework and avoids issues with upfront time discretization.
Moreover, 
the CDLP method still offers a theoretical upper bound on the optimal expected revenue, as mentioned in Section ~\ref{ssec:benchmarks}. In this example, the upper bound is $708$. 
Therefore, it can be inferred that the performance of the Linear-Pair policy is within a 5\% gap from the performance of the optimal policy.

\begin{remark}[The instability of ADP time discretization]\label{rmk:adp-interpolation}
To explore the source of instability of time discretization in ADP, we have conducted supplementary experiments using interpolation.
Recall that for a given $\Delta t$, solving the LP problem \eqref{eq:ADP_optimization_problem} yields a linear value function approximation $V_{\text{ADP-}\Delta t} (t_k, x) = \theta^*_{k} + \sum_{i=1}^m V^*_{k,i} x_i$, which is exclusively defined on the time grid $0 = t_0 < t_1 < \cdots < t_K = T$.
To utilize the above value function approximation in a continuous manner, we employ the following interpolation technique: Denote 
\begin{align*}
    V_{\text{I-ADP-}\Delta t} (t, x) \coloneqq \Big[\frac{\theta^*_{k+1} - \theta^*_k}{\Delta t} (t - t_k) + \theta_k^*\Big] +  \sum_{i=1}^m \Big[\frac{V^*_{k+1, i} - V^*_{k,i}}{\Delta t} (t - t_k) + V^*_{k,i}\Big] x_i, \quad t \in (t_k, t_{k+1}]
\end{align*}
as the interpolated approximate value function.
With this, we can obtain a continuous-time policy, referred to as
I-ADP-$\Delta t$, by offering 
\begin{align*}
    &\arg\max_{S \in \mathcal{A}(x)} \sum_{j \in S} P_j(S)[p_j - (V_{\text{I-ADP-}\Delta t} (t, x) - V_{\text{I-ADP-}\Delta t} (t, x - A^j))] \\ 
    = &\arg\max_{S \in \mathcal{A}(x)} \sum_{j \in S} P_j(S)\bigg\{p_j - \sum_{i=1}^m \Big[\frac{V^*_{k+1, i} - V^*_{k,i}}{\Delta t} (t - t_k) + V^*_{k,i}\Big] A_{ij}\bigg\}
\end{align*}
at time $t$ and state $x$.
Table \ref{tab:9_pros_ADP_interpolation} shows that the performance of the interpolation-based ADP policies is almost consistent with that of the original ADP policies, which implies the instability stems from other sources than the coarseness of the decision epochs.
Additional experiments suggest one potential issue: ADP requires the solution to the large-scale LP \eqref{eq:ADP_optimization_problem}.
When $\Delta t$ is small, \eqref{eq:ADP_optimization_problem} presents a greater computational challenge than CDLP.
As a result, the {Gurobi Optimizer} uses column generation to obtain an approximate solution to \eqref{eq:ADP_optimization_problem} with a prespecified optimality tolerance.
When $\Delta t$ changes slightly, there is no guarantee that the solution to \eqref{eq:ADP_optimization_problem}, $\theta_k$ and $V_{k,i}$, are continuous.
{Indeed, we provide the experimental results that capture the variation of the coefficients $V_{0,i}$ for $i=1, \ldots, 6$, selected from the ADP solution under five different levels of time discretization near $\Delta t = 0.5$ in Appendix \ref{app:ADP_solution_time_discretization}.}
The discontinuity of $V_{\text{I-ADP-}\Delta t} (t, x)$ in $\Delta t$ is likely due to the fact that we do not have control over how the (approximate) solution to the LP \eqref{eq:ADP_optimization_problem} is obtained from column generation.
We believe this is the key reason behind the instability of the ADP approach. 
Note that this issue of ADP is unique to the continuous-time problem with na\"ive discretization.
\begin{table}[htbp]
    \centering
    \caption{Simulation results for interpolation-based ADP policies in Section~\ref{sec:mid-net}}
    \label{tab:9_pros_ADP_interpolation}
    \begin{tabular}{p{2cm}>{\raggedleft\arraybackslash}p{2.5cm}>{\raggedleft\arraybackslash}p{2cm}}
    \hline
    Policy & Average revenue & $99\%$-CI ($\pm$) \\
    \hline
    I-ADP-1 & 674.873 & 0.711 \\
    I-ADP-0.5 & 559.358 & 1.467 \\
    I-ADP-0.2 & 664.400 & 0.946 \\
    I-ADP-0.1 & 641.149 & 1.257 \\
    I-ADP-0.05 & 675.288 & 0.674 \\
    \hline
    \end{tabular}
\end{table}
\end{remark}

\subsection{Experiment Three: A Large Network}\label{sec:large-net}
We consider a network with $m = 100$ resources and $n = 200$ products,
each resource having an initial capacity of $c_i = 10$ for $i = 1, \ldots, 100$.
This configuration results in a problem with a state space of size $11^{100}$ and an action space of size $2^{200}$, mirroring the scale of real-world problems. 
The selling horizon is set to $T = 2,000$.
The consumption matrix $A_{100 \times 200}$ is formed by concatenating two square matrices, $B_{100 \times 100}$ and $\bar{B}_{100 \times 100}$, such that $A = [B \ \bar B]$, where matrices $B$ and $\bar{B}$ are generated from the identity matrix $I_{100}$, with 50 and 150 off-diagonal positions respectively set to 1 through random selection. 
Product prices $p_1, \ldots, p_{100}$ are drawn from a uniform distribution with range $[50, 250]$, while prices $p_{101}, \ldots, p_{200}$ are draw from a uniform distribution with range $[250, 500]$.
We assume that the customer choice behavior follows the example in Section~\ref{sec:mid-net}. 
More precisely, the $200$ products are categorized into $10$ segments,
with each product being assigned to one of the segments with equal probability. 
For each segment $l$, the preference weight $v_{lj}$ for the first 100 products $(j=1, \ldots, 100)$ are drawn from the Poisson distribution with mean $50$, while those for the latter $100$ products $(j=101, \ldots, 200)$ are drawn from the Poisson distribution with mean $100$.
The no-purchase weight $v_{l0}$ for each segment $l$ is drawn from the Poisson distribution with mean $10$.
We set the total arrival rate $\lambda$ to $0.8$. To randomly generate the arrival rate $\lambda_l$ for each customer segment $l$, we draw $E_l$ from the Poisson distribution with mean $10$ for each $l$ and set $\lambda_l = 0.8 \frac{E_l}{\sum_{l'= 1}^{10} E_{l'}}$.
Therefore, given an assortment $S$ and a product $j \in \mathcal{J}_l \cap S$, the choice probability is calculated as $P_{j}(S) = \frac{\lambda_l}{\lambda} \cdot \frac{v_{lj}}{v_{l0} + \sum_{j' \in \mathcal{J}_{l} \cap S} v_{lj'}}$.

In this example, we implement Algorithm \ref{alg:NN} with 2-NNs parametrization. 
Both the actor and critic networks are configured with four hidden layers, each having widths of $150$, $120$, $80$, and $20$, respectively.
Given the memory limitations of the GPU, we have reduced the batch size to $M=1$ for this large-scale problem.
The hyperparameters for this model are set as follows: entropy factor $\gamma = 1\times 10^{-4}$, learning rate $\alpha_{\theta} = 1 \times 10^{-3}$, and learning rate $\alpha_{\phi} = 5 \times 10^{-8}$. 
The other two approximation schemes are computationally infeasible because the policy space is too large to be stored or optimized.
Even for Linear-RO, the number of revenue-ordered assortments grows exponentially in the number of segments.
Figure~\ref{fig:200_pros_performance_evolution} illustrates how the average revenues of the updated policies varies throughout the learning process.
The dashed line marked in the figure represents an upper bound of the optimal value for the problem, obtained by solving the CDLP.
{For the ADP method, we cannot solve the (dual of) LP \eqref{eq:ADP_optimization_problem} even with column generation, due to the huge number of constraints}.
Therefore, we only implement other benchmark policies for comparison.
Table \ref{tab:200_pros_expected_revenue} displays the average revenues of the final policy from the 2-NNs experiment and the benchmark policies, \textsc{Uniform-Random}, \textsc{Greedy} and CDLP, along with the calculated ratios relative to the upper bound of $197,785$.
The 2-NNs policy exhibits a small performance gap of $0.13\%$ from the upper bound, indicating that it nearly achieves the optimal solution. 
The computational time for the 2-NNs experiment was approximately $33.6$ hours, which is reasonable given the scale of the example.

\begin{figure}[htbp]
    \centering
    \begin{tikzpicture}
    \begin{axis}[
        xlabel={Episode},
        ylabel={Average Revenue},
        xtick={0,20000,40000,60000,80000,100000},
        xticklabels={0,2,4,6,8,10},
        xtick scale label code/.code={$\times 10^4$},
        ytick={0,140000,160000,180000,197785},
        yticklabel style={
            /pgf/number format/fixed,
            /pgf/number format/precision=5
        },
        scaled y ticks=false,
        xmin=0, xmax=100000,
        ymin=140000, ymax=200000,
        legend pos=south east,
        grid=minor
    ]
    \addplot[very thick, blue] table [col sep=comma, x=x, y=y] {data/large_network.csv};
    \addlegendentry{2-NNs}
    \addplot [dashed, very thick, samples=2] coordinates {(0,197785) (100000,197785)};
    \end{axis}
    \end{tikzpicture}
    \caption{Average revenues of Algorithm~\ref{alg:NN} over episodes for the example in Section~\ref{sec:large-net}}
    \label{fig:200_pros_performance_evolution}
\end{figure}

\begin{table}[htbp]
    \centering
    \caption{Simulation results for selected policies in Section~\ref{sec:large-net}}
    \label{tab:200_pros_expected_revenue}
    \begin{tabular}{p{3.5cm}>{\raggedleft\arraybackslash}p{2.5cm}>{\raggedleft\arraybackslash}p{2cm}>{\raggedleft\arraybackslash}p{3cm}>{\raggedleft\arraybackslash}p{3cm}}
        \hline
        Policy  &  Average revenue & $99\%$-CI ($\pm$) & {$1 - \frac{\text{Average revenue}}{\text{Upper bound}}$ ($\%$)} & Time cost (second)\\
        \hline
        2-NNs & 197,533 & 5.192 & 99.87 & 120,836.50 \\
        \textsc{Uniform-Random} & 141,260 & 78.462 & 71.42 & $*$ \\
        \textsc{Greedy} & 161,205 & 61.465 & 81.51 & $*$ \\
        CDLP & 173,502 & 113.694 & 87.72 & 29.64 \\ 
        \hline
        \end{tabular}
        
    \footnotesize
    $*$ indicates that the method involves neither solving LP problems nor learning; time cost is virtually zero.
\end{table}

\subsection{Experiment Four: Continuous-Time vs. Discrete-Time RL Methods}\label{sec:CT-vs-DT} 
While the preceding experiments examine the scalability of our approach across different network sizes, we now illustrate its structural advantage relative to discretization-based RL methods through an example with highly non-stationary arrivals. Specifically, we consider a network featuring $3$ resources and $7$ products. The consumption is captured by matrix 
$A = \begin{bmatrix} 1 & 1 & 0 & 1 & 1 & 0 & 0 \\ 1 & 1 & 1 & 0 & 0 & 1 & 0 \\ 1 & 0 & 1 & 1 & 0 & 0 & 1 \end{bmatrix}$. 
The initial stocks for the $3$ resources are set to $c = (4,\, 4,\, 4)^{\top}$ and the prices for the $7$ products are fixed at $p = (800,\, 100,\, 100,\, 100,\, 10,\, 10,\, 10)^\top$.
Customer choices follow the MNL model, with the no-purchase weight $v_0 = 1$ and the product preference weights $(v_1, \dots, v_7) = (0.02,\, 1,\, 1,\, 1,\, 10,\, 10,\, 10)$.
The booking horizon is set to $T=10$, during which the arrival rate is maintained at $\lambda_0 = 0.5$, except for a sudden surge to $\lambda_1 = 50$ within the time window $[7.5,\, 8.0]$.

In this example, we adopt  the discrete-time Advantage Actor-Critic (A2C) algorithm as the benchmark, which is the synchronous variant of the widely adopted A3C algorithm \citep{mnih2016asynchronous}.
We select the A2C algorithm because its underlying architecture closely aligns with our continuous-time algorithm, ensuring a fair comparison that isolates the impact of time discretization.
Following the discretization procedure with time step $\Delta t$ as detailed in Section~\ref{sec:structural-advantage}, the A2C algorithm is implemented as follows. The overall objective is to minimize a composite loss function: $L_{total}(\theta, \phi) =  \beta L_{value}(\theta) + L_{policy}(\phi) + \gamma L_{entropy}(\phi)$. 
The value loss $L_{value}(\theta)$ is computed using empirical Monte Carlo returns:
\begin{align*}
    L_{value}(\theta) = \frac{1}{2}\E \left[\sum_{k=0}^{K-1} \bigg(\sum_{k'=k}^{K - 1} r_{(t_{k'}, t_{k'+1}]}^{\bpi^{\phi}} - J^{\theta}_{\Delta t}(t_k, X_{t_k}^{\bpi^{\phi}})\bigg)^2\right].
\end{align*}
The policy loss $L_{policy}(\phi)$ and the entropy loss $L_{entropy}(\phi)$ are defined as
\begin{align*}
    L_{policy}(\phi) = - \E \left[\sum_{k=0}^{K-1} \log \bpi^{\phi}(S_{t_k}^{\bpi^{\phi}} \mid t_k, X_{t_k}^{\bpi^{\phi}}) \widehat{A}_{t_k}^{\bpi^{\phi}}\right] \quad \text{ and } \quad L_{entropy}(\phi) = -\E\left[\sum_{k=0}^{K-1} \mathcal{H}(\bpi^{\phi}(\cdot \mid t_k, X_{t_k}^{\bpi^{\phi}}))\right],
\end{align*}
where the advantage $\widehat{A}_{t_k}^{\bpi^{\phi}}$ is estimated via the 1-step TD error: $\widehat{A}_{t_k}^{\bpi^{\phi}} = r_{(t_{k}, t_{k+1}]}^{\bpi^{\phi}} + J^{\theta}_{\Delta t}(t_{k+1}, X_{t_{k+1}}^{\bpi^{\phi}}) - J^{\theta}_{\Delta t}(t_k, X_{t_k}^{\bpi^{\phi}})$.
We note that, under a linear parameterization $J^{\theta}_{\Delta t}(t_k, x) = \theta^\top \varphi(t_k, x)$, the minimization of the value loss $L_{value}(\theta)$ admits a closed-form solution $\theta^*$, which is analogous to the result derived for the continuous-time approach in \eqref{eq:PE_MC}.
Accordingly, when implementing the A2C algorithm with both Linear-Pair and Linear-RO approximations, we directly compute the optimal parameter $\theta^*$ and omit the value loss term $L_{value}(\theta)$  from the overall objective, in a manner analogous to Algorithm~\ref{alg:MC}.

We compare the performance of our continuous-time algorithm (CT) against the discrete-time A2C algorithm (DT-$\Delta t$) across three different approximation schemes: Linear-Pair, Linear-RO, and 2-NNs. For the A2C benchmark, we examine two discretization levels: $\Delta t = 0.5$ and $\Delta t = 0.05$.}
To ensure a fair comparison, the CT and DT-$\Delta t$ algorithms share identical network configurations under each given approximation scheme. 
The specific configurations are as follows: For all experiments, each update is performed using a batch of $M=10$ episodes. 
In the Linear-Pair and Linear-RO experiments, the degree of the polynomials in both critic and actor approximations is set to $d=2$, and the initial values for all components of $\phi$ are set to zero.
In the 2-NNs experiments, both the actor and critic networks are configured with two hidden layers, each having a width of $64$. 
Table~\ref{tab:DT-CT-revenue} reports the average revenues of the final policies for the three algorithms across the three approximation schemes.
These results are obtained after $10^6$ training episodes and evaluated under their respective well-performing hyperparameters. Table~\ref{tab:DT-CT-time} provides the corresponding computational time cost. 
The proposed continuous-time algorithm consistently outperforms both the DT-0.5 and DT-0.05 benchmarks in terms of average revenue across all three approximation schemes. Notably, the CT algorithm achieves the best performance under the 2-NNs parameterization, with a significant improvement of 16.64\% over the DT-0.5 algorithm.
For the discrete-time A2C algorithms, while adopting a finer time grid (DT-0.05) noticeably improves the average revenue over DT-0.5, it incurs a severe computational penalty, increasing the training time by a factor of roughly 3.5. 
In contrast, our continuous-time algorithm is not subject to the performance-efficiency trade-off inherent in discretization-based methods. It delivers superior revenue while maintaining a time cost comparable to the DT-0.5 algorithm with coarse discretization.
\begin{table}[htbp]
    \centering
    \begin{minipage}[t]{0.48\textwidth}
    \centering
    \caption{Average revenues of three algorithms across different approximation schemes in Section~\ref{sec:CT-vs-DT}}
    \label{tab:DT-CT-revenue}
    \begin{tabular}{p{2cm}>{\raggedleft\arraybackslash}p{1.5cm}>{\raggedleft\arraybackslash}p{1.5cm}>{\raggedleft\arraybackslash}p{1.5cm}}
    \hline
    Policy & DT-0.5 & DT-0.05 & CT \\
    \hline
    Linear-Pair & 585.630 & 640.862 & 681.222 \\
    Linear-RO & 567.776 & 654.919 & 657.115 \\
    2-NNs & 615.275 & 675.850 & 717.660 \\
    \hline
    \end{tabular}
    \end{minipage}
    \hfill
    \begin{minipage}[t]{0.48\textwidth}
    \centering
    \caption{Time costs (in seconds) of three algorithms across different approximation schemes in Section~\ref{sec:CT-vs-DT}}
    \label{tab:DT-CT-time}
    \begin{tabular}{p{2cm}>{\raggedleft\arraybackslash}p{1.5cm}>{\raggedleft\arraybackslash}p{1.5cm}>{\raggedleft\arraybackslash}p{1.5cm}}
    \hline
    Policy & DT-0.5 & DT-0.05 & CT \\
    \hline
    Linear-Pair & 4,428.20 & 16,459.60 & 4,638.53 \\
    Linear-RO & 4,396.50 & 15,967.21 & 4,222.07 \\
    2-NNs & 4,286.19 & 14,730.78 & 4,333.73 \\
    \hline
    \end{tabular}
    \end{minipage}
\end{table}

\section{Conclusion}\label{sec:conclusion}
In this paper, we propose a practical continuous-time RL framework for event-driven intensity control. 
Although grounded in the specific context of choice-based network revenue management, our framework extends naturally to broader contexts of event-driven intensity control. 
For instance, in the queueing context, the decision of whether to admit an arriving job can be modeled as controlling the intensity of the effective arrival process.
In Appendix~\ref{app:queueing-control}, we apply our continuous-time RL approach to the classic problem of admission control in a single-server queue. The accompanying numerical example in this appendix demonstrates the strong performance of our approach in this application.

While we have laid the theoretical foundation for the proposed continuous-time framework and the algorithms have been shown to perform well numerically, the study also opens a number of important future directions.
First, we have shown that by choosing a proper value function approximation, the discretization error associated with integrating the basis functions over time can be eliminated. 
It remains unknown what class of policy and function approximation can have the same property.
Second, the convergence of the continuous-time RL algorithms developed in this paper for event-driven intensity control has yet to be established. 
Third, many new algorithms have been developed in the RL community (such as proximal policy optimization) and implemented in practice with success, most of which target discrete-time systems. 
It remains an open question to systematically convert these algorithms to the continuous-time control problems with possibly discrete state spaces (see e.g. \cite{zhao2024policy} for a recent study on policy optimization for controlled diffusions).

\newpage
\bibliographystyle{chicago}
\bibliography{refs}	

\newpage
\appendix

\section{Summary of Notation}\label{sec:app_notation}
\begin{longtable}{>{\raggedright\arraybackslash}p{.2\textwidth}>{\raggedright\arraybackslash}p{.8\textwidth}}
    $A$ & consumption matrix, $A \coloneqq [a_{ij}]_{m \times n}$ \\
    
    $A^j$ & $j$-th column of $A$, incidence vector for product $j$ \\
    
    $\mathcal{A}$ & action space, collection of all subsets of $\mathcal{J}$ \\

    $\mathcal{A}(x)$ & collection of all available assortments at the state $x$, $\mathcal{A}(x) \coloneqq \{S \in \mathcal{A}: x \geq A^j \text{ for all } j \in S\}$ \\
    $B([0, T] \times \mathcal{X})$ & space of Borel measurable functions defined on $[0, T] \times \mathcal{X}$ such that $\Vert \psi \Vert \coloneqq \sup_{(t, x) \in [0, T] \times \mathcal{X}} \vert \psi(t, x) \vert < \infty$ \\
    $c$ & initial inventory of $m$ resources, $c = (c_1, \ldots, c_m)^\top$ \\
    $C^{1, 0}([0, T] \times \mathcal{X})$ & space of real-valued functions defined on $[0, T] \times \mathcal{X}$ that are continuously differentiable in $t$ over $[0, T]$ for all $x \in \mathcal{X}$\\
    $d$ & degree of polynomial involved in parametric families of the value functions or policies \\
    $\mathcal{F}$ & rich enough $\sigma$-algebra defined on the sample space $\Omega$ \\
    
    {$\{\mathcal{F}_t\}_{t \geq 0}$} & filtration generated by standard Poisson process $N^{\lambda}$ and two sequences of i.i.d. uniform random variables for randomized actions and customer choices, respectively  \\ 
    $\{\mathcal{F}_t^{X^{\bpi}}\}_{t \geq 0}$ & natural filtration generated by state process $X^{\bpi}$ under randomized policy $\bpi$ \\   
    $\{\mathcal{F}_t^{\tilde{X}^{\bpi}}\}_{t \geq 0}$ & natural filtration generated by exploratory state process $\tilde{X}^{\bpi}$ \\  
    {$\{\mathcal{G}_t\}_{t \geq 0}$} & filtration generated by standard Poisson process $N^{\lambda}$ and a sequences of i.i.d. uniform random variables for customer choices \\
    $H(t, x, S, v(\cdot, \cdot))$ & Hamiltonian of the classical intensity control problem \\
    $\mathcal{H}(\bpi(\cdot \mid t, x))$ & entropy of randomized policy $\bpi(\cdot \mid t, x)$\\

    $J(t, x; \bpi)$ & entropy-regularized value function under policy $\bpi$ in continuous-time RL formulation \\
    $J_{\Delta t}(t_k, x; \bpi)$ & entropy-regularized value function under policy $\bpi$ in discretization-based RL methods \\
    $J^{*}(t, x)$ & optimal value function in the continuous-time RL formulation \\
    $J^{\theta}(t, x)$ & parametric value function in continuous-time RL formulation \\
    $J_{\Delta t}^{\theta}(t_k, x)$ & parametric value function in discretization-based RL methods \\
    $\mathcal{J}$ & set of all products, $\mathcal{J} = \{1, \ldots, n\}$ \\
    $K$ & total number of intervals in a uniform grid $0 = t_0 < t_1 < \cdots < t_K = T$ with step size $\Delta t = \frac{T}{K}$ \\
    $\mathbb{K}$ & set of all valid time-state-action triples, $\mathbb{K} \coloneqq \{(t, x, S) : t \in [0, T],\ x \in \mathcal{X},\ S \in \mathcal{A}(x)\}$ \\
    $L$ & total number of state jumps over $[0, T]$ \\
    $\bar{L}$ & number of disjoint consideration sets for customers \\
    $m$ & number of resources \\
    $M$ & batch size in algorithms \\
    $\mathcal{M}_{\Delta t}$ & discrete-time model under a uniform grid $0 = t_0 < t_1 < \cdots < t_K = T$ with step size $\Delta t = \frac{T}{K}$, it is assumed that in each period $(t_k, t_{k+1}]$, one arrival occurs with probability $\lambda \Delta t$, and no arrivals occur with probability $ 1 - \lambda \Delta t$ \\
    $n$ & number of products \\
    $\bar{n}$ & total number of revenue-ordered assortments \\
    $N$ & number of episodes in algorithms \\ 
    $N^{\lambda}$, $N_t^{\lambda}$ & a standard Poisson process with rate $\lambda$, to model the arrival process of all potential customers \\
    
    $N^{\boldsymbol{u}}$, $N_t^{\boldsymbol{u}}$ &  $N_t^{\boldsymbol{u}} = (N_{1,t}^{\boldsymbol{u}}, \ldots, N_{n, t}^{\boldsymbol{u}})^{\top}$, a vector of Poisson processes with intensities $(\lambda P_{1}(S_t), \ldots, \lambda P_{n}(S_t))$ modulated by deterministic control $\boldsymbol{u} = \{S_t : t \in [0, T]\}$ \\

    $N^{\bpi}$, $N_t^{\bpi}$ &  $N_t^{\bpi} = (N_{1,t}^{\bpi}, \ldots, N_{n, t}^{\bpi})^{\top}$, with $N_{j,t}^{\bpi}$ representing the number of product $j$ sold up to time $t$ under the randomized policy $\bpi$ \\

    $p$ & prices vector for $n$ products, $p = (p_1, \ldots, p_n)^\top$ \\
    
    $P_0(S)$ & probability that an arriving customer makes no purchase when assortment $S$ is offered \\
    
    $P_j(S)$ & probability that an arriving customer purchases product $j$ when assortment $S$ is offered \\
    $\Prob$ & probability measure defined on $\mathcal{F}$\\
    $q(y \mid t, x, S)$ & transition rate of state process \\
    
    $r(S)$ & average revenue rate when offered $S$, $r(S)\coloneqq \lambda \sum_{j=1}^{n} p_j P_j(S)$\\
    
    $S$ & assortment offered by the firm, $S \subseteq \mathcal{J}$\\
    
    $S_t^{\bpi}$ & assortment offered at time $t$ under policy $\bpi$ \\
    
    $T$ & selling time limit \\ 
    
    $\boldsymbol{u}$ & deterministic control process, $\boldsymbol{u} = \{S_t \in \mathcal{A}: t \in [0, T] \}$  \\

    $\mathcal{U}$ & set of all non-anticipating control processes \\

    $V(t, x; \boldsymbol{u})$ & value function under control $\boldsymbol{u}$ in the classical formulation \\
    $V^{*}(t, x)$ & optimal value function in the classical formulation \\
    $V_{\Delta t}^{*}(t_k, x)$ &  optimal value function under the discrete-time model $\mathcal{M}_{\Delta t}$ \\
    $V_{\text{ADP-}\Delta t}(t_k, x)$ & approximate value function derived from ADP under the discrete-time model $\mathcal{M}_{\Delta t}$, $V_{\text{ADP-}\Delta t}(t_k, x) = \theta_k + \sum_{i = 1}^{m} V_{k, i} x_i$
    \\
     $V_{\text{I-ADP-}\Delta t}(t, x)$ & continuous-time interpolated approximated value function based on $V_{\text{ADP-}\Delta t}(t_k, x)$ \\
    $X^{\boldsymbol{u}}$, $X_t^{\boldsymbol{u}}$, $X_{t-}^{\boldsymbol{u}}$ & process of the remaining inventory (state) under deterministic control $\boldsymbol{u}$, $X_{t-}^{\boldsymbol{u}}$ refers to the left limit of process $X^{\boldsymbol{u}}$ at time $t$ \\
    
    $X^{\bpi}$, $X_t^{\bpi}$, $X_{t-}^{\bpi}$ &  process of the remaining inventory (state) under randomized policy $\bpi$, $X_{t-}^{\bpi}$ refers to the left limit of process $X^{\bpi}$ at time $t$ \\
    $\tilde{X}^{\bpi}$, $\tilde{X}_t^{\bpi}$, $\tilde{X}_{t-}^{\bpi}$ &  exploratory state process defined on the probability space $(\Omega, \mathcal{F}, \Prob; \{\mathcal{G}_t\}_{t \geq 0})$, it has the same distribution as the sample state process $\{X^{\bpi}_t: t \in [0, T] \}$ defined on $(\Omega, \mathcal{F}, \Prob; \{\mathcal{F}_t\}_{t \geq 0})$) \\

    $\mathcal{X}$ & state space, $\mathcal{X} = \{0, \ldots, c_1\} \times \cdots \times \{0, \ldots, c_m\}$ \\
    $\alpha_{\phi}$ & learning rate for actor (policy) network \\
    $\alpha_{\theta}$ & learning rate for critic (value function) network \\
    $\gamma$ & entropy factor (temperature parameter) \\
    $\lambda$ & customer arrival rate \\
    $\bpi(S \mid t, x)$ & randomized policy \\
    $\bpi^{*}(S \mid t, x)$ & optimal randomized policy for the entropy-regularized optimization problem \\
    $\bpi^{\phi}(S \mid t, x)$ & parametric randomized policy \\
    $\bsigma^{\phi}(S^{[j]} \mid t)$ & parametric policy with action space that comprises all revenue-ordered assortments $S^{[1]}$, $\cdots$, $S^{[\bar{n}]}$  \\
    $\varphi(t, x)$ & vector of basis functions for the linear value function approximation, $\varphi(t, x) \coloneqq (\varphi_1(t, x), \ldots, \varphi_W(t, x))^\top$  \\
    $\phi$ & parameters in actor (policy) network \\
    $\theta$ & parameters in critic (value function) network \\
    $\tau_l$ & $l$-th jump time the state process
\end{longtable}

\section{The Reformulation of \eqref{eq:value_func_entropy_exploratory} in Section \ref{ssec:RL-formulation}\label{app:rl-reformulation}}

{We first prove that
for each $j=1, \ldots, n$, the process $\{N^{\bpi}_{j,t} - \int_0^t \sum_{S \in \mathcal{A}(X^{\bpi}_{t-})} \lambda P_j(S)\bpi (S \mid s, X^{\bpi}_{t-}) dt : t \in [0, T]\}$ is an $\{\mathcal{F}_t\}_{t\geq 0}$-martingale.
Let $N^{\lambda}_{(t, t+\Delta t]}$ denote the number of customer arrivals in $(t, t+\Delta t]$; if there is exactly one customer arrival during this time interval, we denote the arrival time by $t_a$. Given $N^{\lambda}_{(t, t+\Delta t]} = 1$, the random variable $t_{a}$ is uniformly distributed over $(t, t+\Delta t]$ by the property of the Poisson process $N^{\lambda}$. 
To demonstrate the Markov property of $N^{\bpi}$ and derive its infinitesimal generator, for any bounded, measurable, real-valued function $f$, we evaluate $\E[f( N_{t+\Delta t}^{\bpi}) \mid N_t^{\bpi} = z,\, \mathcal{F}_t]$ as follows:
\begin{align*}
    &\E[f( N_{t+\Delta t}^{\bpi}) \mid N_t^{\bpi} = z,\, \mathcal{F}_t] \\
    = {}&\sum_{y \in \mathcal{X}} f(y) \cdot \Prob (N_{t+\Delta t}^{\bpi} = y \mid N_t^{\bpi} = z,\, \mathcal{F}_t) \\
    = {}&\sum_{y \in \mathcal{X}} f(y) \cdot \bigg\{\Prob (N_{t+\Delta t}^{\bpi} = y \mid N^{\lambda}_{(t, t+\Delta t]} = 0,\, N_t^{\bpi} = z,\, \mathcal{F}_t) \cdot \Prob (N^{\lambda}_{(t, t+\Delta t]} = 0 \mid  N_t^{\bpi} = z,\, \mathcal{F}_t) \\
    &+ \bigg[\int_{t}^{t+\Delta t}\Prob (N_{t+\Delta t}^{\bpi} = y \mid t_a = u,\, N^{\lambda}_{(t, t+\Delta t]} = 1,\, N_t^{\bpi} = z,\, \mathcal{F}_t) \cdot \frac{1}{\Delta t} du \bigg] \cdot \Prob (N^{\lambda}_{(t, t+\Delta t]} = 1 \mid  N_t^{\bpi} = z,\, \mathcal{F}_t)  \\
    &+ \Prob (N_{t+\Delta t}^{\bpi} = y \mid N^{\lambda}_{(t, t+\Delta t]} \geq 2,\, N_t^{\bpi} = z,\, \mathcal{F}_t) \cdot \Prob (N^{\lambda}_{(t, t+\Delta t]} \geq 2 \mid  N_t^{\bpi} = z,\, \mathcal{F}_t)
    \bigg\} \\
    = {}&f(z) \cdot [1-\lambda \Delta t + o(\Delta t)] + \bigg[f(z) \bigg(\sum_{S \in \mathcal{A}(c-Az)} P_0(S) \frac{\int_{t}^{t+\Delta t} \bpi(S \mid u,\, c-Az) du}{\Delta t} \bigg)\\
    &+ \sum_{j=1}^n f(z + \boldsymbol{e}^j) \bigg( \sum_{S \in \mathcal{A}(c-Az)} P_j(S) \frac{\int_{t}^{t+\Delta t} \bpi(S \mid u, c-Az) du}{\Delta t}\bigg)\bigg] \cdot [\lambda \Delta t + o(\Delta t)] + o(\Delta t) \\
    = {}& f(z) + \lambda \Delta t \cdot \sum_{j=1}^n \bigg(\sum_{S \in \mathcal{A}(c-Az)} P_j(S) \frac{\int_{t}^{t+\Delta t} \bpi(S \mid u, c-Az) du}{\Delta t}\bigg)  \cdot [f(z + \boldsymbol{e}^j) - f(z)] + o(\Delta t).
\end{align*}
Since the result depends only on the current state $z$, the process $N^{\bpi}$ satisfies the Markov property. 
Next, its infinitesimal generator is given by
\begin{align}\label{eq:N-generator}
    \lim_{\Delta t\to 0}\frac{\E[f( N_{t+\Delta t}^{\bpi}) \mid N_t^{\bpi} = z] -f(z)}{\Delta t} = \lambda \sum_{j=1}^n \sum_{S \in \mathcal{A}(c-Az)} P_j(S) \bpi(S \mid t, c-Az)  \cdot [f(z + \boldsymbol{e}^j) - f(z)].
\end{align}
Given that $N^{\bpi}$ has generator \eqref{eq:N-generator} and $X_t^{\bpi} = c - AN_t^{\bpi}$ is bounded for all $t \in [0, T]$, it follows from Dynkin's formula that 
for each $j=1, \ldots, n$, the process $\{N^{\bpi}_{j,t} - \int_0^t \sum_{S \in \mathcal{A}(X^{\bpi}_{t-})} \lambda P_j(S)\bpi (S \mid s, X^{\bpi}_{t-}) dt : t \in [0, T]\}$ is an $\{\mathcal{F}_t\}_{t\geq 0}$-martingale}.

Next, we introduce the exploratory dynamics of $\{\tilde X^{\bpi}_t: t \in [0, T] \}$, which can be interpreted as the average of trajectories of the sample state process $\{X^{\bpi}_t: t \in [0, T] \}$ over infinitely many randomized actions. In the case of controlled diffusions, \cite{wang2020reinforcement} derived heuristically the exploratory state process by applying a law-of-large-numbers argument to the drift and diffusion coefficients of the controlled diffusion process. However, the sample state process $\{X^{\bpi}_t: t \in [0, T] \}$ in our setting is a pure jump process, so their approach does not apply. Instead, we derive the infinitesimal generator of the sample state process $\{X^{\bpi}_t: t \in [0, T] \}$, from which we will identify the dynamics of the exploratory state process. 

From the above derivation for the generator of $N^{\bpi}$, it is evident that for $s<t$, given the present state $X_s^{\bpi}$, the increment $X_t^{\bpi} - X_s^{\bpi} = -A(N_t^{\bpi} - N_s^{\bpi})$ does not depend on any other past information in $\mathcal{F}_s$. Therefore, the sample state process $\{X_t^{\bpi}: t\in [0, T]\}$ is also a Markov process. Using a similar derivation as for \eqref{eq:N-generator}, we can obtain the generator of $\{X_t^{\bpi}: t\in [0, T]\}$. For each time-state pair $(t, x)$ and any bounded and measurable function $g: \mathcal{X} \mapsto \mathbb{R}$,
\begin{align}\label{eq:gen-expl}
    \lim_{\Delta t\to 0}\frac{\E[g( X_{t+\Delta t}^{\bpi}) \mid X_t^{\bpi} = x] -g(x)}{\Delta t} = \lambda \sum_{j=1}^n \sum_{S \in \mathcal{A}(x)} P_j(S) \bpi(S \mid t, x)  \cdot [g(x - A^j) - g(x)].
\end{align}
{One can observe from \eqref{eq:gen-expl} that the effect of individually sampled actions has been averaged out. We can construct an (``averaged") exploratory process $\{\tilde{X}^{\bpi}_{t}: t \in [0, T] \}$, defined on the original probability space $(\Omega, \mathcal{F}, \Prob; \{\mathcal{G}_t\}_{t \geq 0})$, as a continuous-time Markov chain with the generator given by \eqref{eq:gen-expl} and $\tilde{X}^{\bpi}_{0}=c$.} 
Because the generator and the initial state of $\{X^{\bpi}_{t}: t \in [0, T] \}$ and $\{\tilde{X}^{\bpi}_{t}: t \in [0, T] \}$ are identical, we infer that 
the sample state process $\{X^{\bpi}_t: t \in [0, T] \}$ has the same distribution as the distribution of the exploratory state process $\{\tilde X^{\bpi}_t: t \in [0, T] \}$. 
This completes the proof of \eqref{eq:value_func_entropy_exploratory}.

\newpage
\section{Actor-Critic Algorithms}\label{app:AC_algs}
\begin{algorithm}[H]
\caption{Actor-Critic Algorithm with Linear Value Function Approximation (PE via TD(0))}\label{alg:TD}
\small
    \begin{algorithmic}[1]
    \State\textbf{Inputs:} initial state $c$, time horizon $T$, number of episodes $N$, batch size $M$; functional forms of basis functions $\varphi_1, \ldots, \varphi_W$,
    functional form of the policy $\bpi^{\phi}(\cdot \mid \cdot, \cdot)$ and an initial value $\phi_0$; 
    entropy factor $\gamma$, learning rate $\alpha_{\phi}$

    \State\textbf{Required program:} an environment simulator $(t', x', S', r') = Environment(t, x, \bpi^{\phi}(\cdot \mid \cdot, \cdot))$
    \State Initialize $\phi \leftarrow \phi_0$
    \For{$i = 1$ to $N$}
    \State Store $(\tau_0^{(i)}, x_0^{(i)}) \gets (0, c)$
    \State Initialize $l = 0$, $(t, x) = (0, c)$ 
    \Comment{Initialize $l$ to count jumps in each state trajectory, and $(t, x)$ to record the time and state right after each jump}
    \While{True}
    \State Apply $(t, x)$ to the environment simulator to get $(t', x', S', r') = Environment(t, x, \bpi^{\phi}(\cdot \mid \cdot, \cdot))$
    \If{$t' \geq T$}
    \State Store $L^{(i)} \gets l$, $\tau_{L^{(i)}+1}^{(i)} \gets T$
    \State \textbf{break}
    \EndIf
    \State Update $l \gets l + 1$
    \State Store observation at jump time: $(\tau_l^{(i)}, x_l^{(i)}, S_l^{(i)}, r_l^{(i)}) \gets (t', x', S', r')$
    \State Update $(t, x) \gets (t', x')$
    \EndWhile
    \If{$i= 0 \pmod{M} $} \Comment{Perform an update using $M$ episodes generated under the policy}
        \State Evaluate policy $\bpi^{\phi}$: [using formula \eqref{eq:PE_TD_0}, incorporating techniques \eqref{eq:integral_TD_right_vector} and \eqref{eq:integral_F}]
        \begin{align*}
            \theta^* = 
            &\Bigg[\frac{1}{M} \sum_{k=i-M+1}^i \bigg(\sum_{l = 0}^{L^{(k)}} F(\tau_{l}^{(k)}, \tau_{l+1}^{(k)}, x_{l}^{(k)}) + \sum_{l = 1}^{L^{(k)}} \varphi(\tau_l^{(k)}, x_{l-1}^{(k)}) [\varphi(\tau_l^{(k)}, x_{l}^{(k)}) - \varphi(\tau_l^{(k)}, x_{l-1}^{(k)})] \bigg) \Bigg]^{(-1)} \times\\
            & \Bigg[ \frac{1}{M} \sum_{k=i-M+1}^i \bigg( \sum_{l = 1}^{L^{(k)}} \varphi(\tau_l^{(k)}, x_{l-1}^{(k)}) r_{\tau_l}^{(k)} + \gamma \sum_{l = 0}^{L^{(k)}} E(\tau_l^{(k)}, \tau_{l+1}^{(k)}, x_{l}^{(k)}; \varphi(\cdot, x_{l}^{(k)}), \bpi^{\phi}) \bigg) \Bigg]
        \end{align*}
        \State Compute policy gradient at $\phi$: [using formula \eqref{eq:PG-update-rule}]
        \begin{align*}
            \Delta \phi = \nabla_{\phi}\Bigg\{ 
            \frac{1}{M} \sum_{k= i - M + 1}^{i} \bigg(&\sum_{l = 1}^{L^{(k)}} \log \bpi^{\phi} ( S_{l}^{(k)} \mid \tau_{l}^{(k)}, x_{l-1}^{(k)}) [ J^{\theta^*} (\tau_l^{(k)}, x_l^{(k)}) - J^{\theta^*} (\tau_{l}^{(k)}, x_{l-1}^{(k)} ) + r_{l}^{(k)} ] \\
            & + \gamma \sum_{l = 0}^{L^{(k)}} E(\tau_{l}^{(k)}, \tau_{l+1}^{(k)}, x_{l}^{(k)}; \boldsymbol{1}, \bpi^{\phi}) \bigg)\Bigg\}
        \end{align*}
        \State Update $\phi$ by
        \begin{align*}
            \phi \leftarrow \phi + \alpha_{\phi}\Delta \phi
        \end{align*} 
    \EndIf
    \EndFor
    \end{algorithmic}
\end{algorithm}

\begin{algorithm}[H]
\caption{Actor-Critic Algorithm with Neural Networks (PE via Monte Carlo)} \label{alg:NN}
\small
    \begin{algorithmic}[1]
    \State\textbf{Inputs:} initial state $c$, time horizon $T$, number of episodes $N$, batch size $M$; critic network $J^{\theta}(t, x)$ and an initial value $\theta_0$; 
    actor network $\bpi^{\phi}(S \mid t, x)$ and an initial value $\phi_0$; 
    entropy factor $\gamma$, learning rates $\alpha_{\theta}$, $\alpha_{\phi}$
    \State\textbf{Required program:} an environment simulator $(t', x', S', r') = Environment(t, x, \bpi^{\phi}(\cdot \mid \cdot, \cdot))$
    \State {Initialize $\theta \leftarrow \theta_0$, $\phi \leftarrow \phi_0$}
    \For{$i = 1$ to $N$} 
    \State Store $(\tau_0^{(i)}, x_0^{(i)}) \gets (0, c)$
    \State Initialize $l = 0$, $(t, x) = (0, c)$ 
    \Comment{Initialize $l$ to count jumps in each state trajectory, and $(t, x)$ to record the time and state right after each jump}
    \While{True}
    \State Apply $(t, x)$ to the environment simulator to get $(t', x', S', r') = Environment(t, x, \bpi^{\phi}(\cdot \mid \cdot, \cdot))$
    \If{$t' \geq T$}
    \State Store $L^{(i)} \gets l$, $\tau_{L^{(i)}+1}^{(i)} \gets T$
    \State \textbf{break}
    \EndIf
    \State Update $l \gets l + 1$
    \State Store observation at jump time: $(\tau_l^{(i)}, x_l^{(i)}, S_l^{(i)}, r_l^{(i)}) \gets (t', x', S', r')$
    \State Update $(t, x) \gets (t', x')$
    \EndWhile
    \If{$i= 0 \pmod{M}$} \Comment{Perform an update using $M$ episodes generated under the policy}
        \State Update critic network: [using formula \eqref{eq:MC-general-update-rule}]
        \begin{align*}
            \theta \leftarrow \theta - \alpha_{\theta} \nabla_{\theta} \Bigg\{\frac{1}{M} \sum_{k=i - M + 1}^i \sum_{l = 0}^{L^{(k)}} \bigg( &\frac{1}{2} D(\tau_l^{(k)}, \tau_{l+1}^{(k)}, x_{l}^{(k)}; \theta) \\
            &- b(\tau_{l}^{(k)}, \tau_{l+1}^{(k)}, x_{l}^{(k)}; \theta) \cdot \sum_{l' = l + 1}^{L^{(k)}} [ r_{l'}^{(k)} + \gamma E(\tau_{l'}^{(k)}, \tau_{l'+1}^{(k)}, x_{l'}^{(k)}; \boldsymbol{1}, \bpi^{\phi})] \\
            &- \gamma E(\tau_{l}^{(k)}, \tau_{l+1}^{(k)}, x_{l}^{(k)}; b(\tau_{l}^{(k)}, \cdot, x_{l}^{(k)}; \theta), \bpi^{\phi}) \bigg)\Bigg\}
        \end{align*}
        \State Update actor network: [using formula \eqref{eq:PG-update-rule}]
        \begin{align*}
            \phi \leftarrow \phi + \alpha_{\phi} \nabla_{\phi} \Bigg\{
            \frac{1}{M} \sum_{k=i - M + 1}^i \bigg(&\sum_{l = 1}^{L^{(k)}} \log \bpi^{\phi} ( S_{l}^{(k)} \mid \tau_{l}^{(k)}, x_{l-1}^{(k)}) [ J^{\theta} (\tau_l^{(k)}, x_l^{(k)}) - J^{\theta} (\tau_{l}^{(k)}, x_{l-1}^{(k)} ) + r_{l}^{(k)} ] \\
            & + \gamma \sum_{l = 0}^{L^{(k)}} E(\tau_{l}^{(k)}, \tau_{l+1}^{(k)}, x_{l}^{(k)}; \boldsymbol{1}, \bpi^{\phi}) \bigg)\Bigg\} 
        \end{align*}
    \EndIf
    \EndFor
    \end{algorithmic}
\end{algorithm}

\section{Proofs of Statements}\label{app:pf_statements}

\begin{proof}[Proof of Lemma~\ref{lem:ode_value_func}]
    Let $B([0, T] \times \mathcal{X})$ denote the space of all Borel measurable functions $\psi: [0, T] \times \mathcal{X} \rightarrow \mathbb{R}$ such that $\Vert \psi \Vert \coloneqq \sup_{(t, x) \in [0, T] \times \mathcal{X}} \vert \psi(t, x) \vert < \infty$. This space becomes a Banach space when endowed with the supreme norm $\Vert \cdot \Vert$
    We denote 
    \begin{equation*}
        R(t, x) = \sum_{S \in \mathcal{A}(x)} r(S) \bpi(S \mid t, x) + \gamma \mathcal{H}(\bpi(\cdot \mid t, x)),
    \end{equation*}
    and then define an operator $\mathcal{L}$ on $B([0, T] \times \mathcal{X})$ by
    \begin{align}\label{eq:contraction_operator}
       \mathcal{L} \psi (t, x) \coloneqq e^{\beta t} \int_{t}^{T} \bigg\{ R(s, x) + \sum_{S \in \mathcal{A}(x)} \bigg(\sum_{y \in \mathcal{X}} e^{-\beta s}\psi(s, y) q(y \mid s, x, S)\bigg) \bpi(S \mid s, x) \bigg\} ds,
    \end{align}
    where $\beta \coloneqq 2 \lambda + 1$.
    Given that the integrand on the right-hand side of \eqref{eq:contraction_operator} also belongs to the space $B([0, T] \times \mathcal{X})$,  it follows that the operator $\mathcal{L}$ is well-defined. 
    Then, for $\psi_1, \psi_2 \in B([0, T] \times \mathcal{X})$, we obtain
    \begin{align*}
        \left\vert \mathcal{L} \psi_1 (t, x) - \mathcal{L} \psi_2 (t, x) \right\vert &\leq e^{\beta t}  \int_t^T e^{-\beta s}\bigg(\sum_{S \in \mathcal{A}(x)} \sum_{y \in \mathcal{X}} \left\vert \psi_1(s, y) - \psi_2(s, y) \right\vert \cdot \left\vert q(y \mid s, x, S) \right\vert \bpi(S \mid s, x) \bigg) ds \\
        &\leq e^{\beta t}  \int_t^T e^{-\beta s}\bigg(\left\Vert \psi_1 - \psi_2 \right\Vert \cdot \sum_{S \in \mathcal{A}(x)} \Big(\sum_{y \in \mathcal{X}} \left\vert q(y \mid s, x, S) \right\vert \Big) \bpi(S \mid s, x) \bigg) ds \\
        &\leq 2  \lambda \left\Vert \psi_1 - \psi_2 \right\Vert e^{\beta t} \int_{t}^{T} e^{- \beta s} ds\\
        &\leq \frac{2  \lambda}{\beta} (1 - e^{-\beta(T-t)}) \left\Vert \psi_1 - \psi_2 \right\Vert \\
        &\leq \frac{2  \lambda}{\beta}\left\Vert \psi_1 - \psi_2 \right\Vert.
    \end{align*}
    Observing that $\frac{2\lambda}{\beta} = \frac{2\lambda}{2 \lambda + 1} < 1$, we identify $\mathcal{L}$ as a contraction operator on the Banach space $B([0, T] \times \mathcal{X})$.
    Let $\psi^* \in B([0, T] \times \mathcal{X})$ be the fixed point of $\mathcal{L}$, i.e.,
    \begin{align}
        \psi^* (t, x) \coloneqq e^{\beta t} \int_{t}^{T} \bigg\{ R(s, x) + \sum_{S \in \mathcal{A}(x)} \bigg(\sum_{y \in \mathcal{X}} e^{-\beta s}\psi^*(s, y) q(y \mid s, x, S)\bigg) \bpi(S \mid s, x) \bigg\} ds.
    \end{align}
    Let $v(t, x) = e^{-\beta t} \psi^* (t, x)$ for all $(t, x) \in [0, T] \times \mathcal{X}$, then $v$ is continuously differentiable in $t$, i.e. $v \in C^{1, 0}([0, T] \times \mathcal{X})$, and satisfies \eqref{eq:ode_value_func}.
    Thus, we establish the existence of the solution to \eqref{eq:ode_value_func}.
    
    Suppose $v^* \in C^{1, 0}([0, T] \times \mathcal{X})$ is a solution to \eqref{eq:ode_value_func}. One can readily see that the transition rates of the Markov process $\tilde{X}^{\bpi}$, introduced in Appendix~\ref{app:rl-reformulation}, at time $s$ is given by 
    \begin{align*}
    \sum_{S \in \mathcal{A}(\tilde{X}^{\bpi}_{s-})} q (y \mid s, \tilde{X}^{\bpi}_{s-}, S) \bpi(S \mid s, \tilde{X}^{\bpi}_{s-}), \quad y \ne \tilde{X}^{\bpi}_{s-}.
     \end{align*}
    It follows from Theorem 3.1 in \cite{guo2015finite} that we have the Dynkin's formula:
    \begin{align*}
        &\E \left[v^*(T, \tilde{X}^{\bpi}_{T}) \mid \tilde{X}_t^{\bpi} = x\right] - v^*(t,x) \\
        ={}& \E \left[ \int_t^T \bigg( \frac{\partial v^*}{\partial s}(s, \tilde{X}^{\bpi}_{s-}) +  \sum_{y \in \mathcal{X}} v^*(s, y) \Big[\sum_{S \in \mathcal{A}(\tilde{X}^{\bpi}_{s-})} q (y \mid s, \tilde{X}^{\bpi}_{s-}, S) \bpi(S \mid s, \tilde{X}^{\bpi}_{s-}) \Big]\bigg) ds \mid \tilde{X}_t^{\bpi} = x\right].
    \end{align*}
    Since $v^*$ satisfies Equation \eqref{eq:ode_value_func} with $v^*(T, x) =0$,  we then infer that
    \begin{align*} 
        v^*(t,x) &= \E \left[ \int_t^T \bigg\{\sum_{S \in \mathcal{A}(\tilde{X}^{\bpi}_{s-})} r(S)\bpi (S \mid s, \tilde{X}^{\bpi}_{s-}) + \gamma \mathcal{H}(\bpi (\cdot \mid s, \tilde{X}_{s-}^{\bpi}))\bigg\} ds \mid \tilde{X}_t^{\bpi} = x \right] \\
        &= J(t, x; \bpi), \quad (t, x) \in [0, T] \times \mathcal{X}.
    \end{align*}
    The proof is therefore complete. 
\end{proof}

\begin{proof}[Proof of the optimal stochastic policy \eqref{eq:fixed_point} in Section \ref{ssec:RL-formulation}]
    Let the space $B([0, T] \times \mathcal{X})$, the endowed norm $\Vert \cdot \Vert$, and $R(t, x)$ be as specified in the proof of Lemma \ref{lem:ode_value_func}.
    We then define an operator $\bar{\mathcal{L}}$ on $B([0, T] \times \mathcal{X})$ by
    \begin{align*}
       \bar{\mathcal{L}} \psi (t, x) \coloneqq e^{\beta t} \int_{t}^{T} \sup_{\bpi \in \Pi} \bigg\{ R(s, x) + \sum_{S \in \mathcal{A}(x)} \bigg(\sum_{y \in \mathcal{X}} e^{-\beta s}\psi(s, y) q(y \mid s, x, S)\bigg) \bpi(S \mid s, x) \bigg\} ds,
    \end{align*}
    where $\beta \coloneqq 2 \lambda + 1$.
    Using a similar argument as in the proof of Lemma \ref{lem:ode_value_func}, we can identify $\bar{\mathcal{L}}$ as a contraction operator on the Banach space $B([0, T] \times \mathcal{X})$ and then establish the existence of the solution in the space $C^{1, 0}([0, T] \times \mathcal{X})$ to the following exploratory HJB equation:
     \begin{equation}\label{eq:opt_ode}
     \left\{
        \begin{aligned}
            &\frac{\partial v}{\partial t } (t, x) + \sup_{\bpi \in \Pi} \left\{\sum_{S \in \mathcal{A}(x)} [H(t, x, S, v(\cdot, \cdot)) - \gamma \log\bpi(S \mid t, x)] \bpi(S \mid t, x) \right\} = 0, \quad (t, x) \in [0, T) \times \mathcal{X} \\
            &v(T, x) = 0, \quad x \in \mathcal{X}.
        \end{aligned}
        \right.
    \end{equation}

    Assume $v^* \in C^{1, 0}([0, T] \times \mathcal{X})$ is a solution to \eqref{eq:opt_ode}, then $v^* (T, x) = 0$ for all $x \in \mathcal{X}$ and the following inequality holds for all
    $\bpi \in \Pi$ and $(t, x) \in [0, T] \times \mathcal{X}$:
    \begin{align}\label{eq:pf_opt_inequlity}
       \frac{\partial v^*}{\partial t } (t, x) + \sum_{S \in \mathcal{A}(x)} \left\{H(t, x, S, v^*(\cdot, \cdot)) - \gamma \log\bpi(S \mid t, x) \right\} \bpi(S \mid t, x)  \leq 0.
    \end{align}
    Then, it follows from Theorem 3.1 in \cite{guo2015finite} that
    \begin{align*}
        &v^*(t,x) \\ 
        = {}&\E \left[v^*(T, \tilde{X}^{\bpi}_{T}) \mid \tilde{X}_t^{\bpi} = x\right] \\
        &- \E \left[ \int_t^T \bigg( \frac{\partial v^*}{\partial s}(s, \tilde{X}^{\bpi}_{s-}) +  \sum_{y \in \mathcal{X}} v^*(s, y) \Big[\sum_{S \in \mathcal{A}(\tilde{X}^{\bpi}_{s-})} q (y \mid s, \tilde{X}^{\bpi}_{s-}, S) \bpi(S \mid s, \tilde{X}^{\bpi}_{s-}) \Big]\bigg) ds \mid \tilde{X}_t^{\bpi} = x\right] \\
        \geq {}&\E \left[ \int_t^T \sum_{S \in \mathcal{A}(\tilde{X}^{\bpi}_{s-})}\{r(S) - \gamma \log \bpi(S \mid t, \tilde{X}^{\bpi}_{s-})\} \bpi(S \mid s, \tilde{X}^{\bpi}_{s-}) ds \mid
        \tilde{X}_t^{\bpi} = x\right] \\
        = {}&J(t, x; \bpi), \quad \text{ for all } \bpi \in \Pi, \ (t, x) \in [0, T] \times \mathcal{X}.
    \end{align*}
    Moreover, the equality in \eqref{eq:pf_opt_inequlity} holds for policy $\bpi^*$ defined as follows:
    \begin{equation*}
    \bpi^*(S \mid t, x) = \frac{\exp\{\frac{1}{\gamma} H(t, x, S, v^*(\cdot, \cdot))\}}{\sum_{S \in \mathcal{A}(x)} \exp\{\frac{1}{\gamma} H(t, x, S, v^*(\cdot, \cdot))\}}, \quad \text{ for } (t, x, S) \in \mathbb{K}.
    \end{equation*}
    Thus, we conclude that $v^*(t, x) = J(t, x; \bpi^*) = J^*(t, x)$ for all $(t, x) \in [0, T] \times \mathcal{X}$.
    This establishes the characterization of the optimal policy in \eqref{eq:fixed_point}.
\end{proof}

\begin{proof}[Proof of Proposition~\ref{prop:martingale}]
    Denote
    \begin{align*}
    \tilde{M}_t \coloneqq J ( t, \tilde{X}_t^{\bpi}; \bpi) + \int_0^t \bigg\{\sum_{S \in \mathcal{A}(\tilde{X}^{\bpi}_{s-})} r(S)\bpi (S \mid s, \tilde{X}^{\bpi}_{s-}) + \gamma \mathcal{H}(\bpi (\cdot \mid s, \tilde{X}_{s-}^{\bpi}))\bigg\} ds, \quad t \in [0, T].
    \end{align*}
    We then show that $\{\tilde{M}_t: t \in [0, T] \}$ is an $\{\mathcal{F}_t^{\tilde{X}^{\bpi}}\}_{t \geq 0}$-martingale. We integrate the expression for $J(t, x; \bpi)$ in \eqref{eq:value_func_entropy_exploratory}, yielding
   \begin{equation}\label{eq:pf_martingale_only_if}
    \begin{aligned}
        \tilde{M}_t = {}& \E \left[\int_t^T \bigg\{\sum_{S \in \mathcal{A}(\tilde{X}^{\bpi}_{s-})} r(S)\bpi (S \mid s, \tilde{X}^{\bpi}_{s-}) + \gamma \mathcal{H}(\bpi (\cdot \mid s, \tilde{X}_{s-}^{\bpi}))\bigg\} ds \mid  \tilde{X}_{t}^{\bpi} \right]  \\
        & + \left[\int_0^t \bigg\{\sum_{S \in \mathcal{A}(\tilde{X}^{\bpi}_{s-})} r(S)\bpi (S \mid s, \tilde{X}^{\bpi}_{s-}) + \gamma \mathcal{H}(\bpi (\cdot \mid s, \tilde{X}_{s-}^{\bpi}))\bigg\} ds \right].
    \end{aligned}
    \end{equation}
    Due to the Markov property of $\{\tilde{X}_t^{\bpi}: t \in [0, T] \}$, and also note that the second term on the right side of equation \eqref{eq:pf_martingale_only_if} is $\mathcal{F}_t^{\tilde{X}^{\bpi}}$-measurable, we obtain
    \begin{align*}
        \tilde{M}_t = \E \left[\int_0^T \bigg\{\sum_{S \in \mathcal{A}(\tilde{X}^{\bpi}_{s-})} r(S)\bpi (S \mid s, \tilde{X}^{\bpi}_{s-}) + \gamma \mathcal{H}(\bpi (\cdot \mid s, \tilde{X}_{s-}^{\bpi}))\bigg\} ds  \mid \mathcal{F}_t^{\tilde{X}^{\bpi}}\right].
    \end{align*}
    Since $J(T, x; \bpi) = 0$ for all $x \in \mathcal{X}$, we conclude that  $\tilde{M}_t = \E[\tilde{M}_T \mid \mathcal{F}_t^{\tilde{X}^{\bpi}}]$, which implies $\{\tilde{M}_t : t \in [0, T] \}$ is an $\{\mathcal{F}_t^{\tilde{X}^{\bpi}}\}_{t \geq 0}$-martingale.
    Moreover, since $\tilde{M}_t$ is bounded for any fixed time $t \in [0, T]$, it follows that $\{\tilde{M}_t : t \in [0, T] \}$ is square-integrable.

	Conversely, assume that $\{\tilde{M}_t^v: t \in [0, T] \}$ is an $\{\mathcal{F}_t^{\tilde{X}^{\bpi}}\}_{t \geq 0}$-martingale and $v(T, x) = 0$ for all $x \in \mathcal{X}$.
    From this, it follows directly that
    $\tilde{M}_t^v= \E [\tilde{M}_T^v \mid \mathcal{F}_t^{\tilde{X}^{\bpi}}]$for all $t \in [0, T]$, 
    which is equivalent to \begin{equation}\label{eq:pf_martingale_if}
        \begin{aligned}
            v ( t, \tilde{X}_t^{\bpi} ) = &\E \left[\int_0^T \bigg\{\sum_{S \in \mathcal{A}(\tilde{X}^{\bpi}_{s-})} r(S)\bpi (S \mid s, \tilde{X}^{\bpi}_{s-}) + \gamma \mathcal{H}(\bpi (\cdot \mid s, \tilde{X}_{s-}^{\bpi}))\bigg\} ds  \mid \mathcal{F}_t^{\tilde{X}^{\bpi}} \right] \\ 
            &- \left[ \int_0^t \bigg\{\sum_{S \in \mathcal{A}(\tilde{X}^{\bpi}_{s-})} r(S)\bpi (S \mid s, \tilde{X}^{\bpi}_{s-}) + \gamma \mathcal{H}(\bpi (\cdot \mid s, \tilde{X}_{s-}^{\bpi}))\bigg\} ds  \right], \quad t \in [0, T].
        \end{aligned}
    \end{equation}
    Since the second term on the right side of Equation \eqref{eq:pf_martingale_if} is $\mathcal{F}_t^{\tilde{X}^{\bpi}}$-measurable, then it follows from the Markov property of $\{\tilde{X}_t^{\bpi}: t \in [0, T] \}$ that
    \begin{align*}
        v ( t, \tilde{X}_t^{\bpi} ) & = \E \left[\int_t^T \bigg\{\sum_{S \in \mathcal{A}(\tilde{X}^{\bpi}_{s-})} r(S)\bpi (S \mid s, \tilde{X}^{\bpi}_{s-}) + \gamma \mathcal{H}(\bpi (\cdot \mid s, \tilde{X}_{s-}^{\bpi}))\bigg\} ds \mid \mathcal{F}_t^{\tilde{X}^{\bpi}} \right] \\
        & = \E \left[\int_t^T \bigg\{\sum_{S \in \mathcal{A}(\tilde{X}^{\bpi}_{s-})} r(S)\bpi (S \mid s, \tilde{X}^{\bpi}_{s-}) + \gamma \mathcal{H}(\bpi (\cdot \mid s, \tilde{X}_{s-}^{\bpi}))\bigg\} ds \mid \tilde{X}_t^{\bpi} \right] \\
        & = J(t, \tilde{X}_t^{\bpi}; \bpi),  \quad t \in [0, T].
    \end{align*}
    Therefore, we conclude that $v(t, x) = J(t, x; \bpi)$ for all $(t, x) \in [0, T] \times \mathcal{X}$.
\end{proof}

\begin{proof}[Proof of Theorem \ref{thm:argminML=argminMSVE}]
Denote $U_t \coloneqq \int_{(t, T]} p^{\top} dN_s^{\bpi} - \int_t^T \sum_{S \in \mathcal{A}(X^{\bpi}_{s-})} r(S)\bpi (S \mid s, X^{\bpi}_{s-}) ds$, then we have
\begin{equation}\label{eq:sample_state_jump_to_intensity}
    \E \left[ U_t\mid X_t^{\bpi}\right] = 0. 
\end{equation}
We further denote $M_t \coloneqq J (t, X_t^{\bpi}; \bpi) + \int_0^t \{\sum_{S \in \mathcal{A}(X^{\bpi}_{s-})} r(S)\bpi (S \mid s, X^{\bpi}_{s-}) + \gamma \mathcal{H}(\bpi (\cdot \mid s, X_{s-}^{\bpi}))\} ds$ and $M_t^{\theta} \coloneqq J^{\theta}(t, X_t^{\bpi}) + \int_0^t \{\sum_{S \in \mathcal{A}(X^{\bpi}_{s-})} r(S)\bpi (S \mid s, X^{\bpi}_{s-}) + \gamma \mathcal{H}(\bpi (\cdot \mid s, X_{s-}^{\bpi}))\} ds$.
It follows from $J (T, x; \bpi) \equiv 0$ and \eqref{eq:sample_state_jump_to_intensity} that
\begin{align*}
   2 L(\theta) &= \E \left[\int_{0}^{T} (U_t  + M_T - M_t^{\theta})^{2} dt \right] \\
   &= \E \left[\int_{0}^{T} [U_t^2 + 2U_t(M_T - M_t^{\theta})  + (M_T - M_t^{\theta})^2] dt \right] \\
   &=\E \left[\int_{0}^{T} \bigg[U_t^2 + 2U_t\bigg(\int_t^T \bigg\{\sum_{S \in \mathcal{A}(X^{\bpi}_{s-})} r(S)\bpi (S \mid s, X^{\bpi}_{s-}) + \gamma \mathcal{H}(\bpi (\cdot \mid s, X_{s-}^{\bpi}))\bigg\} ds\bigg)\bigg] dt \right] \\
   &\hspace{0.5cm} - \E \left[\int_{0}^{T} 2J^{\theta}(t, X_t^{\bpi}) \E [U_t \mid X_t^{\bpi}]  dt \right] +\E \left[\int_{0}^{T} (M_T - M_t^{\theta})^2 dt \right] \\
   &=\E \left[\int_{0}^{T} \bigg[U_t^2 + 2U_t\bigg(\int_t^T \bigg\{\sum_{S \in \mathcal{A}(X^{\bpi}_{s-})} r(S)\bpi (S \mid s, X^{\bpi}_{s-}) + \gamma \mathcal{H}(\bpi (\cdot \mid s, X_{s-}^{\bpi}))\bigg\} ds\bigg)\bigg] dt \right] \\
   &\hspace{0.5cm} +\E \left[\int_{0}^{T} (M_T - M_t^{\theta})^2 dt \right].
\end{align*}
Since the first term does not rely on $\theta$, 
we have 
\begin{equation}\label{eq:sample_state_loss_equivalent}
    \argmin_{\theta} L(\theta) = \argmin_{\theta} \E \left[\int_{0}^{T} (M_T - M_t^{\theta})^2 dt \right].
\end{equation}

Next, we denote $\tilde{M}_t \coloneqq J (t, \tilde{X}_t^{\bpi}; \bpi) + \int_0^t \{\sum_{S \in \mathcal{A}(\tilde{X}^{\bpi}_{s-})} r(S)\bpi (S \mid s, \tilde{X}^{\bpi}_{s-}) + \gamma \mathcal{H}(\bpi (\cdot \mid s, \tilde{X}_{s-}^{\bpi}))\} ds$ and $\tilde{M}_t^{\theta} \coloneqq J^{\theta} (t, \tilde{X}_t^{\bpi}) + \int_0^t \{\sum_{S \in \mathcal{A}(\tilde{X}^{\bpi}_{s-})} r(S)\bpi (S \mid s, \tilde{X}^{\bpi}_{s-}) + \gamma \mathcal{H}(\bpi (\cdot \mid s, \tilde{X}_{s-}^{\bpi}))\} ds$.
From Proposition~\ref{prop:martingale} we know that $\{\tilde{M}_t: t \in [0, T] \}$ is an $\{\mathcal{F}_t^{\tilde{X}^{\bpi}}\}_{t \geq 0}$-martingale. Therefore, 
\begin{align*}
    \E \left[ \int_0^T ( \tilde{M}_T - \tilde{M}_t^\theta )^2 dt \right]
    &=\E \left[ \int_0^T ( \tilde{M}_T - \tilde{M}_t + \tilde{M}_t - \tilde{M}_t^\theta )^2 dt \right] \\
    &=\E \left[ \int_0^T [ ( \tilde{M}_T - \tilde{M}_t )^2 + ( \tilde{M}_t - \tilde{M}_t^\theta )^2 + 2( \tilde{M}_T - \tilde{M}_t ) ( \tilde{M}_t - \tilde{M}_t^\theta ) ] dt \right] \\
    &=\E \left[ \int_0^T ( \tilde{M}_T - \tilde{M}_t )^2 dt \right] + \E \left[ \int_0^T ( \tilde{M}_t - \tilde{M}_t^\theta )^2 dt \right] \\
    & \hspace{0.5cm} + 2 \int_0^T \E \left[( \tilde{M}_t - \tilde{M}_t^\theta ) \cdot \E \big[\tilde{M}_T-\tilde{M}_t \mid \mathcal{F}_t^{\tilde{X}^{\bpi}}\big] \right] dt \\
    &= \E \left[ \int_0^T ( \tilde{M}_T - \tilde{M}_t )^2 dt \right] + \E \left[ \int_0^T | J(t, \tilde{X}_t^{\bpi}; \bpi) - J^{\theta}(t, \tilde{X}_t^{\bpi}) |^2 dt \right].
\end{align*}
Noting that the first term does not rely on $\theta$, we obtain
\begin{align*}
    \arg\min_{\theta} \E \left[ \int_0^T | \tilde{M}_T - \tilde{M}_t^\theta |^2 dt \right]
    = \arg\min_{\theta}\E \left[ \int_0^T | J(t, \tilde{X}_t^{\bpi}; \bpi) - J^{\theta}(t, \tilde{X}_t^{\bpi}) |^2 dt \right].
\end{align*}
Since the distribution of $\{\tilde X^{\bpi}_t: t \in [0, T] \}$ is the same as the distribution of $\{X^{\bpi}_t: t \in [0, T] \}$ together with \eqref{eq:sample_state_loss_equivalent}, we conclude that $\argmin_{\theta} L(\theta) = \argmin_{\theta} \operatorname{MSVE} (\theta)$.
\end{proof}

\begin{proof}[Proof of Theorem \ref{thm:martingale_orthogonality_condition}]
Note that a bounded process $\xi$ with $\xi_t \in \mathcal{F}_{t-}^{X^{\bpi}}$ is also $\{\mathcal{F}_{t}\}_{t \geq 0}$-predictable.
On the other hand, for each $j=1, \cdots, n$, {the process $\{N^{\bpi}_{j,t} - \int_0^t \sum_{S \in \mathcal{A}(X^{\bpi}_{t-})} \lambda P_j(S)\bpi (S \mid s, X^{\bpi}_{t-}) dt : t \in [0, T]\}$ is an $\{\mathcal{F}_t\}_{t\geq 0}$-martingale.} Thus we have
\begin{equation*}
    \E\left[ \int_{0}^{T} \xi_t (p^\top d N_t^{\bpi}) \right] = 
    \E\left[ \int_{0}^{T} \xi_t \Big[\sum_{S \in \mathcal{A}(X^{\bpi}_{t-})} r(S)\bpi (S \mid t, X^{\bpi}_{t-})\Big] dt \right].
\end{equation*}
It suffices to prove that $v(t, x) = J(t, x; \bpi)$ for all $(t, x) \in [0, T] \times \mathcal{X}$, if and only if $v$ satisfies $v (T, x ) = 0$ for all $x \in \mathcal{X}$, and 
\begin{equation}\label{eq:pf_N_martingale}
    \E\left[ \int_{0}^{T} \xi_t \bigg\{ d v ( t, X_t^{\bpi}) + \Big[\sum_{S \in \mathcal{A}(X^{\bpi}_{t-})} r(S)\bpi (S \mid t, X^{\bpi}_{t-}) + \gamma \mathcal{H}(\bpi (\cdot \mid t, X_{t-}^{\bpi}))\Big] dt \bigg\} \right]
\end{equation}
for any bounded process $\xi$ with $\xi_t \in \mathcal{F}_{t-}^{X^{\bpi}}$.

We first establish that $v(t, x) = J(t, x; \bpi)$ for all $(t, x) \in [0, T] \times \mathcal{X}$, if and only if $v$ satisfies $v (T, x ) = 0$ for all $x \in \mathcal{X}$, and 
\begin{align} \label{eq:orthogonality_condition_exploratory}
    \E\left[ \int_{0}^{T} \tilde{\xi}_t \bigg\{ d v ( t, \tilde{X}_t^{\bpi}) + \Big[\sum_{S \in \mathcal{A}(\tilde{X}^{\bpi}_{t-})} r(S)\bpi (S \mid t, \tilde{X}^{\bpi}_{t-}) + \gamma \mathcal{H}(\bpi (\cdot \mid t, \tilde{X}_{t-}^{\bpi}))\Big] dt \bigg\} \right] = 0,
\end{align}
for any bounded process $\tilde{\xi}$ with $\tilde \xi_t \in \mathcal{F}_{t-}^{\tilde{X}^{\bpi}}$.
The ``only if'' part follows immediately from Proposition~\ref{prop:martingale}. 
To establish the ``if'' part, assume that $v ( T, x) = 0$ for all $x \in \mathcal{X}$, and \eqref{eq:orthogonality_condition_exploratory} holds for any bounded process $\tilde{\xi}$ with $\tilde \xi_t \in \mathcal{F}_{t-}^{\tilde{X}^{\bpi}}$.
It follows from Theorem 3.1 of \cite{guo2015finite} that the process 
\begin{align*}
    v(t, \tilde{X}_t^{\bpi}) - \int_{0}^{t} \bigg(\frac{\partial v}{\partial s }(s, \tilde{X}^{\bpi}_{s-}) + \sum_{y \in \mathcal{X}} v(s, y) \Big[\sum_{S \in \mathcal{A}(\tilde{X}^{\bpi}_{s-})} q (y \mid s, \tilde{X}^{\bpi}_{s-}, S) \bpi(S \mid s, \tilde{X}^{\bpi}_{s-}) \Big] \bigg) ds
\end{align*}
defines an $\{\mathcal{F}_t^{\tilde{X}^{\bpi}}\}_{t \geq 0}$-martingale, and its boundedness at any fixed time $t \in [0, T]$ ensures square-integrability.
Then, it follows that
\begin{equation}\label{eq:pf_TD_xi_1}
    \begin{aligned}
        \E\bigg[ \int_{0}^{T} \tilde{\xi}_t \bigg\{d v ( t, \tilde{X}_t^{\bpi}) - \bigg(\frac{\partial v}{\partial t }(t, \tilde{X}^{\bpi}_{t-}) + \sum_{y \in \mathcal{X}} v(t, y) \Big[\sum_{S \in \mathcal{A}(\tilde{X}^{\bpi}_{t-})} q (y \mid t, \tilde{X}^{\bpi}_{t-}, S) \bpi(S \mid t, \tilde{X}^{\bpi}_{t-}) \Big] \bigg) dt \bigg\} \bigg] = 0,
    \end{aligned}
\end{equation}
for any bounded process $\tilde{\xi}$ with $\tilde \xi_t \in \mathcal{F}_{t-}^{\tilde{X}^{\bpi}}$.
By taking the difference between equations \eqref{eq:orthogonality_condition_exploratory} and \eqref{eq:pf_TD_xi_1}, we obtain that
\begin{align}\label{eq:pf_TD_xi_2}
    \E \left[\int_{0}^{T} \tilde{\xi}_t \bigg(\frac{\partial v }{\partial t } (t, \tilde{X}_{t-}^{\bpi}) + \sum_{S \in \mathcal{A}(\tilde{X}_{t-}^{\bpi})} H(t, \tilde{X}_{t-}^{\bpi}, S, v(\cdot, \cdot))\bpi(S \mid t, \tilde{X}_{t-}^{\bpi}) + \gamma \mathcal{H}( \bpi(\cdot \mid t, \tilde{X}_{t-}^{\bpi}))\bigg) dt \right] = 0
\end{align}
holds for any bounded process $\tilde{\xi}$ with $\tilde \xi_t \in \mathcal{F}_{t-}^{\tilde{X}^{\bpi}}$.
Define the test function $\tilde{\xi}_t$ as follows:
\begin{align*}
    \tilde{\xi}_t = \operatorname{sgn} \bigg(\frac{\partial v }{\partial t } (t, \tilde{X}_{t-}^{\bpi}) + \sum_{S \in \mathcal{A}(\tilde{X}_{t-}^{\bpi})} H(t, \tilde{X}_{t-}^{\bpi}, S, v(\cdot, \cdot))\bpi(S \mid t, \tilde{X}_{t-}^{\bpi}) + \gamma \mathcal{H}( \bpi(\cdot \mid t, \tilde{X}_{t-}^{\bpi}))\bigg),
\end{align*}
{where $\operatorname{sgn}(\cdot)$ is the sign function.}
Then, \eqref{eq:pf_TD_xi_2} implies that
\begin{align}
    \frac{\partial v }{\partial t } (t, \tilde{X}_{t-}^{\bpi}) + \sum_{S \in \mathcal{A}(\tilde{X}_{t-}^{\bpi})} H(t, \tilde{X}_{t-}^{\bpi}, S, v(\cdot, \cdot))\bpi(S \mid t, \tilde{X}_{t-}^{\bpi}) + \gamma \mathcal{H}( \bpi(\cdot \mid t, \tilde{X}_{t-}^{\bpi})) = 0, \quad t \in [0, T].
\end{align}
It follows from Lemma \ref{lem:ode_value_func} that $v(t, x) = J(t, x; \bpi)$ for all $(t, x) \in [0, T] \times \mathcal{X}$.
Hence, we complete the proof of the equivalent condition \eqref{eq:orthogonality_condition_exploratory}.

Let $D[0, T]$ denote the space of all mappings $f: [0, T] \mapsto \mathcal{X}$. For any process $Y \in D[0, T]$ and fixed $t\in [0, T]$, define the stopped process $Y_{t- \wedge \cdot}^{\bpi} \in D[0, T]$ such that $Y_{t- \wedge \cdot}^{\bpi}(s) = Y_s^{\bpi}$ for $s \in [0, t)$, and $Y_{t- \wedge \cdot}^{\bpi}(s) = Y_{t-}^{\bpi}$ for $s \in [t, T]$.
Note that any process $\xi$ with $\xi_t \in \mathcal{F}_{t-}^{X^{\bpi}}$ corresponds to a measurable function $\boldsymbol{\xi}: [0, T] \times D([0, T]) \mapsto \mathbb{R}$ such that $\xi_t = \boldsymbol{\xi}(t, X_{t- \wedge \cdot}^{\bpi})$. Then, $\boldsymbol{\xi}(t, \tilde{X}_{t- \wedge \cdot}^{\bpi})$ is $\mathcal{F}_{t-}^{\tilde{X}^{\bpi}}$-measurable for each $t \in [0, T]$.
Since the distribution of $\{X_t^{\bpi}: t \in [0, T] \}$ is the same as the distribution of $\{\tilde{X}_t^{\bpi}: t \in [0, T] \}$, it follows that
\begin{equation}\label{eq:pf_orthogonality_sample_exploratory}
    \begin{aligned}
        {}&\E\left[ \int_{0}^{T} \boldsymbol{\xi}(t, X_{t- \wedge \cdot}^{\bpi}) \bigg\{ d v ( t, X_t^{\bpi}) + \Big[\sum_{S \in \mathcal{A}(X^{\bpi}_{t-})} r(S)\bpi (S \mid t, X^{\bpi}_{t-}) + \gamma \mathcal{H}(\bpi (\cdot \mid t, X_{t-}^{\bpi}))\Big] dt \bigg\} \right] \\
        = {}&\E\left[ \int_{0}^{T} \boldsymbol{\xi}(t, \tilde{X}_{t- \wedge \cdot}^{\bpi}) \bigg\{ d v ( t, \tilde{X}_t^{\bpi}) + \Big[\sum_{S \in \mathcal{A}(\tilde{X}^{\bpi}_{t-})} r(S)\bpi (S \mid t, \tilde{X}^{\bpi}_{t-}) + \gamma \mathcal{H}(\bpi (\cdot \mid t, \tilde{X}_{t-}^{\bpi}))\Big] dt \bigg\} \right]. 
    \end{aligned}
\end{equation}
Conversely, any process $\tilde{\xi}$ with $\tilde{\xi}_t \in \mathcal{F}_{t-}^{\tilde{X}^{\bpi}}$ corresponds to a measurable function $\tilde{\boldsymbol{\xi}}: [0, T] \times D([0, T]) \mapsto \mathbb{R}$ such that $\tilde{\xi}_t = \tilde{\boldsymbol{\xi}}(t, \tilde{X}_{t- \wedge \cdot}^{\bpi})$.
Then, $\tilde{\boldsymbol{\xi}}(t, X_{t- \wedge \cdot}^{\bpi})$ is $\mathcal{F}_{t-}^{X^{\bpi}}$-measurable for each $t \in [0, T]$, and equation \eqref{eq:pf_orthogonality_sample_exploratory} holds for $\tilde{\boldsymbol{\xi}}(\cdot, \cdot)$.
This establishes the equivalence between condition \eqref{eq:orthogonality_condition_exploratory} and condition \eqref{eq:pf_N_martingale}. The proof is therefore complete. 
\end{proof}

\begin{proof}[Proof of Theorem~\ref{thm:PG}]
    By the representation \eqref{eq:representation_of_gradient_Hamiltonian} of $g(t, x; \phi) = \nabla_{\phi} J( t, x; \bpi^{\phi})$, we have 
    \begin{equation}\label{eq:PG2}
        \begin{aligned}
        \nabla_{\phi} J (0, c; \bpi^{\phi}) = 
        \E \Bigg[{}&\int_{0}^{T}  \bigg(\sum_{S \in \mathcal{A}(X^{\bpi^{\phi}}_{t-})} \nabla_{\phi} \log \bpi^{\phi} ( S \mid t, X^{\bpi^{\phi}}_{t-}) H(t, X^{\bpi^{\phi}}_{t-}, S, J(\cdot, \cdot; \bpi^{\phi})) \bpi^{\phi} ( S \mid t, X^{\bpi^{\phi}}_{t-}) \bigg) dt \\
        & + \gamma \int_{0}^{T} \nabla_{\phi} \mathcal{H}(\bpi^{\phi} ( \cdot \mid t, \tilde{X}_{t-}^{\bpi^{\phi}})) dt \mid X_0^{\bpi^{\phi}}=c \Bigg].
        \end{aligned}
    \end{equation}
    It suffices to show the first term of \eqref{eq:PG2} also equals to the first term of \eqref{eq:PG-formula}.
    To simplify the notation, we use $\E_c$ to denote the expectation conditional on the initial state $X_0^{\bpi^{\phi}}=c$ in the subsequent derivations.
    
    {Denote the arrival time of the $k$-th customer as $T_k$ and the product chosen by the customer as $\zeta_{T_k}$, where $\zeta_{T_k} \in \{0, 1, \ldots, n\}$. 
    For $j = 1, \ldots, n$, denote
    \begin{align*}
        I_j(t, x, S) \coloneqq \nabla_{\phi} \log \bpi^{\phi} (S \mid t, x) [J(t, x-A^{j}; \bpi^{\phi}) - J(t, x; \bpi^{\phi}) + p_j].
    \end{align*}
    Then, the first term of \eqref{eq:PG-formula} can be rewritten as \begin{align}\label{eq:pf_PG_arrival_time_formulation}
        \sum_{k=1}^{\infty} \E_c \Bigg[ \sum_{j=1}^{n} I_j(T_k, X_{T_k -}^{\bpi^{\phi}}, S_{T_k}^{\bpi^{\phi}}) \boldsymbol{1}_{\{\zeta_{T_k} = j\}} \boldsymbol{1}_{\{T_k \leq T\}} \Bigg].
    \end{align}
    For each $k \geq 1$, by taking conditional expectation with respect to $\sigma(T_k, X_{T_k -}^{\bpi^{\phi}})$ inside the expectation in \eqref{eq:pf_PG_arrival_time_formulation}, we obtain
    \begin{align}
        &\sum_{k=1}^{\infty} \E_c \Bigg[ \sum_{S \in \mathcal{A}(X_{T_k -}^{\bpi^{\phi}})} \bigg( \sum_{j=1}^{n}  I_j(T_k, X_{T_k -}^{\bpi^{\phi}}, S) P_j(S) \bpi^{\phi}(S \mid T_k, X_{T_k -}^{\bpi^{\phi}}) \bigg) \boldsymbol{1}_{\{T_k \leq T\}} \Bigg] \notag\\
        = {}& \E_c \Bigg[\int_0^T \sum_{S \in \mathcal{A}(X_{t -}^{\bpi^{\phi}})} \bigg( \sum_{j=1}^{n}  I_j(t, X_{t -}^{\bpi^{\phi}}, S) P_j(S) \bpi^{\phi}(S \mid t, X_{t -}^{\bpi^{\phi}}) \bigg) d N_t^{\lambda} \Bigg] \notag\\
        = {}& \E_c \Bigg[\int_0^T \sum_{S \in \mathcal{A}(X_{t -}^{\bpi^{\phi}})} \bigg( \sum_{j=1}^{n}  I_j(t, X_{t -}^{\bpi^{\phi}}, S) \lambda P_j(S) \bpi^{\phi}(S \mid t, X_{t -}^{\bpi^{\phi}}) \bigg) d t\Bigg] \label{eq:pf_PG_possion_integral},
    \end{align}
    where \eqref{eq:pf_PG_possion_integral} is derived from the fact that $\{N_t^{\lambda} - \lambda t: t\in [0, T]\}$ is an $\{\mathcal{F}_t\}_{t\geq 0}$-martingale, and the integrand is $\{\mathcal{F}_t\}_{t\geq 0}$-predictable and bounded. The boundedness of the integrand is ensured by Assumption \ref{apt:policy_phi} and the inherent boundedness of the state process $\{X_t^{\bpi^{\phi}}: t\in [0, T]\}$.
    On the other hand, given the definition of $H$ in \eqref{eq:Hamiltonian}, we have
    \begin{align*}
        H(t, x, S, J(\cdot, \cdot; \bpi^{\phi})) &= \lambda \sum_{j=1}^{n} p_j P_j(S) + \sum_{y \neq x} \bigg( \sum_{\{j \in \mathcal{J}: A^j = x - y\}} \lambda P_j(S) \bigg) J(t, y; \bpi^{\phi})  - \lambda [1 - P_0(S)] J(t, x; \bpi^{\phi}) \\
        &= \sum_{j = 1}^{n} [J(t, x - A^j; \bpi^{\phi}) - J(t, x; \bpi^{\phi}) + p_j] \lambda P_j(S).
    \end{align*}
    Thus, \eqref{eq:pf_PG_possion_integral} is exactly the first term of \eqref{eq:PG2}. This concludes the proof.} 

\end{proof}

\section{The instability of ADP time discretization}\label{app:ADP_solution_time_discretization}
Figure \ref{fig:ADP_solution_time_discretization} illustrates the variation of the coefficients $V_{0,i}$ for $i=1, \ldots, 6$, selected from the ADP solution in Section~\ref{sec:mid-net}, across five closely spaced levels of time discretization $\Delta t = $ 0.48, 0.49, 0.50, 0.51 and 0.52.
\begin{figure}[htp]
    \centering
    \begin{subfigure}{.32\textwidth}
        \centering
        \begin{tikzpicture}
            \begin{axis}[xlabel={$\Delta t$}, ylabel={$V_{0, 1}$}, height=5cm, width=\linewidth, xtick={0.48, 0.49, 0.50, 0.51, 0.52}]
                \addplot table [col sep=comma, x=Delta_t, y=Resource_1] {data/ADP.csv};
            \end{axis}
        \end{tikzpicture}
    \end{subfigure}
    \hfill
    \begin{subfigure}{.32\textwidth}
        \centering
        \begin{tikzpicture}
            \begin{axis}[xlabel={$\Delta t$}, ylabel={$V_{0, 2}$}, height=5cm, width=\linewidth, xtick={0.48, 0.49, 0.50, 0.51, 0.52}]
                \addplot table [col sep=comma, x=Delta_t, y=Resource_2] {data/ADP.csv};
            \end{axis}
        \end{tikzpicture}
    \end{subfigure}
    \hfill
    \begin{subfigure}{.32\textwidth}
        \centering
        \begin{tikzpicture}
            \begin{axis}[xlabel={$\Delta t$}, ylabel={$V_{0, 3}$}, height=5cm, width=\linewidth, xtick={0.48, 0.49, 0.50, 0.51, 0.52}]
                \addplot table [col sep=comma, x=Delta_t, y=Resource_3] {data/ADP.csv};
            \end{axis}
        \end{tikzpicture}
    \end{subfigure} 
    
    \begin{subfigure}{.32\textwidth}
        \centering
        \begin{tikzpicture}
            \begin{axis}[xlabel={$\Delta t$}, ylabel={$V_{0, 4}$}, height=5cm, width=\linewidth, xtick={0.48, 0.49, 0.50, 0.51, 0.52}]
                \addplot table [col sep=comma, x=Delta_t, y=Resource_4] {data/ADP.csv};
            \end{axis}
        \end{tikzpicture}
    \end{subfigure}
    \hfill
    \begin{subfigure}{.32\textwidth}
        \centering
        \begin{tikzpicture}
            \begin{axis}[xlabel={$\Delta t$}, ylabel={$V_{0, 5}$}, height=5cm, width=\linewidth, xtick={0.48, 0.49, 0.50, 0.51, 0.52}]
                \addplot table [col sep=comma, x=Delta_t, y=Resource_5] {data/ADP.csv};
            \end{axis}
        \end{tikzpicture}
    \end{subfigure}
    \hfill
    \begin{subfigure}{.32\textwidth}
        \centering
        \begin{tikzpicture}
            \begin{axis}[xlabel={$\Delta t$}, ylabel={$V_{0, 6}$}, height=5cm, width=\linewidth, xtick={0.48, 0.49, 0.50, 0.51, 0.52}]
                \addplot table [col sep=comma, x=Delta_t, y=Resource_6] {data/ADP.csv};
            \end{axis}
        \end{tikzpicture}
    \end{subfigure}
    \caption{Variations of selected coefficients from the ADP solution across different time discretization levels}
    \label{fig:ADP_solution_time_discretization}
\end{figure}

\section{Extension to Admission Control in a Single-Server Queue}\label{app:queueing-control}
To demonstrate the generality of our framework, we apply it to another classic problem in Operations Research:
admission control in a single-server queue. We assume the customer arrival process and service process are non-homogeneous Poisson process with rate function $\lambda(t)$ and $\mu(t)$, respectively. If $\mu(t)\equiv \mu$, that means the service times of customers are i.i.d. exponentially distributed with rate $\mu$. 
An agent is responsible for determining whether to admit arrivals in order to maximize the expected total revenue over a finite horizon $[0, T]$.
The system has a maximum capacity of $C$ customers, and it is intially empty. 
Let $A_t$ denote the admitted arrival process, and $D_t$ denote the customer departure process. 
The queue length (which includes the customer in service), represented by $X_t = A_t - D_t$, is the system state.
The action space is a binary set $\{0, 1\}$, where 0 indicates rejecting an arriving customer, and 1 indicates admitting the customer. 
Then, a (randomized) policy $\bpi$ is defined as a mapping from each time-state pair $(t, x)$ to a Bernoulli distribution with success probability $p$. The entropy-regularized value function is given by
\begin{align*}
    J(t, x; \bpi) \coloneqq \E \left[ K_1 \int_{(t, T]}  d {A_s^{\bpi}} -K_2 \int_t^T X_{s-}^{\bpi} ds - K_3 X_T^{\bpi} + \gamma \int_t^T \mathcal{H}(\bpi(\cdot \mid s, X_{s-}^{\bpi})) ds \mid X_t^{\bpi} = x\right].
\end{align*}
The positive coefficients $K_1$, $K_2$, $K_3$ are interpreted as follows: $K_1$ is the jump reward for admitting a customer, $K_2$ is the holding cost per unit of time for each admitted customer, and $K_3$ denotes the penalty for an admitted customer not served by the terminal time $T$.
The earlier analysis can be readily extended to this queueing formulation, with almost identical proof methodologies. 
We present the results as follows.

For an approximate value function parametrized by $\theta$, we define the loss function $L(\theta)$ as 
\begin{align*}
    L(\theta) = \frac{1}{2} \E \Bigg[\int_0^T \bigg( &K_1 \int_{(t, T]}  d {A_s^{\bpi}} -K_2 \int_t^T X_{s-}^{\bpi} ds - K_3 X_T^{\bpi} + \gamma \int_t^T  \mathcal{H}(\bpi(\cdot \mid s, X^{\bpi}_{s-})) ds \\
    &- J^{\theta} ( t, X_t^{\bpi}) \bigg)^{2} dt \Bigg].
\end{align*}
It allows us to establish the equivalence theorem, which confirms that the Monte Carlo method for PE remains effective within the queueing framework.
\begin{theorem}[Parallel to Theorem \ref{thm:argminML=argminMSVE}]
    It holds that {$\argmin_{\theta} L(\theta) = \argmin_{\theta} \operatorname{MSVE} (\theta)$}.
\end{theorem}
Moreover, we have developed the martingale orthogonality conditions for the queueing model, which form the basis for the TD method in PE.
\begin{theorem}[Parallel to Theorem \ref{thm:martingale_orthogonality_condition}]
    A function $v \in C^{1, 0}([0, T] \times \mathcal{X})$ is the value function associated with the policy $\bpi$, i.e. $v(t, x) = J(t, x; \bpi)$ for all $(t, x) \in [0, T] \times \mathcal{X}$, if and only if it satisfies $v (T, x) = - K_3 x$ for all $x \in \mathcal{X}$, and the following martingale orthogonality condition holds for any bounded process $\xi$ with $\xi_t \in \mathcal{F}_{t-}^{X^{\bpi}}$ for all $t \in [0, T]$:
    \begin{align*}
        \E \Bigg[ \int_{0}^{T} \xi_t \bigg\{ d v ( t, X_t^{\bpi}) + K_1 d A_t^{\bpi} - K_2 X_{t-}^{\bpi} dt + \gamma \mathcal{H}(\bpi(\cdot \mid t, X_{t-}^{\bpi})) dt \bigg\} \Bigg] = 0.
    \end{align*}
\end{theorem}
Finally, we derive the formula for PG. 
\begin{theorem}[Parallel to Theorem \ref{thm:PG}]
    Given an admissible parameterized policy $\bpi^{\phi}$ satisfying Assumption~\ref{apt:policy_phi}, the policy gradient $\nabla_{\phi} J (0, c; \bpi^{\phi})$ admits the following representation: 
    \begin{align*}
        \nabla_{\phi} J (0, c; \bpi^{\phi}) =
        \E \Bigg[ {}& \int_{(0,T]}  \nabla_{\phi} \log \bpi^{\phi} ( 1 \mid t, X_{t-}^{\bpi^{\phi}}) [J(t, X_{t-}^{\bpi^{\phi}} + 1; \bpi^{\phi} ) - J(t, X_{t-}^{\bpi^{\phi}}; \bpi^{\phi}) + K_1] dA_t^{\bpi^{\phi}} \notag \\
        & + \gamma \int_{0}^{T} \nabla_{\phi} \mathcal{H}(\bpi^{\phi} ( \cdot \mid t, X_{t-}^{\bpi^{\phi}})) dt \Bigg] .
    \end{align*}
\end{theorem}
On combining PE with PG, we can now design actor-critic algorithms. We present Algorithm \ref{alg:NN_for_queueing} as an example. 
This algorithm employs the neural-network-based approximation for both value and policy functions, and utilizes the Monte Carlo method for PE.

To demonstrate the performance of Algorithm \ref{alg:NN_for_queueing}, we consider a simple example: we set the maximum capacity $C = 10$, the time horizon $T = 20$, the coefficients $K_1 = 10$, $K_2 = 1$, $K_3 = 0.1$. The rate functions for customer arrival process and service process are defined as $\lambda(t) = 0.5 + 0.3 \cdot \sin(\frac{2\pi t}{T})$ and $\mu(t) = 0.1 + 0.1 \cdot \frac{t}{T}$, respectively. However, the RL agent is not aware of these underlying rate functions.

We employ the neural-network-based approximation for value functions and policies, following the 2-NNs approach described in Section \ref{sec:Numerical_Experiments}. 
We continue to use fully connected networks with each hidden layer followed by a ReLU activation function.
The only minor variation in this example is that the outputs of the actor network are one-dimensional, and we directly apply the sigmoid function to convert these outputs into admission probabilities. Algorithm \ref{alg:NN_for_queueing} is implemented with both the actor and critic networks configured with two hidden layers, each having a width of $8$.
The hyperparameters are set as follows: batch size $M=100$, entropy factor $\gamma = 1 \times 10^{-3}$, learning rate $\alpha_{\theta} = 3 \times 10^{-2}$, and learning rate $\alpha_{\phi} = 1 \times 10^{-5}$. 
Figure \ref{fig:queueing_performance_evolution} 
illustrates how the average revenues of the updated policies varies throughout the learning process. Upon completing the final update, the simulated average revenues of the resulting policy, achieves a value of $23.930$.

We also compare the performance of our RL algorithm with three benchmark policies:
\begin{itemize}
    \item Optimal Policy from Discretized Dynamic Programming: Similar as in Section~\ref{sec:dynamic_programming}, for small \(\Delta t\), solving the DP problem (without entropy regularization) for the discrete-time model can be expected to yield a reliable approximation $V_{\Delta t}^*(0, 0)$ of the true optimal expected revenue $V^*(0, 0)$. The associated DP equation is given by \begin{align}\label{eq:discrete_time_dynamic_programming_for_queueing}
    \left\{
        \begin{aligned}
        &V_{\Delta t}^*(t_k, x)= \lambda(t_{k+1}) \Delta t \cdot \boldsymbol{1}_{\{x \leq C-1\}} \cdot
        \max_{a \in \{0, 1\}} \big\{a \cdot [K_1 + V_{\Delta t}^*(t_{k+1}, x+1) - V_{\Delta t}^*(t_{k+1}, x)]\big\} + K_2 \Delta t \cdot x \\ 
        & \hspace{2cm}+ \mu(t_{k+1})\Delta t \cdot \boldsymbol{1}_{\{x \geq 1\}} \cdot [V_{\Delta t}^*(t_{k+1}, x-1) - V_{\Delta t}^*(t_{k+1}, x) ] + V_{\Delta t}^*(t_{k+1}, x), \quad \forall\, k,\ x \\
        & V_{\Delta t}^*(t_{K}, x) = K_3 \cdot x, \quad \forall\, x
        \end{aligned}
    \right.
    \end{align}
    \item \textsc{Uniform-Random}: 
    At each time and state, the decision to admit or reject an arriving customer is made with an equal probability of $0.5$.
    \item \textsc{Optimal-Threshold}: For each state $\bar{x} \in \{1, \ldots, C\}$, a (deterministic) threshold policy $\bpi^{\bar{x}}$ admits an arriving customer only if the current state is below the threshold \(\bar{x}\). We execute all these threshold policies and refer to the maximum average revenues obtained as the performance of the \textsc{Optimal-Threshold} policy.
\end{itemize}

We set $\Delta t = 0.001$ and solve the DP problem \eqref{eq:discrete_time_dynamic_programming_for_queueing}, yielding an optimal expected revenue of $V_{0.001}^*(0, 0) = 23.997$, as indicated in Figure \ref{fig:queueing_performance_evolution}.
Table \ref{tab:queueing_expected_revenue} reports the simulated average revenues for our 2-NNs policy and the benchmark policies \textsc{Uniform-Random} and \textsc{Optimal-Threshold}.
The numerical results demonstrate that our RL algorithm significantly outperforms the \textsc{Uniform-Random} and \textsc{Optimal-Threshold} benchmarks, and its performance approaches that of the optimal policy from DP with a negligible gap of only $0.3\%$.
\begin{algorithm}[htbp]
\caption{Actor-Critic Algorithm with Neural Networks for Queueing Control} \label{alg:NN_for_queueing}
\small
    \begin{algorithmic}[1]
    \State\textbf{Inputs:} maximum capacity $C$, time horizon $T$, the values of $K_1, K_2, K_3$; number of episodes $N$, batch size $M$; critic network $J^{\theta}(t, x)$ and an initial value $\theta_0$; 
    actor network $\bpi^{\phi}(a \mid t, x)$ and an initial value $\phi_0$; 
    entropy factor $\gamma$, learning rates $\alpha_{\theta}$, $\alpha_{\phi}$
    \State\textbf{Required program:} an environment simulator $(t', x', r') = Environment(t, x, \bpi^{\phi}(\cdot \mid \cdot, \cdot); C)$
    \State {Initialize $\theta \leftarrow \theta_0$, $\phi \leftarrow \phi_0$}
    \For{$i = 1$ to $N$} 
    \State Store $(\tau_0^{(i)}, x_0^{(i)}) \gets (0, 0)$
    \State Initialize $l = 0$, $(t, x) = (0, 0)$ 
    \Comment{Initialize $l$ to count jumps in each state trajectory, and $(t, x)$ to record the time and state right after each jump}
    \While{True}
    \State Apply $(t, x)$ and policy $\bpi^{\phi}$ to the environment simulator to get the next jump time $t'$, post-jump state $x'$ and the jump reward $r'$ \Comment{$r' = K_1$ for $x' = x + 1$ and $r' = 0$ for $x' = x - 1$}
    \If{$t' \geq T$}
    \State Store $L^{(i)} \gets l$, $\tau_{L^{(i)}+1}^{(i)} \gets T$
    \State \textbf{break}
    \EndIf
    \State Update $l \gets l + 1$
    \State Store observation at jump time: $(\tau_l^{(i)}, x_l^{(i)}, r_l^{(i)}) \gets (t', x', r')$
    \State Update $(t, x) \gets (t', x')$
    \EndWhile
    \State Store $\mathcal{I}^{(i)} = \{1 \leq l \leq L^{(i)} : x_l^{(i)} = x_{l-1}^{(i)} + 1\}$ \Comment{Store indices of customer admission transitions }
    \If{$i= 0 \pmod{M }$} \Comment{Perform an update using $M$ episodes generated under the policy}
        \State Update critic network with $\bar D(t_1, t_2, x; \theta) \coloneqq \int_{t_1}^{t_2} [J^{\theta}(s, x)]^2 ds$, $\bar b(t_1, t_2, x; \theta) \coloneqq \int_{t_1}^{t_2} J^{\theta}(s, x) ds$, $E(t_1, t_2, x;v, \bpi) \coloneqq \int_{t_1}^{t_2} v(s) \mathcal{H}(\bpi ( \cdot \mid s, x)) ds$ and $\bar w(t_1, t_2, x; \theta) \coloneqq \int_{t_1}^{t_2} s \cdot J^{\theta}(s, x) ds$:
        \begin{align*}
            \theta \leftarrow \theta - \alpha_{\theta} \nabla_{\theta} \Bigg\{& \frac{1}{M} \sum_{k=i - M + 1}^i \sum_{l = 0}^{L^{(k)}} \bigg( \frac{1}{2} \bar D(\tau_l^{(k)}, \tau_{l+1}^{(k)}, x_{l}^{(k)}; \theta) + \bar b(\tau_{l}^{(k)}, \tau_{l+1}^{(k)}, x_{l}^{(k)}; \theta)[K_2 \cdot x_{l}^{(k)} \tau_{l+1}^{(k)} + K_3 \cdot x_{L^{(k)}}^{(k)}] \\
            &- \bar b(\tau_{l}^{(k)}, \tau_{l+1}^{(k)}, x_{l}^{(k)}; \theta) \sum_{l' = l + 1}^{L^{(k)}} \left[ r_{l'}^{(k)} - K_2 \cdot x_{l'}^{(k)} (\tau_{l'+1}^{(k)} - \tau_{l'}^{(k)}) + \gamma E(\tau_{l'}^{(k)}, \tau_{l'+1}^{(k)}, x_{l'}^{(k)}; \boldsymbol{1}, \bpi^{\phi})\right] \\
            &- K_2 \cdot x_{l}^{(k)} \bar w(\tau_{l}^{(k)}, \tau_{l+1}^{(k)}, x_{l}^{(k)}; \theta) - \gamma E(\tau_{l}^{(k)}, \tau_{l+1}^{(k)}, x_{l}^{(k)}; \bar b(\tau_{l}^{(k)}, \cdot, x_{l}^{(k)}; \theta), \bpi^{\phi}) \bigg)\Bigg\}
        \end{align*}
        \State Update actor network: 
        \begin{align*}
            \phi \leftarrow \phi + \alpha_{\phi} \nabla_{\phi} \Bigg\{
            \frac{1}{M} \sum_{k=i - M + 1}^i \bigg(&\sum_{l \in \mathcal{I}^{(k)}} \log \bpi^{\phi} ( 1 \mid \tau_{l}^{(k)}, x_{l-1}^{(k)}) [ J^{\theta} (\tau_l^{(k)}, x_{l-1}^{(k)} + 1) - J^{\theta} (\tau_{l}^{(k)}, x_{l-1}^{(k)} ) + K_1 ] \\
            & + \gamma \sum_{l = 0}^{L^{(k)}} E(\tau_{l}^{(k)}, \tau_{l+1}^{(k)}, x_{l}^{(k)}; \boldsymbol{1}, \bpi^{\phi}) \bigg)\Bigg\} 
        \end{align*}
    \EndIf
    \EndFor
    \end{algorithmic}
\end{algorithm}

\begin{figure}[htbp]
    \centering
    \begin{tikzpicture}
    \begin{axis}[
        xlabel={Episode},
        ylabel={Average Revenue},
        xtick={0,1000000,2000000,3000000,4000000,5000000},
        xticklabels={0,1,2,3,4,5},
        xtick scale label code/.code={$\times 10^6$},
        ytick={0,10,15,20,23.997},
        yticklabel style={
            /pgf/number format/fixed,
            /pgf/number format/precision=5
        },
        scaled y ticks=false,
        xmin=0, xmax=5000000,
        ymin=8, ymax=25,
        legend pos=south east,
        grid=minor
    ]
    \addplot[very thick, blue] table [col sep=comma, x=x, y=y] {data/queueing_example.csv};
    \addlegendentry{2-NNs}
    \addplot [dashed, very thick, samples=2] coordinates {(0,23.997) (5000000,23.997)};
    \end{axis}
    \end{tikzpicture}
    \caption{Average revenues of Algorithm~\ref{alg:NN_for_queueing} over episodes for the example in Appendix~\ref{app:queueing-control}}
    \label{fig:queueing_performance_evolution}
\end{figure}

\begin{table}[htbp]
    \centering
    \caption{Simulation results for selected policies in Appendix~\ref{app:queueing-control}}
    \label{tab:queueing_expected_revenue}
    \begin{tabular}{p{3.5cm}>{\raggedleft\arraybackslash}p{2.5cm}>{\raggedleft\arraybackslash}p{2cm}>{\raggedleft\arraybackslash}p{3cm}>{\raggedleft\arraybackslash}p{3cm}}
    \hline
    Policy & Average revenue & $99\%$-CI ($\pm$) & {$\frac{\text{Average revenue}}{V_{0.001}^*(0, 0)}$ ($\%$)} &  Time cost (second)\\
    \hline
    2-NNs & 23.930 &  0.390 & 99.72 & 1,526 \\
    \textsc{Uniform-Random}  & 8.603 & 0.415 & 35.85 & $*$\\
    \textsc{Optimal-Threshold} & 13.358 & 0.325 & 55.67 & $*$\\
    \hline
    \end{tabular}
    
    \footnotesize
    $*$ indicates that the method involves neither solving LP problems nor learning; the time cost is virtually zero.
\end{table}

\end{document}